\documentclass[sigconf,nonacm,balance=false]{aamas}

\usepackage{array}
\usepackage{tabularx}
\usepackage{xspace}
\usepackage{balance}

\makeatletter
\AtBeginDocument{%
  \def\city#1{\global\@ACM@citypresenttrue\ignorespaces}%
  \def\country#1{\global\@ACM@countrypresenttrue\ignorespaces}%
}
\makeatother

\AtBeginDocument{%
  \fancypagestyle{standardpagestyle}{%
    \fancyhf{}%
    \renewcommand{\headrulewidth}{0pt}%
    \renewcommand{\footrulewidth}{0pt}%
    \fancyfoot[C]{\footnotesize\thepage}%
  }%
  \fancypagestyle{firstpagestyle}{%
    \fancyhf{}%
    \renewcommand{\headrulewidth}{0pt}%
    \renewcommand{\footrulewidth}{0pt}%
    \fancyfoot[C]{\footnotesize\thepage}%
  }%
  \pagestyle{standardpagestyle}%
  \hypersetup{%
    pdftitle={GenCoord: Executable Commitments for Embodied LLM Agents under Private Information},%
    pdfauthor={Junning Zhu, Peng He, Haohan Yuan, Jianpeng Liang}%
  }%
}

\title[GenCoord]{GenCoord: Executable Commitments for Embodied LLM Agents under Private Information}

\author{Junning Zhu}
\authornote{Equal contribution.}
\affiliation{%
  \institution{Beijing Normal-Hong Kong Baptist University}
  \city{Zhuhai}
  \country{China}}
\email{t330025113@mail.bnbu.edu.cn}

\author{Peng He}
\authornotemark[1]
\authornote{Corresponding author.}
\affiliation{%
  \institution{Tsinghua University; Sophon.ai}
  \city{Beijing}
  \country{China}}
\email{hepeng@tsinghua-wx.org}

\author{Haohan Yuan}
\affiliation{%
  \institution{University of North Carolina at Charlotte}
  \city{Charlotte}
  \country{United States}}
\email{hyuan3@charlotte.edu}

\author{Jianpeng Liang}
\affiliation{%
  \institution{University of California San Diego}
  \city{San Diego}
  \country{United States}}
\email{jil652@ucsd.edu}

\renewcommand{\shortauthors}{Zhu et al.}

\begin{abstract}
Embodied language-model agents must coordinate interdependent actions while each holds private task information. In dialogue-based coordination, models generate messages and interpret replies to decide what each agent should do, spending inference before the team acts. We propose GenCoord, a decentralized protocol that makes the division of work itself the message. A proposal pairs the sender's skill path with its partner's continuation in Short DSL, a compact language for expressing joint work. Compatible paths form commitments linking the sender's delivery to the partner's work after verified receipt. We check fulfillment separately from original-task success. Complete requests also let a deterministic receiver proceed without another model call. In Minecraft, trained local 0.8B agents match the task success of controlled free-form communication while reducing median commitment time by 68.2\%. The same-field full-key JSON control approaches free-form output length and delay, indicating that the amount generated matters even when messages are structured. With the same prompted DeepSeek model, both protocols complete nearly all tasks, but dialogue generates 2.8 times as many output tokens. Multistep commitments improve held-out-template success from 91.3\% to 98.1\% with fewer decisions. These results point to a design principle for embodied cooperation: make the division of work the message, and spend inference on choosing and revising the work.
\end{abstract}

\newcommand{\method}{\textsc{GenCoord}\xspace}

\newcolumntype{L}[1]{>{\raggedright\arraybackslash}p{#1}}
\newcolumntype{C}[1]{>{\centering\arraybackslash}p{#1}}
\newcolumntype{Y}{>{\centering\arraybackslash}X}
\newcolumntype{Z}{>{\raggedright\arraybackslash}X}
\newcommand{\tabstyle}{\small\setlength{\tabcolsep}{3pt}\renewcommand{\arraystretch}{1.10}}

\usepackage{algorithm}
\usepackage[noend]{algpseudocode}
\usepackage{enumitem}
\usepackage{float}
\usepackage{listings}
\usepackage{multirow}
\usepackage{needspace}
\usepackage{placeins}
\usepackage{siunitx}
\usepackage{longtable}

\begin{document}
\maketitle
\raggedbottom
\hypersetup{pdfauthor={Junning Zhu, Peng He, Haohan Yuan, Jianpeng Liang}}

\section{Introduction}
Language-model-driven embodied agents cooperate by handing over materials, processing objects, and completing actions that depend on a partner's work. We consider teams without a central planner, where each agent holds local observations and private task information that its partner needs to choose compatible work~\cite{white2025collaborating,long2024teamcraft,zhang2024coela}. In dialogue-based coordination, a model generates each message and its partner processes that message before choosing an action. These exchanges consume inference on both sides, so communication can delay the agents' next coordinated action.

We illustrate this information split with two Minecraft agents fulfilling a crafting-table order. Agent $A$ sees the order, whereas $B$ observes which agent the task assigns to craft it. When $B$ is assigned to craft, $A$ should hand over planks; when $A$ is assigned to craft, it should deliver the finished table. The product is the same, but the required actor, transfer object, and continuation differ. Such task rules represent responsibilities that can arise from authorization or resource ownership even when both agents can perform the operation. To coordinate, $A$ needs the consequence of $B$'s private information for their division of work, and $B$ needs a request consistent with the order known to $A$.

Recent embodied systems combine language with planning and execution in several ways. CoELA chooses between communication and environmental actions, RoCo uses dialogue to coordinate robot plans, and MINDcraft supports action-by-action exchanges~\cite{zhang2024coela,mandi2024roco,white2025collaborating}. SMART-LLM constructs team plans through decomposition and allocation, while COHERENT uses a centralized task assigner with feedback from robot executors~\cite{kannan2023smart,liu2025coherent}. Efficiency-oriented approaches reorganize communication~\cite{chen2024scalable,guo2026organized}, reuse plans~\cite{aamas2026_QPUC1624}, or select the observations to share~\cite{aamas2026_WHRE3513}. We instead ask what each message should contain so that decentralized agents spend fewer tokens and calls coordinating their work. That choice matters for frequent decisions and local deployment, where compact models and asynchronous execution are being developed to improve embodied responsiveness~\cite{shukor2025smolvla}.

\begin{figure*}[t]
  \centering
  \includegraphics[width=0.96\textwidth]{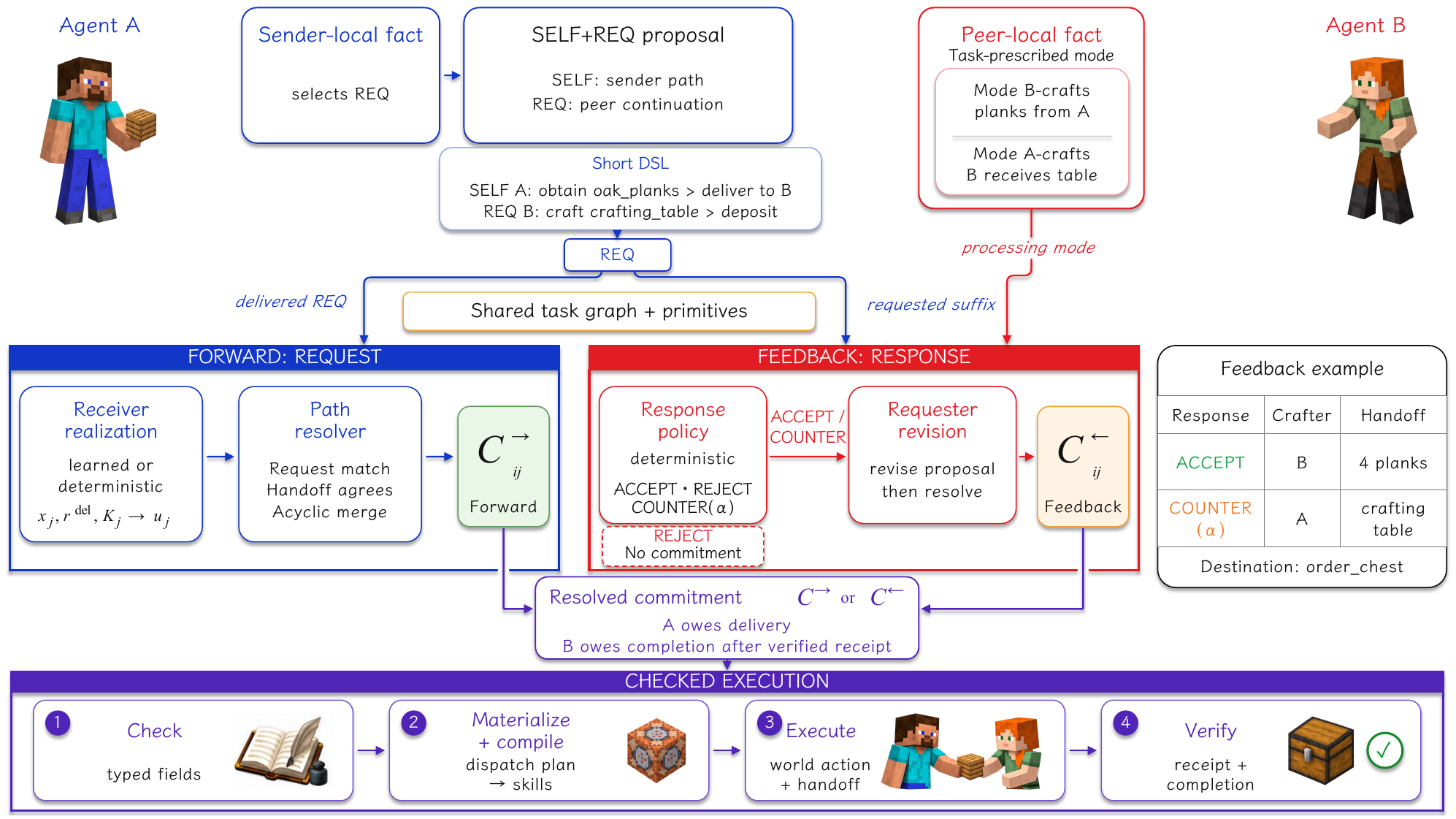}
  \caption{GenCoord in the shared-graph configuration: two information directions create executable obligations. \texttt{SELF+REQ} pairs the sender's work with the peer's continuation. The blue route realizes and resolves the delivered request; the red route applies a deterministic response policy, followed by requester revision and resolution. The example shows how the task-prescribed processing mode changes the handoff. The sender owes delivery, and the receiver owes completion conditional on verified receipt. Checked execution verifies fulfillment through the handoff and declared outcome.}
  \Description{An overview with two agents at the top, a blue forward route and a red feedback route in the middle, and checked execution at the bottom. Agent A's private fact selects the request in a SELF+REQ proposal. SELF describes the sender's path and REQ the peer's continuation, encoded with the Short DSL precedence separator >. Agent B observes a task-prescribed processing mode: B crafts from planks delivered by A, or A crafts and B receives the table. A shared task graph and primitive semantics support both routes. In the blue route, receiver realization uses local context x sub j, the actual delivered request r delivered, and executable prior K sub j to produce continuation u sub j. It can be learned or deterministic. The path resolver checks a match to the normalized delivered request, agreement of handoff fields, and an acyclic dependency merge, producing the forward commitment C forward. In the red route, a deterministic policy maps the requested suffix and processing mode to ACCEPT, REJECT, or COUNTER(alpha). ACCEPT and COUNTER lead to requester revision and resolution, producing C feedback; REJECT creates no commitment. The example table assigns B to craft after receiving four planks under ACCEPT, or A to craft and hand over the table under COUNTER; both use the order chest as the destination. Either route yields the resolved obligation: A owes delivery and B owes completion after verified receipt. The bottom strip checks typed fields, materializes the dispatch plan and compiles it to skills, executes world actions and the handoff, and verifies receipt and declared completion.}
  \label{fig:overview}
\end{figure*}

Our insight is to make the division of work itself the message. A sender proposes the skills it will perform together with the continuation it asks its partner to perform, so the transmitted object already identifies the work on both sides of the handoff. Compatible proposals form a \emph{commitment}, which specifies what each agent owes its partner and when that work becomes due. Because a complete request gives the receiver its continuation, the receiver can check and instantiate that work deterministically. The model's role is then to interpret private information and select or revise the proposed work. We can cut model work in three ways: encode paths compactly, realize complete requests without another model call, and commit beyond the next intermediate decision. A local model can therefore spend its generation budget specifying joint work that its partner can execute directly.

In this paper, we propose \emph{GenCoord}, which implements this idea through \emph{Short DSL}, a compact language for skill paths and partner requests. Each agent generates a proposal from its own information, and the protocol supports both realizing a delivered request and revising a proposal after structured feedback. Our contribution is a peer protocol in which models generate proposals for both agents' physical work and handoff dependencies, while fixed rules resolve and verify commitments. This links decentralized generation to conditional commitment semantics~\cite{singh1998acl,singh2008commitments,chopra2025azorus}, with checks before dispatch and physical readback after execution (Figure~\ref{fig:overview}).

We find that this protocol substantially reduces the model work needed for coordination. Trained local 0.8B agents complete the same tasks through Short DSL, JSON, and controlled free-form communication, yet Short DSL reduces peer traffic by 92.8\% and median commitment time by 68.2\% relative to free-form communication. Our same-field, full-key JSON control produces nearly as many tokens and takes nearly as long as free-form communication. With DeepSeek, both protocols complete nearly all tasks, while dialogue uses 8.1 times the peer traffic and 2.8 times the output tokens; Short DSL reduces mean commitment time by 30.4\%. The local control isolates the encoding of identical fields, and the DeepSeek comparison establishes the cost of each complete communication procedure.

The same protocol also supports longer commitments and different model deployments. Multistep work improves held-out-template success from 91.3\% to 98.1\% and reduces online decisions by 32\%, while structured feedback lets the requester revise the actor and handoff in response to a partner's private information. Local and hosted assignment models can use the same deterministic receiver. Our contributions are:

\begin{itemize}
\item \textbf{A generative coordination protocol.} GenCoord turns local skill paths and partner requests into commitments linking both agents through verified handoffs and feedback.
\item \textbf{Communication as model work.} Controlled comparisons show that compact paths reduce model output, peer traffic, and commitment time, while multistep commitments improve generalization and reduce decisions.
\item \textbf{Independent selection and realization.} One commitment interface supports local and hosted selectors, deterministic reception, and execution of transmitted paths.
\end{itemize}

\section{Related Work}
\paragraph{Embodied LLM agents.}
SayCan combines language-model proposals with skill value functions to select feasible robot actions~\cite{ichter2023saycan}. Voyager builds reusable skills, while Code as Policies and ProgPrompt generate programs using available actions and objects~\cite{wang2023voyager,liang2023code,singh2023progprompt}. MindCraft, MindAgent, and TeamCraft evaluate situated interaction and teamwork~\cite{bara2021mindcraft,gong2024mindagent,long2024teamcraft}, while VillagerAgent and CausalMACE organize task dependencies and Parallelized Planning-Acting overlaps planning and execution~\cite{dong2024villageragent,chai2025causalmace,aamas2026_EXAJ9853}. GenCoord extends model-guided skill selection across agents by making the selected work a message for the partner.

SMART-LLM generates a multi-robot plan through decomposition, coalition formation, and allocation~\cite{kannan2023smart}, while COHERENT couples a central task assigner to executors that report feedback~\cite{liu2025coherent}. Chen et al. compare centralized, decentralized, and hybrid communication~\cite{chen2024scalable}, while Guo et al. organize embodied agents through prompted leadership~\cite{guo2026organized}. GenCoord instead has each agent propose work from its own information and connects the proposals at the handoff. SmolVLA motivates local inference with compact models and responsive execution~\cite{shukor2025smolvla}, while hosted models offer another operating point.

\paragraph{Plans and commitments.}
We build on accounts of shared goals and partial plans in joint intentions, SharedPlans, and STEAM~\cite{cohen1991teamwork,grosz1996collaborative,tambe1997teamwork}, and Contract Net and allocation methods distribute work through contracts and utilities~\cite{smith1980contractnet,gerkey2004mrta,choi2009cbba}. CLIPS infers hidden goals from actions and language, while Li et al. maintain explicit beliefs during collaboration~\cite{aamas2024_p2094,li2023theory}. 

Commitment semantics specifies what agents owe one another and which events activate or discharge those obligations~\cite{singh1998acl,fornara2002operational,singh2008commitments}. Flexible commitment protocols describe how agents enact their commitments~\cite{yolum2002flexible}, while Pippi, Kiko, and Peach address role instantiation, admissible messaging, and communication infrastructure~\cite{aamas2022_p281,aamas2023_p1154,aamas2026_SAXY3841}. Acceptance and counteroffers determine which proposed work enters that commitment~\cite{jennings2001negotiation}.

\paragraph{Communication and computation.}
Learned channels, information bottlenecks, and Optima optimize agent messages~\cite{foerster2016dial,sukhbaatar2016commnet,wang2020imac,chen2024optima}. AutoGen and MetaGPT organize conversations and workflows~\cite{wu2023autogen,hong2024metagpt}, while Gated Coordination, TickingCollab, $\mu$ACP, and AgenticCache address escalation, timing, resource use, and reuse~\cite{jian2026gated,yi2026ticking,aamas2026_PHRW6922,kim2026agenticcache}. MeCo retrieves and reuses plans for similar tasks to avoid invoking the model again~\cite{aamas2026_QPUC1624}. UNCAP selects relevant vehicles and exchanges lightweight messages describing uncertain observations~\cite{aamas2026_WHRE3513}. ReAd learns a sequential advantage function from planned trajectories to guide plan refinement, reducing interaction steps and model queries~\cite{zhang2025read}. GenCoord reduces what agents generate per coordination and directly specifies work that a deterministic receiver can realize. Shared skills connect it to task grammars and ad hoc teamwork~\cite{aamas2024_p544,stone2010adhoc}. We use message interventions inspired by Lowe et al. and Jaques et al. to check request and feedback semantics~\cite{lowe2019pitfalls,jaques2019social}, while the main comparisons measure protocol cost and multistep task performance.

GenCoord makes both agents' executable work the message, connecting their proposed skills through handoff dependencies and fixed commitment checks. Beyond Natural Language compares alternative communication formats, and OPTiMACS learns task-aware message representations for collaborative problem solving~\cite{chen2024beyond,gupta2026optimacs}. In GenCoord, the transmitted content specifies the work that the runtime dispatches, so representation also determines which skills the agents undertake. BSPL, Clouseau, and Azorus address information dependencies, protocol synthesis, and commitment enactment~\cite{singh2011bspl,singh2020clouseau,chopra2025azorus}. We bring those concerns into a protocol whose message content consists of locally generated physical skill paths and whose commitments are checked in the environment. IntentMCTS communicates states that guide partner search~\cite{aamas2025_p481}, whereas GenCoord requests the multistep continuation that its partner will execute. CoELA, RoCo, and MINDcraft coordinate through dialogue~\cite{zhang2024coela,mandi2024roco,white2025collaborating}. GenCoord instead uses the proposed skill paths themselves to establish responsibilities, dispatch skills, and verify fulfillment.

\section{Coordination under Private Information}
\label{sec:information}
We consider two agents $\mathcal A=\{a_i,a_j\}$ that cooperate through environmental actions and messages, with agent $i$ receiving context $x_i=(o_i,h_i,m_i)$ containing its observation, execution history, and received messages. Shared primitives $\mathcal K_0$ define skills, argument types, and grounding rules, while $\mathcal K_i$ adds the task structure available to that agent. We call information in $x_i$ or $\mathcal K_i$ \emph{private} when the partner cannot observe it but needs its consequences to choose compatible work. Because private information can determine the actor, transfer object, destination, or continuation, decentralized decisions depend on what each agent observes and communicates~\cite{bernstein2002decentralized}.

We define a task $z$ by its required outcome and allocation of work, using \emph{binding} $\beta$ for the selected actors, objects, and continuation. The set $\Pi^\star(z;\mathcal K)$ contains the supported paths satisfying that task. In the workcell example, the two processing modes prescribe different crafting actors and handoff objects even though they request the same product. Our tests therefore evaluate the prescribed allocation as well as the terminal world state.

The communication requirement follows directly from the local information available at the decision. When two equiprobable tasks present the same local view but have disjoint acceptable routes, a policy choosing a route from that view has expected success at most $1/2$. We use this paired ambiguity to construct checks of message use, keeping the acceptable task routes distinct from the particular string a model emits. The design question is how to transmit the information so that a partner can act on it, which motivates the executable paths below (Supplement~\S S12).

\section{GenCoord}
\label{sec:method}
GenCoord proceeds from a proposed division of work to a resolved commitment, then checks fulfillment against the execution trace. Its key design choice is to generate the work that each party undertakes in the same representation that the receiver and runtime inspect.

\subsection{From Paths to Commitments}
A grounded task $\tau=(a,k,\mathbf v,\ell,P)$ specifies an actor, skill, typed arguments, location, and predecessors. A proposal $q_i=(s_i,r_{i\rightarrow j})$ pairs the sender's path with a requested continuation for the peer, so both sides of the intended handoff are explicit. We serialize these paths in Short DSL as \texttt{SELF} and \texttt{REQ}, using \texttt{>} for precedence, as in this workcell proposal:
\begin{center}
\begin{minipage}{0.96\columnwidth}
\small\ttfamily
SELF resource.obtain(q=4,item=oak\_planks)\newline
\hspace*{1em}> resource.deliver(q=4,item=oak\_planks,to=agent\_b)\newline
REQ agent\_b craft.item(q=1,input=oak\_planks,\newline
\hspace*{1em}item=crafting\_table)\newline
\hspace*{1em}> resource.deliver(q=1,item=crafting\_table,\newline
\hspace*{2em}dst=order\_chest)
\end{minipage}
\end{center}
When the paths agree on the handoff and their dependencies are acyclic, we merge them into $C_{ij}=(S_i,S_j,\beta,P_C)$, which specifies each agent's work and its ordering. A commitment is created at this resolution step, after the proposed assignment has passed the interaction checks and both sides specify compatible work. The resolver checks compatibility between choices made by the agents. For example, a plank delivery can support a receiver's crafting step only when the actors, item quantity, and dependency agree. This compatibility is established symbolically before dispatch, while the later receipt check establishes that the promised material was actually transferred. Keeping those events distinct allows the same proposal representation to support both generation and deterministic interpretation.

Conditional commitments explain why receipt matters to the receiver's responsibility: the sender's promised delivery establishes the precondition under which the receiver owes its remaining work. Writing $C(d,c,p,q)$ for debtor $d$ owing creditor $c$ outcome $q$ when $p$ holds~\cite{singh2008commitments}, we assign
\begin{equation}\label{eq:commitment}
\operatorname{Obl}(C_{ij})=\{C(a_i,a_j,\top,H_\beta),\ C(a_j,a_i,H_\beta,T_\beta)\}.
\end{equation}
Here $H_\beta$ is verified receipt of the handoff and $T_\beta$ is the declared terminal outcome. Receipt fulfills the sender's obligation and activates the receiver's remaining work, whose completion is checked in the world~\cite{chopra2025azorus}. The receipt remains valid when crafting consumes the transferred materials. In the example, the receiver owes the requested crafting and deposit after it has received the planks, while the sender is responsible for making that transfer. We record the transfer as an event in the trace, so subsequent consumption preserves evidence that the sender fulfilled its delivery.

We distinguish that fulfillment from the original task's success because a commitment can specify the wrong assignment. For a trace $\xi$ ending in world state $\omega_T$,
\begin{equation*}
\begin{aligned}
F(C_{ij},\xi)&=H_\beta(\xi)\land T_\beta(\xi),\\
G_z(\xi)&=\operatorname{ExecOK}(\xi)\land P_z(\omega_T)\land A_z(\xi).
\end{aligned}
\end{equation*}
The predicates $P_z$ and $A_z$ specify the original outcome and allocation, respectively. Thus $F$ evaluates the committed work, whereas $G_z$ evaluates whether that work satisfied the task.

\begin{figure*}[t]
  \centering
  \includegraphics[width=0.90\textwidth]{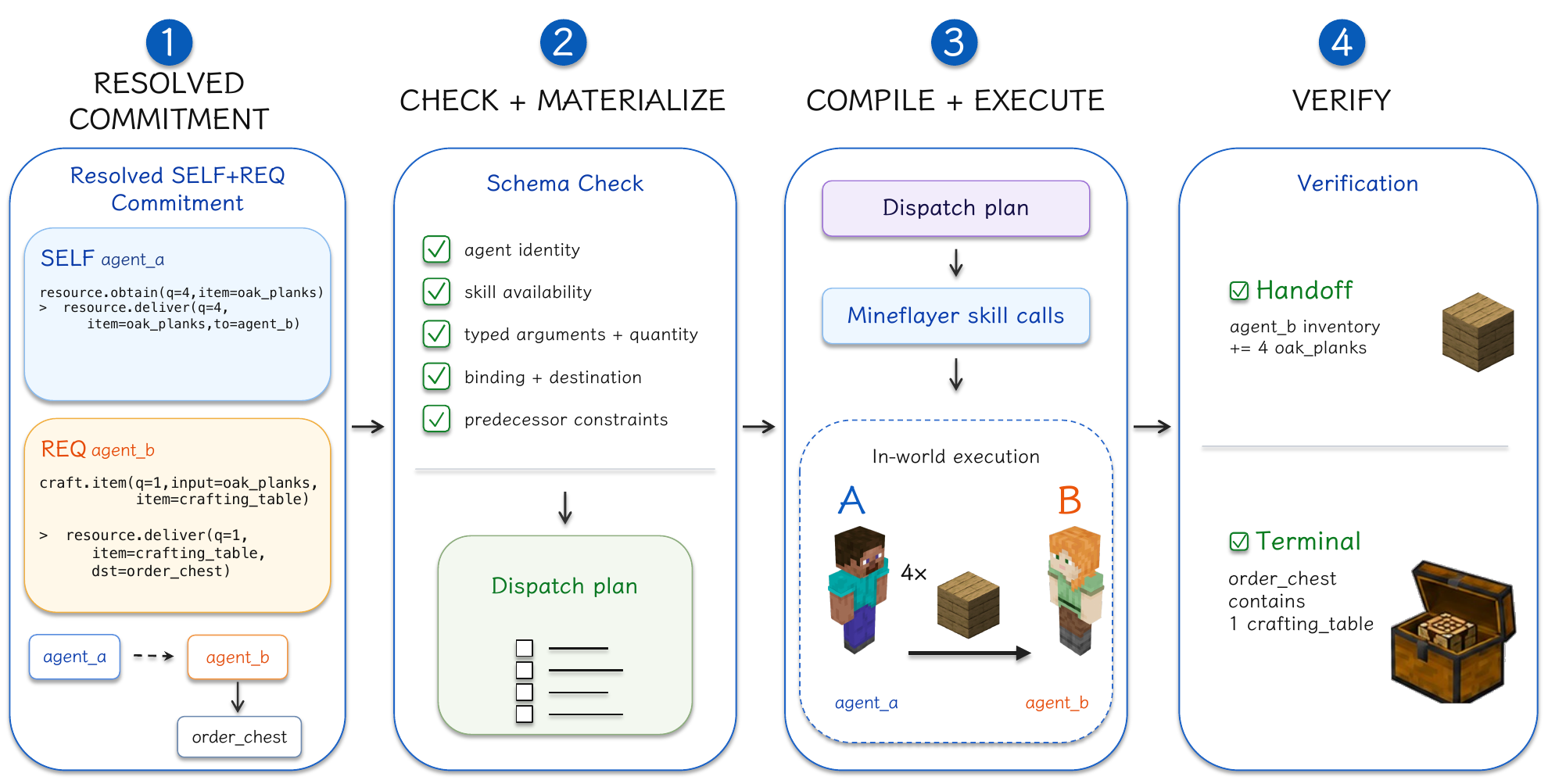}
  \caption{Making an obligation observable. A resolved object is checked, materialized as a dispatch plan, and executed through Minecraft skills. Inventory readback verifies that the receiver obtained the planks and that the crafting table reached the order chest. These checks establish fulfillment of the commitment.}
  \Description{The resolved SELF+REQ object assigns acquisition and delivery to Agent A, followed by crafting and deposit by Agent B. Checked task fields determine a dispatch plan and Mineflayer calls. Four oak planks transferred to the receiver establish handoff; one crafting table in the order chest establishes the declared terminal predicate. Original-task success is evaluated against its own fixed specification.}
  \label{fig:grounded}
\end{figure*}

\subsection{Requests and Feedback}
\label{sec:resolution}
In a forward interaction, the receiver produces \mbox{$u_j=R_j(x_j,r^{\mathrm{del}}_{i\rightarrow j},\mathcal K_j)$} from its local context, executable prior, and delivered request, using the inventory expected after the proposed handoff to generate or deterministically instantiate the continuation. The path resolver forms \mbox{$C^{\rightarrow}_{ij}=\operatorname{Resolve}_{\rightarrow}((s_i,r^{\mathrm{del}}_{i\rightarrow j}),u_j)$} when the continuation matches the normalized delivered request, the handoff fields agree, and the merged dependencies are acyclic. Using the delivered request also gives message substitutions a precise operational meaning.

Feedback instead lets the requester revise who performs the work. A deterministic policy reads the requested suffix and the peer's processing mode, returning \textsc{Accept}, \textsc{Reject}, or $\textsc{Counter}(\alpha)$ for an admissible alternative. Acceptance preserves the proposed branch, while a counteroffer changes the allocation and preserves the goal and peer. The requester then generates a revised proposal from its context, initial proposal, and response, after which the feedback resolver checks the resulting handoff and path. Rejection creates no commitment, whereas a counteroffer in our example asks the requester to craft before handing over the finished table. The response therefore communicates the consequence of the peer's mode in the form of a concrete alternative that the requester can incorporate into its own path. The resolver checks the revised paths against the response, including who processes the material and what the requester will deliver. These two interactions place the additional model call on the agent that still has work to determine: the peer realizes the incoming request in the forward route, whereas the original requester revises its proposal in the feedback route. A deterministic peer can perform the former operation when the request already specifies its full continuation.

\subsection{Realization and Execution}
We parse and check the generated message to validate identities, skills, arguments, destinations, and dependencies, after which the materializer constructs the task actions needed for dispatch. The formation backend chooses and expresses the assignment, while the materializer supplies the task details available through the shared graph or received continuation. In the shared-graph configuration it expands an admitted binding through the available task graph. When the sender keeps its graph private, the receiver instead derives this work from the delivered continuation and shared primitives.

After the compiler converts the dispatch plan to Mineflayer skills with navigation and execution parameters~\cite{prismarinejs2026mineflayer}, we execute it and inspect inventories and world objects to verify the transfer and declared outcome (Figure~\ref{fig:grounded}). Because formation, fulfillment, and task success are distinct checks, we can locate whether an unsuccessful trial failed to form a commitment, failed to carry it out, or carried out an inappropriate commitment. Complete paths consequently support deterministic realization without another model call, since the receiver's remaining task is to validate and instantiate the continuation it has already received. For a generated continuation, the same checks determine whether the model followed the request, so the execution interface remains stable across both implementations.

\subsection{Choosing the Implementation}
\emph{Direct Short DSL} (Direct) uses Qwen3.5-0.8B to generate complete paths, while \emph{Factor-Code + Rule} (Factor) predicts a short binding code that a deterministic composer expands into the same object~\cite{qwenteam2026qwen35}. A rule can transfer an explicitly supplied binding, and lexical or embedding retrieval can select one from a description. Hosted generation provides another choice when greater assignment accuracy is valuable, although the service's response time becomes part of the agent's decision latency. Each backend expresses its selected assignment through the same receiver-facing request interface.

Receiver knowledge provides a second design choice. A shared directory maps codes to known continuations, whereas \texttt{PATH} transmits an ordered continuation and \texttt{SLICE} transmits the relevant graph fragment. \texttt{GOAL} instead supplies a target and constraints, which an adapter realizes by searching inventory, container, and block states from the specified handoff. Because this adapter checks the target and constraints, a different legal path can satisfy the same commitment. The sender can transmit a complete path or leave the receiver to recover steps from shared knowledge or search, trading message length against receiver computation (Supplement~\mbox{\S S7--S8}).

\section{Experimental Setup}
\label{sec:setup}
\subsection{Tasks and Inputs}
We evaluate Minecraft cooperation over eight templates in four families: Destination, Recipe, Allocation, and Active Branch (Figure~\ref{fig:families}). Each template pairs two task assignments with balanced frequencies and swaps sender and receiver identities across two role views. For example, Recipe changes what the receiver makes from planks, whereas Allocation changes whether the receiver stores material or uses it for construction. We hold the skill library and task checks fixed across communication methods so their cost can be compared on the same physical work. The workcell diagnostic crosses two goals with two prescribed processing modes, both of whose crafting routes are physically executable. In the controlled local revision diagnostic, a parsed initial proposal must match the prescribed proposal before the second call begins.

We examine task structure by keeping the instance graph private to the sender and transmitting a continuation produced by deterministic projection. These physical tests cover 114 parameter settings within 24 template--motif families, and a distinct probe measures the unchanged Direct model's ability to generate the object. Language tests use 32 original phrasings across 16 assignments or 640 new test descriptions, with 40 per assignment and 256 disjoint calibration descriptions. We evaluate the existing checkpoints on both description sets. The selector chooses between each template's two assignments, after which the same protocol handles the selected work.
\begin{figure*}[t]
  \centering
  \includegraphics[width=0.96\textwidth]{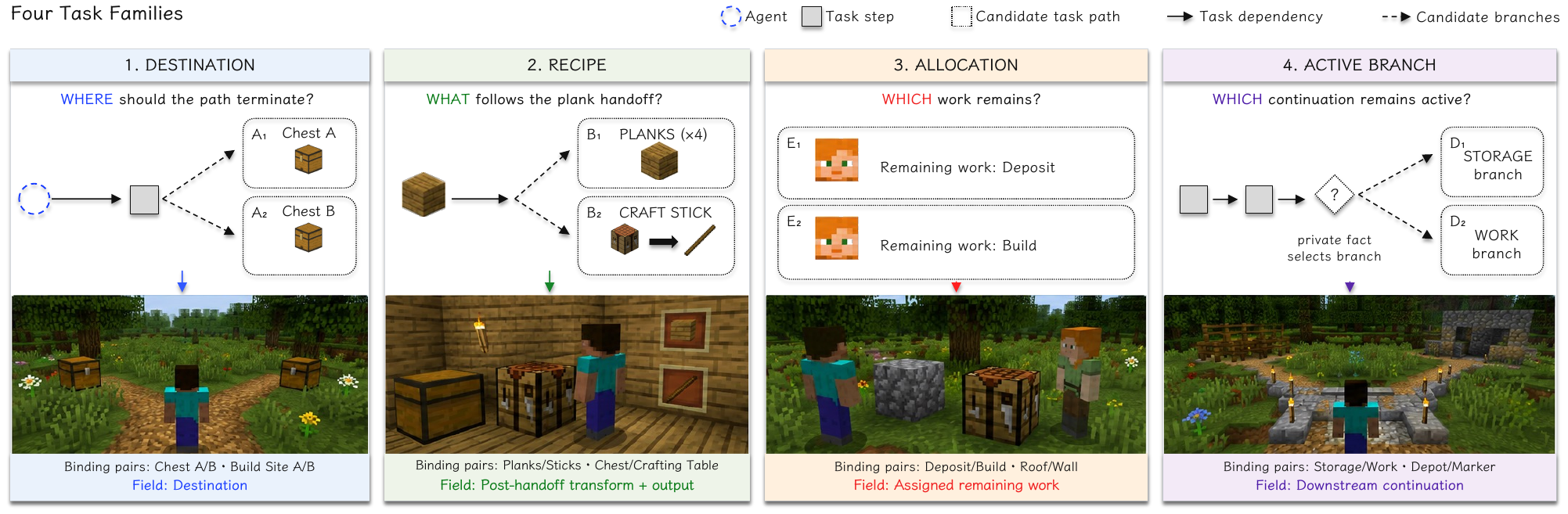}
  \caption{Four families vary the executable consequence of a private fact: destination, post-handoff transformation and output, assigned remaining work, or downstream continuation. Both Recipe templates transfer oak planks. The Allocation alternatives change storage/build work or the assigned wall/roof component.}
  \Description{Four paired-task panels identify the changed request fields. Destination selects a terminal location. Recipe selects a post-handoff transformation and output with the same oak-plank handoff. Allocation selects assigned remaining work, including storage versus building or wall versus roof. Active Branch selects the downstream continuation. These alternatives produce distinct task-relevant physical outcomes.}
  \label{fig:families}
\end{figure*}

\subsection{Models and Protocol Comparisons}
Our local backend is Qwen3.5-0.8B, trained with programmatic examples of local facts and task paths~\cite{qwenteam2026qwen35}. Direct uses 1,920 training examples, two epochs, and 240 updates per seed, with effective batch size 16 and deterministic decoding. We train three checkpoints independently and compare complete training and control procedures for the two horizons, giving single-step control 2,880 examples and 360 updates to supervise its additional intermediate decision. JSON encodes the same fields and uses the same two call stages as Short DSL. Controlled free-form includes dialogue, accumulated context, and structured extraction, with 2,400 training examples and 300 updates. Factor predicts a short assignment code using matched training information.

For the prompted control, we use the same DeepSeek-V4.1-Flash service for both protocols~\cite{deepseek2026api}, with four worked examples balanced across the two role assignments. Dialogue examples use the teacher conversations from the local free-form training, identify both agents' public roles, and begin with the supplier. Within eight calls, the agents exchange natural language and submit structured proposals and acceptances whose explicit control markers the controller recognizes independently of explanatory prose. We test 128 task clusters with two roles and three repetitions per protocol. Feedback tests use 40 paired workcell units, comparing revision from a supplied proposal with the full interaction. Embedding retrieval uses Qwen3-Embedding-0.6B~\cite{qwen3embedding2025}, and the supervised lexical classifier is fitted on the training set. These selectors represent practical choices with different pretraining and supervision (Supplement~\mbox{\S S3--S6}).

\subsection{Measurements and Comparisons}
We measure commitment formation, verified handoff $H$, fulfillment $F$, task success $G$, and the stage where a trial stops, including all assigned trials in task-success rates. Time to commitment (TTC) includes proposal generation and processing up to commitment formation or a stopped attempt, while episode time includes world execution. We measure peer traffic, model calls, input and output tokens, and token-based API charges. For local structured messages, peer payload is the request, whereas free-form payload includes peer dialogue. The private-graph comparison counts application requests and acknowledgments under a common accounting rule.

Local surface and backend comparisons each use a 128-cluster quality pool and a distinct 40-cluster timing sample. Methods run sequentially on one RTX 4090D with hot models and batch size one. The DeepSeek comparison has 768 quality trials per protocol and 64 fresh timing trials per protocol, with the two protocols interleaved in a fixed alternating order, HTTP concurrency one, and one world at a time. Its TTC includes API service and network time. The receiver benchmark measures the receiver stage on a common local device. Selector accuracy is evaluated before world execution on 640 descriptions, while the decision-time comparison uses 128 measurements per strategy.

We report 95\% intervals from 10,000 paired bootstrap draws over task clusters, averaging roles and repeated calls within each unit~\cite{efron1994bootstrap}. For language comparisons, we resample source phrasings within their 16 assignments. Message substitutions preserve the world, local inputs, and call procedure while changing the legal request or response. Their role is to check the protocol, alongside fault injections at admission, handoff, and completion.

\begin{figure*}[t]
  \centering
  \includegraphics[width=0.98\textwidth]{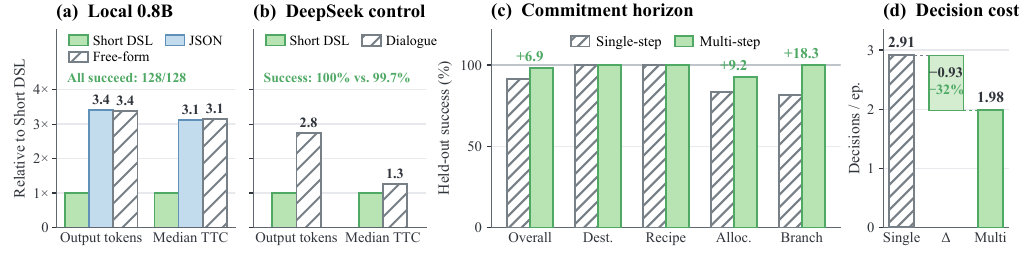}
  \caption{Communication costs relative to Short DSL and commitment horizon. (a) Local JSON and controlled free-form have similar output-token and median-TTC ratios, and all three complete the quality pool. (b) DeepSeek dialogue's output-token ratio is larger than its median-TTC ratio. Bars in (a,b) share a scale, with Short DSL at 1. Panel (a) uses the 40-cluster timing sample; panel (b) combines quality-pool output means with medians from 64 timing trials per protocol. (c,d) Multistep commitments improve held-out-template success and reduce online decisions.}
  \Description{Four panels compare normalized communication costs and commitment horizons. Panels a and b share a vertical scale relative to Short DSL at one. Green solid bars indicate Short DSL, light blue bars indicate JSON, and gray hatched bars indicate the natural-language baseline. In panel a, JSON and controlled free-form communication each produce about 3.4 times the output tokens and take about 3.1 times the median commitment time of Short DSL. All three complete the 128-cluster quality pool; the plotted costs come from the 40-cluster timing sample. In panel b, dialogue with the same DeepSeek model produces 2.8 times the output tokens and takes 1.3 times the median commitment time. Quality-pool success is 100 percent for Short DSL and 99.7 percent for dialogue. Output means use that quality pool, while timing medians use 64 interleaved trials per protocol. Panel c shows overall and per-family held-out-template success for single-step and multistep control, including 91.3 versus 98.1 percent overall. Panel d shows 2.91 versus 1.98 mean online model decisions, a reduction of 32 percent.}
  \label{fig:mechanisms}
\end{figure*}

\section{Results}
\label{sec:results}
We test how path encoding, commitment length, and receiver implementation affect the model computation needed for coordination. Physical checks distinguish fulfillment from task success.
\subsection{Communication and Model Work}
\label{sec:cost-results}
Compact paths reduce the work of expressing a commitment while preserving its task fields. All three local representations complete the 128-cluster quality pool (Table~\ref{tab:main}). The full-key JSON control uses the same fields and call stages as Short DSL, yet its output length and delay approach those of controlled free-form communication (Figure~\ref{fig:mechanisms}(a)). Short DSL expresses these fields as skill paths with typed arguments and an explicit precedence separator, reducing peer traffic by 92.8\%, median TTC by 68.2\%, and median episode time by 17.4\% relative to free-form communication. The encoding thus changes the model work needed to express the same task information. GenCoord defines the joint work, handoff dependencies, and commitment checks, while Short DSL makes this object compact to generate.

\begin{table*}[t]
\centering
\caption{Communication surfaces at matched task success. Quality and episode time use the 128-cluster pool; online costs use a 40-cluster surface-comparison timing sample. Wire is measured peer payload; TTC ends at commitment formation. p50/p95 denote median/95th percentile. }
\label{tab:main}
\tabstyle
\begin{tabular*}{\textwidth}{@{\extracolsep{\fill}}lrrrrrrr@{}}
\toprule
\textbf{Surface} & \textbf{Success} & \textbf{Calls} & \textbf{Input tok.} & \textbf{Output tok.} & \textbf{Wire (B)} & \textbf{TTC p50/p95 (s)} & \textbf{Episode p50/p95 (s)} \\
\midrule
Short DSL & 128/128 & 2.00 & 949.1 & 78.6 & 76.3 & 2.27 / 2.62 & 24.86 / 27.06 \\
JSON & 128/128 & 2.00 & 1154.7 & 267.6 & 303.4 & 7.07 / 7.86 & 29.94 / 31.82 \\
Controlled free-form & 128/128 & 2.50 & 1585.2 & 265.8 & 1061.1 & 7.13 / 7.86 & 30.09 / 31.92 \\
\bottomrule
\end{tabular*}
\end{table*}

With the same DeepSeek backend, Short DSL completes 768/768 quality trials and dialogue completes 766/768 (Table~\ref{tab:prompted-main}). Dialogue uses 8.1 times the peer traffic and 2.8 times the output tokens (Figure~\ref{fig:mechanisms}(b)). The protocols make a similar number of calls, but generate messages of substantially different lengths. Short DSL specifies the requested work in the resolver's input format.

\begin{table}[t]
\centering\tabstyle
\caption{Communication protocols with the same DeepSeek model. Quality uses 768 trials per protocol; TTC uses 64 interleaved trials.}
\label{tab:prompted-main}
\begin{tabular}{@{}lrr@{}}\toprule
Metric & Short DSL & Dialogue\\\midrule
Task success & 768/768 & 766/768\\
Calls & 2.00 & 2.12\\
Peer payload (B) & 112.75 & 913.76\\
Output tokens & 100.4 & 276.3\\
\midrule
Mean TTC (s) & 4.29 & 6.17\\
Median TTC (s) & 4.38 & 5.52\\\bottomrule
\end{tabular}
\end{table}

In the interleaved timing sample, Short DSL reduces mean TTC by 30.4\% and median TTC by 20.8\%. All timed trials form commitments, and task success is 63/64 for Short DSL and 64/64 for dialogue because one Short DSL handoff fails after formation. The hosted timing reduction is smaller than the local reduction, although transmitted messages and generated outputs are substantially shorter in both deployments. Hosted TTC includes service and network time as well as token processing, so we evaluate each protocol's delay in its actual deployment (Supplement~\mbox{\S S5--S6}).

\subsection{Multistep Work and Feedback}
\label{sec:horizon-results}\label{sec:feedback-results}
Compact encoding reduces the model work spent expressing a plan, while a longer commitment changes how often the agents must decide again. We test this second choice by extending the committed path beyond the next handoff or intermediate decision. On held-out templates, multistep control reaches 98.1\% task success against 91.3\% for single-step control, a paired improvement of 6.9 points $[2.9,10.8]$. Mean online decisions fall from 2.91 to 1.98, a 32\% reduction (Figure~\ref{fig:mechanisms}(c,d)). Multistep control commits past the intermediate stage and uses fewer training examples and updates. The result extends multi-step intent from guiding partner search~\cite{aamas2025_p481} to carrying out a committed sequence, showing that longer commitments improve generalization with fewer decisions.

Feedback lets the requester revise that sequence when its partner has private information about the assignment. Given an initial proposal and correct feedback, DeepSeek completes 419/480 workcell trials (87.3\%); when it generates the initial proposal as well, it completes 414/480 (86.3\%). The revised commitment changes who crafts and which item crosses the handoff, while preserving the requested product. Correct versus alternative feedback has a task-success difference of 87.3 points $[84.4,90.0]$. Alternative feedback changes the workcell's processing assignment: DeepSeek fulfills all 403 formed commitments and delivers the product to its target, while original-task success is zero across 480 trials. A compact response thus revises both agents' work, while fulfillment and original-task checks reveal whether the resulting assignment meets the task.

\subsection{Selection and Receiver Implementation}
\label{sec:binding-results}\label{sec:backend-results}\label{sec:private-structure}\label{sec:request-results}
A complete requested path also makes reception an implementation choice: the receiver can instantiate the work directly instead of generating a continuation. On paired sender outputs, both receivers form the same commitments. Task-success counts differ by two Factor trials and three Direct trials, with all differences arising after formation. Mean calls fall from approximately 1.94--1.95 to one, saving 343 ms $[292,374]$ with Factor and 785 ms with Direct before commitment. Because both implementations subsequently dispatch the work expressed by the request, the comparison separates the model needed to select an assignment from the code needed to realize its specified continuation.

To test the receiver's required knowledge, we keep the sender's instance graph private and supply deterministically projected paths, which complete 456 physical cases. Goal-based search supplies another 100 different legal continuations that also fulfill the commitment and task. GOAL sends fewer first-delivery bytes than PATH, 662 versus 741 including acknowledgment, but uses 2.0 rather than 0.065 ms of receiver computation. The unchanged Direct model forms 43.1\% of complete objects in the corresponding generation probe. The projected-path result establishes the receiver's ability to execute supplied structure, while that probe measures the sender's ability to express it.

The fulfillment and task checks used in the feedback intervention also locate errors in commitments formed from descriptions. In closed-loop runs based on 32 original descriptions, 2,181 of the 2,198 task failures after formation follow fulfillment of an incorrect assignment, while 17 arise during execution.

We therefore examine the deployment tradeoff where a model interprets a task description and selects an assignment (Table~\ref{tab:selection-main}). On 640 new descriptions over the shared task vocabulary, the local Factor model selects correctly in 95.94\% of trials at a mean decision time of 412 ms, compared with 99.97\% at 1,402 ms for DeepSeek. Factor uses about 29\% of the hosted model's measured decision time and makes no API request, whereas DeepSeek supplies the highest assignment accuracy. Routing the lowest-margin retrieval decisions to DeepSeek raises accuracy from 82.19\% to 88.91\%; the same-quota random-upgrade expectation is 84.88\%. The routing point uses the hosted model selectively, while running Factor throughout gives higher accuracy with local computation.

\begin{table}[t]
\centering\tabstyle
\caption{Selecting an assignment from new descriptions. Accuracy is measured before execution on 640 descriptions (3,840 choices for DeepSeek). Mean decision time comes from 128 timed decisions per strategy and includes service and network time for API calls.}
\label{tab:selection-main}
\begin{tabular}{@{}lrr@{}}\toprule
Selector & Accuracy (\%) & Time (ms)\\\midrule
Embedding retrieval & 82.19 & 92\\
Factor (local 0.8B) & 95.94 & 412\\
DeepSeek & 99.97 & 1,402\\
Retrieval + DeepSeek & 88.91 & 350\\\bottomrule
\end{tabular}
\end{table}

Each selector expresses its chosen assignment through the same receiver-facing interface, so we can change how work is chosen while keeping its realization and checking procedures fixed. The local and hosted results supply concrete operating points for accuracy, response time, and dependence on an external service.

\section{Discussion}
\label{sec:discussion}
\paragraph{Communication spends inference.}
Making joint work explicit reduces inference at several points before the next physical action. Compact paths reduce what agents generate, complete continuations support deterministic reception, and multistep commitments reduce how often the agents decide again. The full-key JSON control varies how task fields are encoded. The other comparisons vary how the receiver processes a received path and how many steps each commitment covers. Together, these comparisons show why a protocol should specify executable work, not only its transmission.

Figure~\ref{fig:mechanisms}(a,b) shows why generated output and elapsed time provide different views of communication cost. Local output length and commitment time change by similar proportions. With the hosted model, however, the timing difference is smaller, which may partly reflect service and network overhead as well as prompt processing. We therefore measure both the work required to produce a commitment and the delay before the agents can use it. Across both deployments, compact paths keep the actors, transfer, and continuation explicit while reducing generation.

\paragraph{Models choose, skills realize.}
Interpreting a description and instantiating a complete continuation are different operations. A language model can select or revise an assignment, while shared skills and a deterministic receiver realize work that the message has already specified. Changing the selector changes assignment accuracy and decision time, while the interface for sending, executing, and checking the selected work stays fixed. The local selector provides a fast, self-contained operating point, and the hosted selector provides higher assignment accuracy. We check the requested work against the commitment and the chosen assignment against the task, distinguishing fulfillment from task success. 

\paragraph{Commitments connect both agents' work.}
The defining object in GenCoord couples the sender's promised delivery to the receiver's work after receipt. This gives decentralized proposals the dependency structure that a centrally constructed joint plan would otherwise provide, while each proposal is generated from the participant's local information. Conditional commitment semantics explains when responsibilities take effect, and physical readback establishes whether the promised actions occurred. Admission, fulfillment, and task checks then serve different engineering purposes: validate the proposed work, detect an incomplete transfer or outcome, and assess whether the selected work satisfies the task. The same trace can consequently guide improvements to the assignment model, the protocol checks, or the physical skill that did not complete its work.

\paragraph{Knowledge, messages, and search.}
A path transmits a selected solution, while a goal and constraints ask the receiver to find one. A shared directory moves recurring task knowledge into the receiver's prior, and Factor's composer expands a short internal prediction into an explicit partner request. These alternatives show why the shortest message and the least computation need not coincide. Designers can choose how much structure to transmit by considering the receiver's knowledge and search cost together with the model work needed to express the message. Each commitment specifies who performs each step, what is handed over, and which outcome the receiver must achieve after receipt.

\paragraph{Limitations.}
We evaluate bilateral assignments in static Minecraft worlds with reliable messaging, shared discrete skills, and paired task alternatives. Extending the protocol to changing environments, larger teams, open-ended task decomposition, and robot platforms with continuous control and perception errors will require corresponding models of revision and responsibility. The static tests leave the cost of revising long commitments after environmental changes unmeasured.

\section{Conclusion}
We introduced GenCoord to let embodied agents communicate their own work and a partner's continuation as commitments. Its compact skill paths reduce communication and model inference costs with a local 0.8B backend, and communicating those paths through the prompted cloud model also reduces generated output and commitment time. Multistep commitments improve task success on held-out templates while reducing decisions, and structured feedback supports changes to the proposed assignment. Because the receiver can instantiate a complete request deterministically, we can choose between local and hosted selectors according to deployment needs while keeping the partner's execution interface unchanged. After the skills execute the committed paths, we check whether the agents fulfilled the commitment and whether their work completed the original task. Making the division of work the message lets models choose and revise that work and lets skills carry it out.

\clearpage
\bibliographystyle{ACM-Reference-Format}
\bibliography{references}
\definecolor{GCNavy}{HTML}{173F5F}
\definecolor{GCTeal}{HTML}{1F8A8A}
\definecolor{GCAmber}{HTML}{E89B35}
\definecolor{GCCoral}{HTML}{D95D5D}
\definecolor{GCViolet}{HTML}{725AC1}
\definecolor{GCSlate}{HTML}{6B7280}
\definecolor{GCHeader}{HTML}{E8F0F5}
\definecolor{GCLight}{HTML}{F4F7F9}
\definecolor{GCLine}{HTML}{CFD9E0}

\lstdefinestyle{gcpdsl}{
  basicstyle=\ttfamily\footnotesize,
  columns=fullflexible,
  breaklines=true,
  frame=single,
  rulecolor=\color{GCLine},
  backgroundcolor=\color{GCLight},
  xleftmargin=3pt,
  xrightmargin=3pt,
  aboveskip=4pt,
  belowskip=4pt,
  showstringspaces=false
}

\makeatletter\def\hyper@natlinkstart#1{}\let\hyper@natlinkend\@empty\makeatother

\clearpage
\nobalance
\onecolumn
\newgeometry{letterpaper,top=0.68in,bottom=0.70in,left=0.72in,right=0.72in,headheight=14pt,headsep=0.18in}
\appendix
\renewcommand{\normalsize}{\fontsize{10}{12}\selectfont}
\renewcommand{\small}{\fontsize{9}{11}\selectfont}
\renewcommand{\footnotesize}{\fontsize{8}{9.5}\selectfont}
\normalsize
\raggedbottom
\setlength{\emergencystretch}{1em}
\setcounter{section}{0}
\setcounter{equation}{0}
\setcounter{figure}{0}
\setcounter{table}{0}
\renewcommand{\thesection}{S\arabic{section}}
\renewcommand{\thesubsection}{\thesection.\arabic{subsection}}
\renewcommand{\theequation}{S\arabic{equation}}
\renewcommand{\thefigure}{S\arabic{figure}}
\renewcommand{\thetable}{S\arabic{table}}
\renewcommand{\thealgorithm}{S\arabic{algorithm}}
\renewcommand{\theHsection}{appendix.\arabic{section}}
\renewcommand{\theHfigure}{appendix.\arabic{figure}}
\renewcommand{\theHtable}{appendix.\arabic{table}}
\renewcommand{\theHequation}{appendix.\arabic{equation}}
\graphicspath{{./}}
\newcolumntype{R}[1]{>{\raggedleft\arraybackslash}p{#1}}
\renewcommand{\tabstyle}{\small\setlength{\tabcolsep}{5pt}\renewcommand{\arraystretch}{1.12}}
\newcommand{\widetabstyle}{\tabstyle}
\newcommand{\thead}[1]{\textbf{#1}}
\renewcommand{\method}[1]{\textsc{#1}}
\newcommand{\gcp}{\method{GenCoord}\xspace}
\newcommand{\dsl}{\textsc{Short DSL}\xspace}
\newcommand{\artifact}[1]{\texttt{#1}}
\setcounter{topnumber}{4}
\setcounter{bottomnumber}{3}
\setcounter{totalnumber}{7}
\renewcommand{\topfraction}{0.95}
\renewcommand{\bottomfraction}{0.90}
\renewcommand{\textfraction}{0.05}
\renewcommand{\floatpagefraction}{0.82}
\setlength{\floatsep}{7pt plus 2pt minus 2pt}
\setlength{\textfloatsep}{8pt plus 2pt minus 2pt}
\setlength{\intextsep}{8pt plus 2pt minus 2pt}
\captionsetup{font=small,labelfont=bf,labelsep=period,skip=4pt}

\fancypagestyle{appendixstyle}{%
 \fancyhf{}\fancyhead[L]{\small\textsc{GenCoord}\enspace/\enspace Appendix}%
 \fancyhead[R]{\small\thepage}%
 \renewcommand{\headrulewidth}{.35pt}%
 \renewcommand{\footrulewidth}{0pt}}
\pagestyle{appendixstyle}
\thispagestyle{appendixstyle}
\section*{Appendix}
\noindent Sections S1--S4 describe the protocol, tasks, and evaluation units.
S5--S6 report local and prompted coordination, S7--S11 examine implementation and assignment choices,
and S12--S13 provide protocol diagnostics and structural tests.

\section{Protocol Lifecycles and Model Calls}
\label{sec:supp-lifecycle}

This section supports the commitment protocol in main Section 4 by specifying how proposed paths become commitments and how the runtime checks their execution. We provide the complete call sequence and resolver conditions so that the learned and deterministic implementations can be compared against the same semantics.

Figure~\ref{fig:lifecycles} shows where generation occurs in each interaction and where a delivered message can be replaced before resolution. Its final panel compares implementations that share the same checks while using different numbers of model calls.

\begin{figure}[!htbp]
  \centering
  \includegraphics[width=\textwidth]{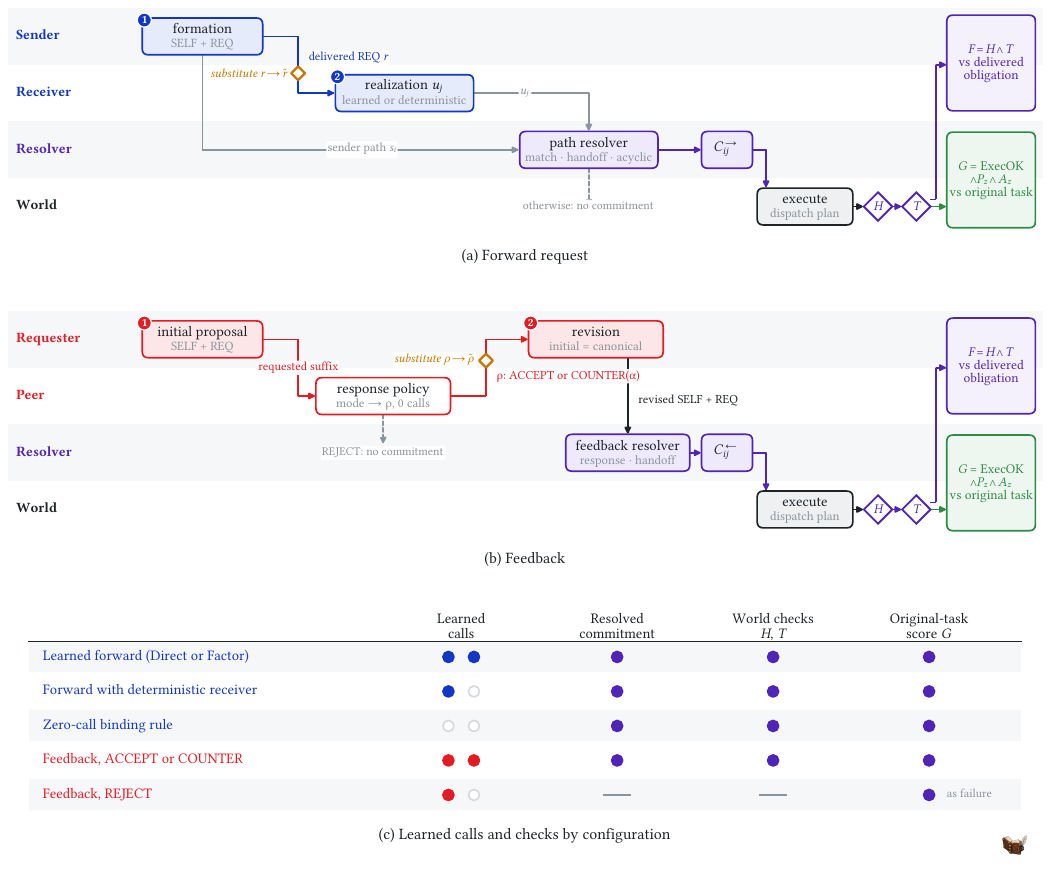}
  \Description{Two swimlane diagrams and a table. In the forward request, the sender forms SELF plus REQ in learned call 1, the delivered request can be substituted by a legal alternative, the receiver realizes it in learned call 2 or deterministically, the path resolver checks match, handoff, and acyclicity, and the resolved commitment is executed, followed by the handoff check H and the terminal check T. Fulfillment F is scored against the delivered obligation and task success G against the original task. In feedback, a deterministic peer policy returns ACCEPT, COUNTER, or REJECT; the response can be substituted; the requester revises in a second learned call when its initial proposal equals the canonical one; REJECT creates no commitment. A table lists learned calls and checks for five configurations.}
  \caption{How a proposal becomes executable joint work. In the forward route (a), the receiver realizes a delivered request through generation or deterministic instantiation before the resolver checks compatibility and dependencies. In the feedback route (b), the deterministic response selects acceptance, a counteroffer, or rejection, and the requester revises an accepted or countered proposal. Physical checks evaluate receipt H and declared completion T, distinguishing fulfillment F from original-task success G. Panel (c) locates model calls and checks across implementations.}
  \label{fig:lifecycles}
\end{figure}

\begin{table}[!htbp]
  \centering
  \caption{Notation for proposals, resolution, dispatch, and world evaluation in Algorithms~\ref{alg:forward}--\ref{alg:feedback}.}
  \label{tab:core-notation}
  \tabstyle
  \begin{tabularx}{\textwidth}{L{3.55cm} Z}
    \toprule
    \thead{Symbol} & \thead{Meaning and domain} \\
    \midrule
    $\mathcal A=\{a_i,a_j\}$ & the two agents; protocol roles are written as \emph{snd}, \emph{rcv}, \emph{req}, and \emph{peer} \\
    $\mathcal K=(\mathcal S,\mathcal T,\Theta,\Gamma)$ & core executable prior: skills $\mathcal S$, shared task DAG $\mathcal T$, schema $\Theta$, and grounding/checker contract $\Gamma$; in the private-structure setting only primitive semantics $\mathcal K_0$ are shared \\
    $\xi_{\mathrm{pub}}$, $\omega_t$, $x_i^t$ & public episode instantiation, current executable world state, and agent $i$'s model-visible local context at round $t$; $\omega_t$ is visible only to the runtime verifier \\
    $\hat y_{\mathrm{snd}},q_{\mathrm{snd}}$; $\hat y_{\mathrm{rcv}},q_{\mathrm{rcv}}$ & raw model strings and parsed sender/receiver proposals; $q=\bot$ denotes parse failure \\
    $r_{\mathrm{snd}\rightarrow\mathrm{rcv}}$, $u_{\mathrm{rcv}}$ & requested peer path and receiver-local realization suffix \\
    $c_{\mathrm{peer}}$, $\rho$, $\alpha$ & peer-local processing mode (data field \texttt{capability}), bounded response, and selected \artifact{counter\_offer\_id} \\
    $N_{\mathcal K}^{\mathrm{path}}$, $N_{\mathcal K}^{\mathrm{prop}}$, $\operatorname{Actorize}$ & canonical normalization for paths/proposals and explicit instantiation of an implicit task actor \\
    $\mathcal C$, $\mathcal D$, $\mathcal E$, $\mathcal F$ & spaces of resolved commitments, dispatch plans, no-commitment reasons, and protocol/runtime failure codes \\
    \bottomrule
  \end{tabularx}
\end{table}

We use the notation in Table~\ref{tab:core-notation} throughout the algorithms, including the distinction between proposed paths and the resolved commitment. Single-step control adds a sender decision after acquisition, whereas the two-call Factor implementation predicts a code at each side and composes the corresponding paths deterministically.

To keep the evaluator's outcome distinct from the state of an individual obligation, we return the following tagged result and report fulfillment alongside original-task success:
\[
\mathcal R=
(\{\mathtt{SUCCESS}\}\times\mathcal C\times\mathcal D)
\;\uplus\;
(\{\mathtt{NO\_COMMITMENT}\}\times\mathcal E)
\;\uplus\;
(\{\mathtt{FAIL}\}\times\mathcal F).
\]
The values $\operatorname{Success}(C,D)$, $\operatorname{NoCommitment}(e)$, and $\operatorname{Fail}(f)$ identify whether a task completes or stops at an earlier stage. Because fulfilling the committed work can still violate the original assignment, a task-failure result may include $H=F=1$ and $G=0$.

\begin{algorithm}[!htbp]
  \caption{Learned Forward Resolution}
  \label{alg:forward}
  \footnotesize
  \begin{algorithmic}[1]
    \Require sender context $x_{\mathrm{snd}}$, receiver context $x_{\mathrm{rcv}}$, prior $\mathcal K$, public episode instantiation $\xi_{\mathrm{pub}}$, executable world $\omega_t$
    \State $\hat y_{\mathrm{snd}} \gets \textsc{Generate}(x_{\mathrm{snd}},\mathcal K,\xi_{\mathrm{pub}})$ \Comment{\artifact{SELF+REQ}}
    \State $q_{\mathrm{snd}} \gets \textsc{Parse}(\hat y_{\mathrm{snd}})$
    \If{$q_{\mathrm{snd}}=\bot$} \State \Return $\textsc{Fail}(\mathtt{PARSE\_FAILURE})$ \EndIf
    \If{$\neg\textsc{ContractCheck}(q_{\mathrm{snd}},x_{\mathrm{snd}},\mathcal K,\xi_{\mathrm{pub}})$}
      \State \Return $\textsc{Fail}(\mathtt{CONTRACT\_REJECT})$
    \EndIf
    \If{$|q_{\mathrm{snd}}.\mathtt{peer\_requests}|\neq1$ \textbf{or} $q_{\mathrm{snd}}.\mathtt{peer\_requests}[0].\mathtt{requested\_plan}=[]$}
      \State \Return $\textsc{Fail}(\mathtt{CONTRACT\_REJECT})$
    \EndIf
    \State $r_{\mathrm{snd}\rightarrow\mathrm{rcv}}\gets q_{\mathrm{snd}}.\mathtt{peer\_requests}[0]$; deliver $r_{\mathrm{snd}\rightarrow\mathrm{rcv}}$
    \State $\hat y_{\mathrm{rcv}} \gets \textsc{Generate}(x_{\mathrm{rcv}},\mathcal K,\xi_{\mathrm{pub}},r_{\mathrm{snd}\rightarrow\mathrm{rcv}})$
    \State $q_{\mathrm{rcv}} \gets \textsc{Parse}(\hat y_{\mathrm{rcv}})$
    \If{$q_{\mathrm{rcv}}=\bot$} \State \Return $\textsc{Fail}(\mathtt{PARSE\_FAILURE})$ \EndIf
    \If{$\neg\textsc{ContractCheck}(q_{\mathrm{rcv}},x_{\mathrm{rcv}},\mathcal K,\xi_{\mathrm{pub}})$ \textbf{or} $q_{\mathrm{rcv}}.\mathtt{self\_plan}=[]$ \textbf{or} $q_{\mathrm{rcv}}.\mathtt{peer\_requests}\neq[]$}
      \State \Return $\textsc{Fail}(\mathtt{CONTRACT\_REJECT})$
    \EndIf
    \State $C \gets \textsc{Resolve}_{\rightarrow}(q_{\mathrm{snd}},q_{\mathrm{rcv}},\mathcal K,\xi_{\mathrm{pub}})$
    \If{$C=\bot$} \State \Return $\textsc{Fail}(\mathtt{RESOLUTION\_CONFLICT})$ \EndIf
    \State $D \gets \textsc{Materialize}(C,\mathcal K,\xi_{\mathrm{pub}})$
    \If{$D=\bot$} \State \Return $\textsc{Fail}(\mathtt{MATERIALIZATION\_FAILURE})$ \EndIf
    \State $\mathtt{status}\gets\textsc{ExecuteAndVerify}(D,\omega_t)$
    \If{$\mathtt{status}\neq\mathtt{SUCCESS}$} \State \Return $\textsc{Fail}(\mathtt{status})$ \EndIf
    \State \Return $\textsc{Success}(C,D)$
  \end{algorithmic}
\end{algorithm}

\begin{algorithm}[!htbp]
  \caption{Feedback resolution from the prescribed initialization}
  \label{alg:feedback}
  \footnotesize
  \begin{algorithmic}[1]
    \Require requester context $x_{\mathrm{req}}$, peer mode $c_{\mathrm{peer}}$, prior $\mathcal K$, public episode instantiation $\xi_{\mathrm{pub}}$, executable world $\omega_t$, prescribed initial proposal $q_0^\star$
    \State $\hat y_0 \gets \textsc{Generate}(x_{\mathrm{req}},\mathcal K,\xi_{\mathrm{pub}})$
    \State $q_0 \gets \textsc{Parse}(\hat y_0)$
    \If{$q_0=\bot$} \State \Return $\textsc{Fail}(\mathtt{PARSE\_FAILURE})$ \EndIf
    \If{$\neg\textsc{ContractCheck}(q_0,x_{\mathrm{req}},\mathcal K,\xi_{\mathrm{pub}})$}
      \State \Return $\textsc{Fail}(\mathtt{CONTRACT\_REJECT})$
    \EndIf
    \If{$|q_0.\mathtt{peer\_requests}|\neq1$ \textbf{or} $q_0.\mathtt{peer\_requests}[0].\mathtt{requested\_plan}=[]$}
      \State \Return $\textsc{Fail}(\mathtt{CONTRACT\_REJECT})$
    \EndIf
    \If{$q_0\neq q_0^\star$} \State \Return $\textsc{Fail}(\mathtt{CONTRACT\_REJECT})$ \EndIf
    \State $r_{\mathrm{req}\rightarrow\mathrm{peer}}\gets q_0.\mathtt{peer\_requests}[0]$
    \State $\rho \gets \textsc{ResponsePolicy}(r_{\mathrm{req}\rightarrow\mathrm{peer}},c_{\mathrm{peer}},\mathcal K,\xi_{\mathrm{pub}})$
    \If{$\neg\textsc{ResponseValid}(\rho,r_{\mathrm{req}\rightarrow\mathrm{peer}},\mathcal K,\xi_{\mathrm{pub}})$}
      \State \Return $\textsc{Fail}(\mathtt{CONTRACT\_REJECT})$
    \EndIf
    \If{$\rho=\textsc{Reject}$}
      \State \Return $\textsc{NoCommitment}(\rho.\mathtt{reason\_code})$
    \EndIf
    \State $\hat y_1 \gets \textsc{Generate}(x_{\mathrm{req}},\mathcal K,\xi_{\mathrm{pub}},q_0,\rho)$ \Comment{ACCEPT or COUNTER}
    \State $q_1 \gets \textsc{Parse}(\hat y_1)$
    \If{$q_1=\bot$} \State \Return $\textsc{Fail}(\mathtt{PARSE\_FAILURE})$ \EndIf
    \If{$\neg\textsc{ContractCheck}(q_1,x_{\mathrm{req}},\mathcal K,\xi_{\mathrm{pub}})$}
      \State \Return $\textsc{Fail}(\mathtt{CONTRACT\_REJECT})$
    \EndIf
    \If{$|q_1.\mathtt{peer\_requests}|\neq1$ \textbf{or} $q_1.\mathtt{peer\_requests}[0].\mathtt{requested\_plan}=[]$}
      \State \Return $\textsc{Fail}(\mathtt{CONTRACT\_REJECT})$
    \EndIf
    \If{$q_1.\mathtt{peer\_requests}[0].\mathtt{target}\neq r_{\mathrm{req}\rightarrow\mathrm{peer}}.\mathtt{target}$}
      \State \Return $\textsc{Fail}(\mathtt{CONTRACT\_REJECT})$
    \EndIf
    \State $C \gets \textsc{Resolve}_{\leftarrow}(q_0,\rho,q_1,\mathcal K,\xi_{\mathrm{pub}})$
    \If{$C=\bot$} \State \Return $\textsc{Fail}(\mathtt{RESOLUTION\_CONFLICT})$ \EndIf
    \State $D \gets \textsc{Materialize}(C,\mathcal K,\xi_{\mathrm{pub}})$
    \If{$D=\bot$} \State \Return $\textsc{Fail}(\mathtt{MATERIALIZATION\_FAILURE})$ \EndIf
    \State $\mathtt{status}\gets\textsc{ExecuteAndVerify}(D,\omega_t)$
    \If{$\mathtt{status}\neq\mathtt{SUCCESS}$} \State \Return $\textsc{Fail}(\mathtt{status})$ \EndIf
    \State \Return $\textsc{Success}(C,D)$
  \end{algorithmic}
\end{algorithm}

The response policy deterministically maps the requested suffix and processing mode to \textsc{Accept}, \textsc{Reject}, or $\textsc{Counter}(\alpha)$, whose alternative is serialized as \artifact{counter\_offer\_id}. Acceptance preserves the proposed branch, while a counteroffer names the admissible change that the requester must incorporate into its revision. Rejection creates no commitment and ends the interaction before a second generation or dispatch. In the controlled workcell diagnostic, exact equality with the prescribed initial proposal gates revision, and the matching proposal joins the delivered response in the next prompt. Table~\ref{tab:response-schema} specifies this contract, while environmental state is accessed by execution and verification rather than by the response policy.

Algorithms~\ref{alg:forward}--\ref{alg:feedback} show how parsing, commitment checks, materialization, and execution compose into a complete interaction. The verifier returns \artifact{SUCCESS}, \artifact{EXECUTION\_FAILURE}, \artifact{HANDOFF\_FAILURE}, or \artifact{TERMINAL\_FAILURE}, so an admitted object can be distinguished from the physical outcome of carrying it out.

\begin{table}[!htbp]
  \centering
  \caption{Response options and validation outcomes. Each option identifies the information returned to the requester and the action permitted by the protocol.}
  \label{tab:response-schema}
  \tabstyle
  \begin{tabularx}{\textwidth}{L{2.45cm} Z Z}
    \toprule
    \thead{Decision / stage} & \thead{Valid fields or condition} & \thead{Protocol consequence} \\
    \midrule
    \textsc{Accept} & reason code; no \artifact{counter\_offer\_id} & second requester call preserves the normalized initial branch \\
    $\textsc{Counter}(\alpha)$ & reason code and exactly one admissible \artifact{counter\_offer\_id}=$\alpha$ & second requester call realizes the selected bounded branch \\
    \textsc{Reject} & reason code; no counter alternative & creates no commitment; no second generation or dispatch \\
    Invalid response & unknown decision, forbidden field, missing alternative, or inadmissible ID & \artifact{CONTRACT\_REJECT} \\
    \midrule
    Parse & raw string cannot be decoded & \artifact{PARSE\_FAILURE} \\
    Contract & schema, role, skill, argument, target, or response check fails & \artifact{CONTRACT\_REJECT} \\
    Resolution & realization, boundary fields, branch, or dependency merge conflicts & \artifact{RESOLUTION\_CONFLICT} \\
    Materialization & validated commitment cannot instantiate a dispatch & \artifact{MATERIALIZATION\_FAILURE} \\
    World discharge & runtime action, inventory handoff, or terminal predicate fails & \artifact{EXECUTION\_FAILURE}, \artifact{HANDOFF\_FAILURE}, or \artifact{TERMINAL\_FAILURE} \\
    \bottomrule
  \end{tabularx}
\end{table}

The algorithms use generation to realize the received work, but deterministic reception replaces that step with validation and instantiation of the same request. During message substitution, admission follows the delivered request while task evaluation keeps the original specification, as detailed in Section~\ref{sec:supp-access-new}. The GOAL adapter instead searches for a continuation and checks it against the transmitted target and constraints, preserving the distinction between a proposed solution and the outcome it promises.

\subsection{Resolver Semantics}

The resolved object $C=(S_i,S_j,\beta,P_C)$ contains the operative paths, assignment, and ordering that the materializer needs, while the proposal and response identify how the commitment was reached. The codebook identifier $b$ denotes assignment $\beta$. As in main Eq.~(1), delivery is owed by the sender and completion becomes due after verified receipt~\cite{singh2008commitments,chopra2025azorus}. We evaluate fulfillment as $F(C,\xi)=H_\beta(\xi)\land T_\beta(\xi)$ and task success as $G_z(\xi)=\operatorname{ExecOK}(\xi)\land P_z(\omega_T)\land A_z(\xi)$, keeping the original outcome and allocation predicates fixed when a message changes. A receipt remains valid after its materials are consumed, whereas a rejected proposal creates no physical fulfillment observation.

Before comparing paths, $\operatorname{Actorize}(z,a)$ fills every implicit actor in path $z$ with agent $a$, making the division of work explicit. We then apply $N_{\mathcal K}^{\mathrm{path}}$ or $N_{\mathcal K}^{\mathrm{prop}}$ to normalize aliases, typed defaults, grounded references, assignments, actors, and predecessor semantics.

\paragraph{Forward resolution.}
Define
\[
S_{\mathrm{snd}}=N_{\mathcal K}^{\mathrm{path}}\!\left(\operatorname{Actorize}(q_{\mathrm{snd}}.\mathtt{self\_plan},\mathrm{snd})\right),
\]
\[
R_{\mathrm{rcv}}=N_{\mathcal K}^{\mathrm{path}}\!\left(\operatorname{Actorize}(r_{\mathrm{snd}\rightarrow\mathrm{rcv}}.\mathtt{requested\_plan},\mathrm{rcv})\right),
\quad
U_{\mathrm{rcv}}=N_{\mathcal K}^{\mathrm{path}}\!\left(\operatorname{Actorize}(q_{\mathrm{rcv}}.\mathtt{self\_plan},\mathrm{rcv})\right).
\]
Under the path contract, the receiver realizes the request exactly when
\begin{equation}
\operatorname{Realizes}_{\mathcal K}(U_{\mathrm{rcv}},R_{\mathrm{rcv}})
\iff U_{\mathrm{rcv}}=R_{\mathrm{rcv}}.
\label{eq:realizes}
\end{equation}
This equality compares the ordered tasks and their skills, values, actors, locations, and dependencies after normalization, so accepted aliases and argument ordering may differ. The realized path contains the requested tasks, while the executor adds its deterministic bookkeeping after resolution.

We write $\operatorname{HandoffCompatible}_{\mathcal K,\xi_{\mathrm{pub}}}\allowbreak(S_i,S_j)$ for matching of the producing and receiving actors, item, quantity, locations, and dependency under the public task instantiation. The forward resolver admits $C^{\rightarrow}=(S_{\mathrm{snd}},U_{\mathrm{rcv}},\beta,P_C)$ when the receiver is the request's target, its \artifact{SELF} path is nonempty, its outbound request is empty, Equation~\ref{eq:realizes} holds, and the handoff is compatible. It also requires a complete acyclic merge, preserving each agent's task order and connecting the handoff producer to the receiver's first consuming task. These checks establish an executable commitment, while physical execution establishes receipt. Missing fields, incompatible transfers, multiple selected branches, and cycles return $\bot$.

\paragraph{Feedback resolution.}
Feedback resolution first determines which branch the response permits. Let $\mathcal B_{\mathcal K}$ contain the public branches, let $B$ extract a normalized branch from a legal proposal, and let $\operatorname{AltBranch}_{\mathcal K,\xi_{\mathrm{pub}}}$ map an admissible alternative identifier to its branch. The response selects:
\[
B_{\mathcal K,\xi_{\mathrm{pub}}}^{\mathrm{resp}}(\rho,q_0)=
\begin{cases}
B(q_0), & \rho=\textsc{Accept},\\
\operatorname{AltBranch}_{\mathcal K,\xi_{\mathrm{pub}}}(\alpha), & \rho=\textsc{Counter}(\alpha),\\
\bot, & \rho=\textsc{Reject}.
\end{cases}
\]
After acceptance, the requester may change surface expression while preserving the initial proposal's normalized branch, peer, and goal. After $\textsc{Counter}(\alpha)$, it preserves the peer and goal but replaces the operative branch with the named alternative. Rejection ends processing without creating a commitment.

For accepted and countered proposals, define
\[
S_{\mathrm{req}}=N_{\mathcal K}^{\mathrm{path}}\!\left(\operatorname{Actorize}(q_1.\mathtt{self\_plan},\mathrm{req})\right),
\]
\[
R_{\mathrm{peer}}=N_{\mathcal K}^{\mathrm{path}}\!\left(\operatorname{Actorize}(q_1.\mathtt{peer\_requests}[0].\mathtt{requested\_plan},\mathrm{peer})\right).
\]
$\textsc{Resolve}_{\leftarrow}$ requires
\[
B(q_1)=B_{\mathcal K,\xi_{\mathrm{pub}}}^{\mathrm{resp}}(\rho,q_0),
\]
$\operatorname{HandoffCompatible}_{\mathcal K,\xi_{\mathrm{pub}}}(S_{\mathrm{req}},R_{\mathrm{peer}})$, and a complete acyclic merge $P_C$, returning $C^{\leftarrow}=(S_{\mathrm{req}},R_{\mathrm{peer}},\beta,P_C)$ when those conditions hold.
The materializer uses the resolved paths to construct $D=\textsc{Materialize}(C,\mathcal K,\xi_{\mathrm{pub}})$, instantiating task actions, actors, locations, dependencies, and their attached checks. Because these are task-level operations, materialization precedes compilation into navigation, inventory transfers, and other Mineflayer execution parameters.

\section{Commitment Schema and Execution}
\label{sec:supp-runtime}

This section supports the representation and execution mechanism in main Section 4 by giving the complete task grammar and the checks attached to each stage. The runtime example also shows how a compact proposal becomes a longer sequence of grounded skills.

\Needspace{5\baselineskip}
\subsection{Short DSL Grammar}

\begin{lstlisting}[style=gcpdsl]
program   ::= self_line NEWLINE req_line
self_line ::= "SELF -" | "SELF " path
req_line  ::= "REQ -" | "REQ " agent " " path
path      ::= task (" > " task){0,4}
task      ::= skill "(" [args] ")"
args      ::= arg ("," arg)*
arg       ::= key "=" value
key       ::= ident
value     ::= json_string | json_number
            | "true" | "false" | "null" | atom
skill     ::= control.wait | resource.obtain
            | resource.deliver | craft.item
            | transform.supply_input
            | transform.supply_fuel
            | transform.collect_output
            | build.component
ident     ::= [A-Za-z][A-Za-z0-9_]*
agent     ::= atom
atom      ::= [A-Za-z0-9_.:/+-]+
json_number ::= JSON numeric literal
json_string ::= RFC 8259 double-quoted string
\end{lstlisting}

We define a canonical grammar while allowing compatible spellings to normalize to the same task object, because the contract concerns the specified work rather than argument order. Values are recognized as quoted JSON strings, complete JSON numbers, reserved literals, and then unquoted ASCII atoms, with escapes decoded before type checking. Quoted strings are scanned atomically, including embedded punctuation and escaped quotes. The serializer uses \artifact{q}, \artifact{item}, \artifact{input}, \artifact{to}, \artifact{dst}, \artifact{bind}, \artifact{site}, \artifact{from}, \artifact{dst\_role}, \artifact{loc}, and \artifact{station}. The parser also canonicalizes the aliases \artifact{count}/\artifact{quantity}, \artifact{target\_agent\_ref}/\artifact{to\_agent\_id}, \artifact{destination}, \artifact{destination\_role}, \artifact{from\_agent\_id}, \artifact{location\_ref}, and \artifact{binding\_id}. Contract checks validate fields and types after parsing, rejecting duplicate keys, unsupported paths, invalid atoms, trailing text, and empty requested continuations.

Each nonempty path contains one to five tasks with positive integer quantities, grounded location references, and an optional public branch identifier. The fields \artifact{to} and \artifact{bind} name the peer and branch, while \artifact{dst}, \artifact{site}, \artifact{loc}, and \artifact{station} supply grounding references. The separator \artifact{>} proposes precedence for resolution, and \artifact{REQ -} produces an empty \artifact{peer\_requests} list. Parsing returns $\bot$ for malformed text, whereas a parsed object can subsequently receive \artifact{CONTRACT\_REJECT}.

We distinguish a syntactically admissible string from an object that satisfies the task contract, since the two checks answer different questions. Let $\mathcal Y_{\mathrm{parse}}$ contain admissible strings, $\mathcal Q_{\mathrm{parsed}}$ parsed proposals, and $\mathcal Q_{\mathrm{legal}}(\mathcal K,\xi_{\mathrm{pub}})$ proposals passing the typed contract. Parsing maps into $\mathcal Q_{\mathrm{parsed}}\cup\{\bot\}$, giving
\[
\mathcal Y_{\mathrm{legal}}(\mathcal K,\xi_{\mathrm{pub}})
=\{y\in\mathcal Y_{\mathrm{parse}}:\operatorname{Parse}(y)\neq\bot,\;
\operatorname{ContractCheck}_{\mathcal K,\xi_{\mathrm{pub}}}(\operatorname{Parse}(y))=\mathtt{true}\}.
\]
For a legal string, decoding applies proposal normalization after parsing: $\operatorname{Decode}_{\mathcal K,\xi_{\mathrm{pub}}}(y)=N_{\mathcal K}^{\mathrm{prop}}(\operatorname{Parse}(y))$. We encode legal objects as canonical strings in $\mathcal Y_{\mathrm{canon}}(\mathcal K)$, defining equivalence between two legal strings by
\begin{equation}
y_1\equiv_{\mathrm{sem}}y_2
\iff
\operatorname{Decode}_{\mathcal K,\xi_{\mathrm{pub}}}(y_1)
=
\operatorname{Decode}_{\mathcal K,\xi_{\mathrm{pub}}}(y_2).
\end{equation}
We check the following round-trip invariant for every legal proposal $Q$, which states that encoding and decoding preserve its normalized task meaning:
\begin{equation}
\operatorname{Decode}_{\mathcal K,\xi_{\mathrm{pub}}}(\operatorname{Encode}_{\mathcal K}(Q))
=
N_{\mathcal K}^{\mathrm{prop}}(Q).
\end{equation}
The invariant preserves the complete ordered paths, target peer, typed arguments, actors, grounding references, and dependencies, allowing serialization to change without changing the work resolved by the protocol.

\begin{lstlisting}[style=gcpdsl]
SELF resource.obtain(q=8,item=oak_planks) >
     resource.deliver(q=8,item=oak_planks,to=agent_a)
REQ agent_a craft.item(bind=WORK_BRANCH,q=1,
     input=oak_planks,item=crafting_table)
\end{lstlisting}

Table~\ref{tab:contract-checks} connects each validation step to the evidence it establishes, distinguishing commitment checks before dispatch from the observations required for physical fulfillment.

\begin{table}[!htbp]
  \centering
  \caption{Validation and physical verification through the execution pipeline. The checks establish an admissible commitment before evaluating transfer and completion.}
  \label{tab:contract-checks}
  \tabstyle
  \begin{tabularx}{\columnwidth}{L{2.35cm} L{2.25cm} Z}
    \toprule
    \thead{Stage} & \thead{Check} & \thead{Accepted condition} \\
    \midrule
    Pre-materialization & Syntax and identity & two-line grammar; actors and peer match the role instance \\
     & Skill and arguments & catalog path, required keys, types, and quantities are valid \\
     & Binding and handoff & branch is admissible; item, quantity, sender, and receiver agree \\
     & Dependency merge & selected predecessor graph is complete and acyclic \\
    Post-materialization & Static plan consistency & resolved actors, fields, handoff edge, and checker attachment match the validated commitment \\
    Post-execution & World verification & handoff event and declared terminal predicate both pass \\
    \bottomrule
  \end{tabularx}
\end{table}

After admission, the materializer combines the selected assignment with the shared task structure, and the compiler produces the corresponding Mineflayer 4.37.1 skills~\cite{prismarinejs2026mineflayer}. We execute these skills on Minecraft Java 1.21.4, where inventories and terminal objects provide the readback used to check fulfillment.

Figure~\ref{fig:runtime-trace} illustrates this expansion for an Active Branch task: two model calls produce thirteen dispatched skills and a verified handoff followed by the declared outcome. The episode takes 22.36 s, showing how the time spent forming a compact commitment fits within complete environmental execution.

\begin{figure}[!htbp]
  \centering
  \includegraphics[width=\textwidth]{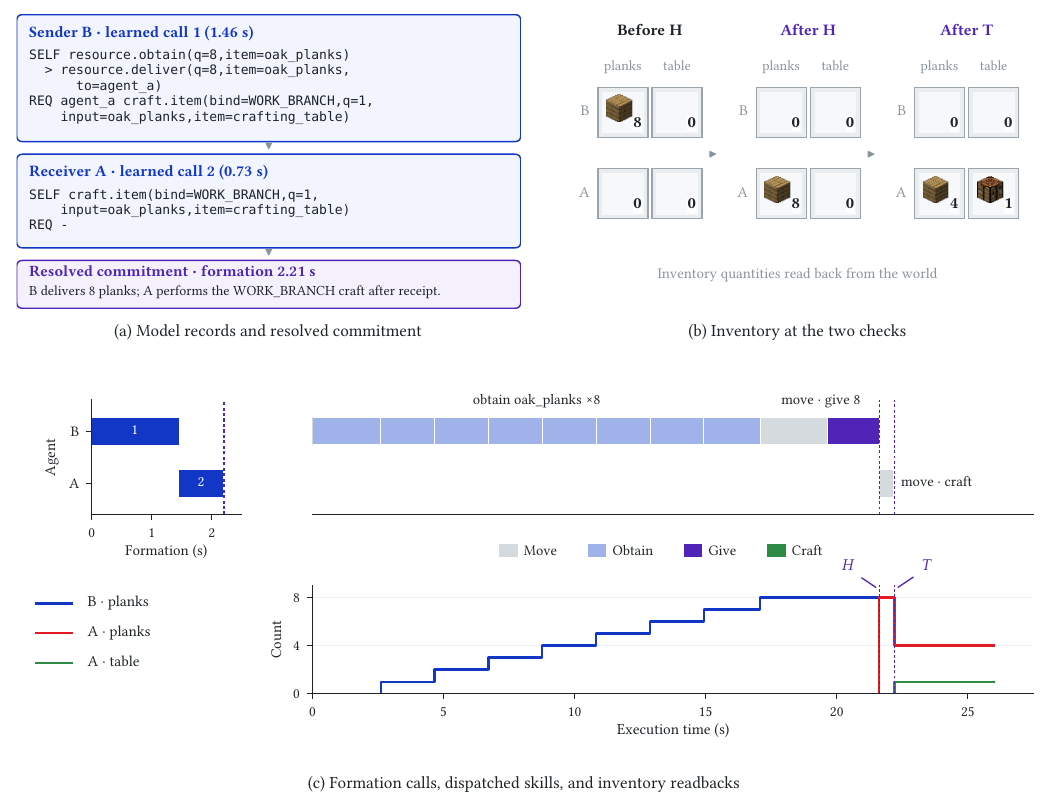}
  \Description{Three panels for one logged episode. Panel a shows the sender and receiver Short DSL records and the resolved commitment. Panel b shows inventory slots before the handoff, after the handoff, and at the terminal check: B holds eight oak planks, then A holds eight, then A holds one crafting table and four planks. Panel c shows two formation calls and thirteen dispatched skills on a time axis, with inventory counts after each dispatch and dashed lines at the handoff and terminal checks.}
  \caption{A complete Short DSL episode. Panel (a) follows sender B's proposal and receiver A's realization to commitment formation at 2.21 s. Panel (b) shows server-read inventories before and after transfer and at completion, omitting tool items. Panel (c) aligns the two generation calls with thirteen dispatched skills and inventory readbacks, with dashed lines locating the handoff H and terminal T checks.}
  \label{fig:runtime-trace}
\end{figure}

\section{Tasks and Training}
\label{sec:supp-tasks-training}

This section supports the local-model comparisons in main Section 5 by describing how task assignments become supervision for both participants. We give the task families and training budgets because representation and horizon controls use different prediction targets and numbers of decisions.

\subsection{Tasks and Supervision}
\label{sec:supp-data}

Table~\ref{tab:task-inventory} identifies which part of the requested work changes between the two assignments in each template. Table~\ref{tab:supervision} then connects those assignments to reciprocal training targets for the sender's proposal and the receiver's continuation.

\begin{table}[!htbp]
  \centering
  \caption{The eight task templates and the assignment fields that distinguish their alternatives.}
  \label{tab:task-inventory}
  \widetabstyle
  \setlength{\tabcolsep}{2.0pt}
  \begin{tabularx}{\textwidth}{L{1.8cm} L{2.4cm} L{1.8cm} L{3.1cm} L{3.4cm} Z}
    \toprule
    \thead{Family} & \thead{Template} & \thead{Source material} & \thead{Balanced bindings} & \thead{Binding-dependent request path} & \thead{Changed field} \\
    \midrule
    Destination & Build Site A/B & oak planks & \artifact{SITE\_A}/\artifact{SITE\_B} & \artifact{build.component} & destination \\
    Destination & Chest A/B & cobblestone & \artifact{CHEST\_A}/\artifact{CHEST\_B} & \artifact{resource.deliver} & destination \\
    Recipe & Chest/Crafting Table & oak planks & \artifact{CHEST}/\newline\artifact{CRAFTING\_TABLE} & \artifact{craft.item} & post-handoff transform; output \\
    Recipe & Planks/Sticks & oak planks & \artifact{PLANKS}/\newline\artifact{STICKS} & \artifact{craft.item}/\newline\artifact{resource.deliver} & post-handoff transform; output \\
    Allocation & Deposit/Build & cobblestone & \artifact{BUILD}/\artifact{DEPOSIT} & \artifact{build.component}/\newline\artifact{resource.deliver} & assigned remaining work \\
    Allocation & Dual Build & oak planks & \artifact{ROOF}/\artifact{WALL} & \artifact{build.component} & assigned component \\
    Active branch & Active Order & oak planks & \artifact{STORAGE}/\artifact{WORK} & \artifact{craft.item} & downstream continuation \\
    Active branch & Remaining Terminal & oak planks & \artifact{DEPOT}/\newline\artifact{MARKER} & \artifact{build.component}/\newline\artifact{resource.deliver} & downstream continuation \\
    \bottomrule
  \end{tabularx}
\end{table}

Every template balances two admissible assignments and requires a verified transfer followed by a task-specific outcome. For example, both Recipe templates transfer \artifact{oak\_planks}, but they change the receiver's subsequent transformation and product. The sender observes the active assignment, whereas the receiver sees a matched context until the request arrives. This construction supports a check of whether the delivered request selects the intended continuation, with the information bound derived in Section~\ref{sec:supp-ambiguity}.

\begin{table}[!htbp]
  \centering
  \caption{Reciprocal supervision for the sender proposal and receiver continuation in the forward interaction.}
  \label{tab:supervision}
  \tabstyle
  \begin{tabular}{@{}L{3.9cm} l r@{}}
    \toprule
    \thead{Record stage} & \thead{Target object} & \thead{Rows} \\
    \midrule
    Sender proposal & multi-step \artifact{SELF} plus peer \artifact{REQ} & 960 \\
    Receiver after request & downstream \artifact{SELF}; empty \artifact{REQ} & 960 \\
    Intermediate sender (single-step) & next local action after obtain & 960 \\
    \bottomrule
  \end{tabular}
\end{table}

We instantiate supervision from task specifications, sampled local facts, and canonical executable objects, using Minecraft resources to define skills and scenarios. The builder validates every target against its intended assignment and terminal predicate, while keeping the target distinct from the context presented to the model.

\subsection{Training Protocols}
\label{sec:supp-training}

Tables~\ref{tab:training-budgets} and~\ref{tab:training-provenance} describe the models, stage counts, optimization budgets, and checkpoint rules that produce the compared implementations. The budgets make explicit how each communication or control protocol is supervised.

\begin{table}[!htbp]
  \centering
  \caption{Training examples and optimization updates per checkpoint for each prediction and control protocol.}
  \label{tab:training-budgets}
  \widetabstyle
  \begin{tabularx}{\textwidth}{L{3.6cm} r Z r r r}
    \toprule
    \thead{Model or condition} & \thead{Rows} & \thead{Stage counts} & \thead{Epochs} & \thead{Steps} & \thead{Seeds} \\
    \midrule
    Multi-step / direct / surface & 1,920 & 960 sender + 960 receiver & 2 & 240 & 3 \\
    Self-plan-only & 1,920 & 960 sender + 960 receiver & 2 & 240 & 3 \\
    Matched-SFT free-form & 2,400 & 1,200 sender + 1,200 receiver & 2 & 300 & 3 \\
    Single-step & 2,880 & 960 sender + 960 intermediate sender + 960 receiver & 2 & 360 & 3 \\
    Factor backend & 1,920 & 960 sender code + 960 receiver code & 2 & 240 & 3 \\
    \bottomrule
  \end{tabularx}
\end{table}

\begin{table}[!htbp]
  \centering
  \caption{Model configuration and supervision budgets for the horizon comparison.}
  \label{tab:training-provenance}
  \widetabstyle
  \begin{tabularx}{\textwidth}{L{3.0cm} Z Z}
    \toprule
    \thead{Field} & \thead{Multi-step} & \thead{Single-step} \\
    \midrule
    \multicolumn{3}{@{}l}{\textbf{(a) Common configuration}} \\
    Base model & \multicolumn{2}{l}{Qwen3.5-0.8B \cite{qwenteam2026qwen35}} \\
    Optimizer & \multicolumn{2}{l}{AdamW; LR $10^{-5}$; weight decay 0.01; cosine schedule; warm-up 0.03; grad-norm 1.0} \\
    Batch / precision & \multicolumn{2}{l}{micro 4; accumulation 4; effective 16; bfloat16 on A100 40GB} \\
    Length / decoding & \multicolumn{2}{l}{maximum 2,048; deterministic decoding; structured cap 256 tokens} \\
    Objective / selection & \multicolumn{2}{l}{target-and-EOS causal loss; final-step checkpoint; DEV reserved for implementation checks} \\
    \midrule
    \multicolumn{3}{@{}l}{\textbf{(b) Model-specific training budgets}} \\
    Training rows / updates & 1,920 / 240 & 2,880 / 360 \\
    Seed IDs & \artifact{Seeds 1--3} & \artifact{Seeds 1--3} \\
    \bottomrule
  \end{tabularx}
\end{table}

\FloatBarrier

\section{Evaluation Units and Information}
\label{sec:supp-evaluation}

This section supports the comparisons in main Sections 5--6 by identifying the independent evaluation units and the information available when each component acts. These definitions distinguish transfer across tasks, wording, and implementations while keeping their respective outcome measures interpretable.

\subsection{Splits and Statistical Units}
\label{sec:supp-splits-new}

Table~\ref{tab:split-map-new} pairs each comparison with its supplied information and evaluation unit, so repeated roles or checkpoints remain distinguishable from independent task instances.

\begin{table}[!htbp]
\centering\tabstyle
\caption{Evaluation design by comparison. Each row states what changes, which information remains available, and which units support inference.}
\label{tab:split-map-new}
\begin{tabularx}{\textwidth}{@{}L{2.4cm}ZZL{3.0cm}@{}}
\toprule
Pool & Changed or held out & Retained information & Counts / unit\\
\midrule
Core requests & New instances within eight templates & Same skill ontology, task schema and public candidates & 160 semantic clusters; roles and three seeds nested\\
Template horizon & Four whole templates absent from training & Same four families and skill ontology; task structure supplied & 80 clusters, 480 evaluations per learned method\\
Surface/backend & Separately sampled instances, within known template families & Same shared task structure and execution path & 128 quality clusters; independent 40-cluster timing samples\\
Workcell feedback & New instances of the four goal--mode cells & Same two goals and two modes; prescribed initial proposal & 40 clusters; 20 paired mode blocks\\
Description wording & Lexical roots and paired distractors & Sixteen semantic choices seen in training & 32 shared original phrasings; 320 role rows per condition\\
API request transfer & New paired blocks and another formation backend & Public task structure and four worked examples & 80 pairs; 960 assignments per backend/condition\\
Receiver comparison & Receiver realization implementation & Same 1,920 actual sender outputs & Formation replay; seven additional physical objects; timing on 16 blocks\\
Private structure & Sender keeps task DAG; parameterized instances & Shared primitives and explicit transmitted boundary & 114 parameter specifications within 24 families; 96-spec sender probe\\
Stage--order & Presentation of the same graph nodes & Same 96 source clusters, targets and checkpoints & Four paired views; 576 evaluations each\\
Online descriptions & Full formation, wire, and Minecraft execution & Same language sources and shared executable interface & 10,240 quality trials on 32 phrasings; 960 timed trials on 16 phrasings\\
Prospective selection & New descriptions and score-based routing & Sixteen task choices, fixed selectors and receiver & 640 test descriptions; 17,920 quality trials; 1,024 timed trials on 64 descriptions\\
\bottomrule
\end{tabularx}
\end{table}

The horizon test holds out Handoff Station A/B, Smelt Output Binding, Fuel-or-Input Allocation, and Partial-Progress Branch. Syntactic recombination evaluates new expressions through execution, whereas the paraphrase and distractor planning tests measure assignment selection before realization. API repeats sample calls to one service and local repeats sample trained checkpoints, with full online language results reported in Section~\ref{sec:supp-online-descriptions}.

\paragraph{Uncertainty.}
We average repeated roles and checkpoints within the relevant cluster before paired bootstrap inference~\cite{efron1994bootstrap}, using 10,000 percentile resamples. Workcell sensitivity additionally pairs the two modes for a fixed goal and world, while source-level language inference averages within each original phrasing. Exact-binomial intervals use two-sided 95\% Clopper--Pearson coverage. The prospective comparison resamples 640 descriptions within 16 choices and uses 97.5\% intervals for its two primary effects, with 95\% intervals for secondary contrasts. Sections~\ref{sec:supp-language-source} and~\ref{sec:supp-online-descriptions} explain how source stratification changes the language estimand.

\subsection{Information Access and Task Predicates}
\label{sec:supp-access-new}

Table~\ref{tab:access-map-new} distinguishes the inputs available to a model from the references used by validation and evaluation, making the timing of each information use explicit.

\begin{table}[!htbp]
\centering\tabstyle
\caption{Information available when a component forms, validates, or evaluates work. The columns distinguish decision inputs from validation and task references.}
\label{tab:access-map-new}
\begin{tabularx}{\textwidth}{@{}L{3.25cm}ZZ@{}}
\toprule
Component & Decision input & Reference or implementation boundary\\
\midrule
Forward sender & Sender-local fact and shared task structure & Model input contains the local task fact; evaluation uses the target after generation\\
Forward receiver & Local post-handoff context and delivered request & The expected object is built from the delivered request; raw-output consistency is checked before admission\\
Workcell initial requester & Private goal and public alternatives & Peer mode is absent; exact match to the prescribed initial object gates revision\\
Workcell revision & Local context, initial-proposal history, delivered response & History stores the canonical initial object only after equality with the actual parsed initial output\\
Deterministic response policy & Requested suffix and peer-local processing mode & Deterministic finite alternatives; no language-model call\\
Core materializer & Admitted binding and shared task structure & Instantiates task actions and dependencies from the admitted binding and shared graph\\
Private-structure receiver & Shared primitive semantics, local grounding, actual message & No private task DAG or hidden task-specific candidate table\\
World evaluator & Actual trace, original goal/mode, declared obligation & Scores task success and fulfillment from the fixed specification and execution trace\\
Naturalized lexical baselines & Local text and public template identifier & TRAIN labels constrain retrieval/classification; the oracle rule is a different reference\\
Centralized quality reference & Both current local views and shared prior & Quality reference with complete task information\\
\bottomrule
\end{tabularx}
\end{table}

\paragraph{Workcell Revision Gate.}
In the controlled workcell diagnostic, we invoke revision when the parsed initial proposal exactly matches the prescribed object, which then enters the history as \texttt{INITIAL\_PROPOSAL\_DELIVERED}. The second prompt combines that history with the delivered response and local context, while the final reference is applied after generation. Initial failures remain unsuccessful assigned trials.

\paragraph{Delivered content and evaluation references.}
Forward evaluation preserves the original proposal, delivered request, receiver output, and admitted object because they serve different purposes. Request consistency can reject a syntactically legal alternative that contradicts the delivered request, as the examples in Section~\ref{sec:supp-api} illustrate. The runtime executes the admitted assignment, while task evaluation continues to use the original requirement, allowing discrepancies to be localized to generation, admission, or environmental execution.

\paragraph{Workcell task predicates.}
The workcell mode assigns the crafting actor and the item that must be handed over, although either agent can physically craft the product. We require exactly one verified handoff, one task-relevant craft by the assigned actor, and the specified product and count at the destination. As a result, the two processing modes have disjoint acceptable routes even when both produce the same terminal item, which supplies the paired-information diagnostic of Section~\ref{sec:supp-ambiguity}.

\paragraph{Forward task predicates.}
The forward templates change destinations, products after transfer, storage or construction choices, assigned wall or roof components, or downstream terminals. We compare the admitted assignment with the original task and verify its physical consequences, using a shared candidate graph in the core tests. In the private-structure tests, the message supplies the continuation that the receiver combines with its shared primitives.

\section{Communication and Commitment Horizon}
\label{sec:supp-cost-horizon}

This section supports main Sections 6.1--6.2 by comparing communication costs at matched task success and the effects of a longer executable commitment. Complete distributions and family-level results complement the aggregate comparisons in the main paper.

\subsection{Communication Surfaces}
\label{sec:supp-quality}

Table~\ref{tab:representation-results} compares the three surfaces at matched observed quality, while Table~\ref{tab:reliability-latency} gives call latency and the reliability of the full quality pool. Together they distinguish communication cost from the execution shared by the protocols.

The JSON control is trained to emit \artifact{self\_plan} and \artifact{peer\_requests}, preserving descriptive nested keys such as \artifact{task\_path}, \artifact{arguments}, \artifact{requested\_plan}, and \artifact{depends\_on}. Its training targets use sorted keys on one line with no optional whitespace. Short DSL expresses the corresponding tasks as ordered skill paths and argument lists, using aliases such as \artifact{q} and \artifact{to}. We compare generated outputs and measured TTC for these concrete encodings.

\begin{table}[!htbp]
  \centering
  \caption{Communication surfaces at matched task success. Quality uses 128 clusters, while calls, tokens, peer payload, and TTC use the independent 40-cluster surface timing sample.}
  \label{tab:representation-results}
  \widetabstyle
  \setlength{\tabcolsep}{4pt}
  \renewcommand{\arraystretch}{1.18}
  \begin{tabular}{@{}l r r r r r r r@{}}
    \toprule
    \multirow{2}{*}{\thead{Method}} & \multirow{2}{*}{\thead{Success}} & \multirow{2}{*}{\thead{Calls}} & \multicolumn{2}{c}{\thead{Tokens}} & \multirow{2}{*}{\thead{Wire B}} & \multicolumn{2}{c}{\thead{TTC (s)}} \\
    \cmidrule(lr){4-5}\cmidrule(l){7-8}
    & & & \thead{Input} & \thead{Output} & & \thead{p50} & \thead{p95} \\
    \midrule
    \dsl GenCoord (0.8B) & 100.0\% & 2.0 & 949.1 & 78.6 & 76.3 & 2.27 & 2.62 \\
    JSON GenCoord (0.8B) & 100.0\% & 2.0 & 1,154.7 & 267.6 & 303.4 & 7.07 & 7.86 \\
    Controlled free-form +\newline structured commit (0.8B) & 100.0\% & 2.5 & 1,585.2 & 265.8 & 1,061.1 & 7.13 & 7.86 \\
    \bottomrule
  \end{tabular}
\end{table}

\begin{table}[!htbp]
  \centering
  \caption{Reliability and call latency for each communication surface. Reliability uses the 128-cluster quality pool and latency uses the 40-cluster timing sample.}
  \label{tab:reliability-latency}
  \tabstyle
  \renewcommand{\arraystretch}{1.18}
  \begin{minipage}[t]{0.47\textwidth}
    \vspace{0pt}
    \textbf{(a) Call latency: 40-cluster sample}\par\vspace{3pt}
    \begin{tabular}{@{}l r r@{}}
      \toprule
      \thead{Surface} & \thead{p50 (s)} & \thead{p95 (s)} \\
      \midrule
      \dsl & 1.14 & 1.66 \\
      JSON & 3.53 & 5.94 \\
      Controlled free-form & 3.27 & 3.75 \\
      \bottomrule
    \end{tabular}
  \end{minipage}\hfill
  \begin{minipage}[t]{0.49\textwidth}
    \vspace{0pt}
    \textbf{(b) Reliability: 128-cluster quality pool}\par\vspace{3pt}
    \begin{tabular}{@{}l r@{}}
      \toprule
      \thead{Measure} & \thead{Value} \\
      \midrule
      Strict successful clusters, each surface & 128/128 \\
      Exact-binomial 95\% lower bound & 97.2\% \\
      Role-swap success difference & 0.0 pp \\
      \bottomrule
    \end{tabular}
  \end{minipage}
\end{table}

The surface timing sample crosses 40 clusters with two roles and three checkpoints per surface, running resident models sequentially on the same device. Its instances are distinct from the backend sample in Section~\ref{sec:supp-backend}. TTC ends when the bilateral commitment is available, while the subsequent executor time stays near a 22.6 s median across surfaces.

Table~\ref{tab:cost-bootstrap} reports uncertainty for the communication-cost contrasts using paired task clusters. The local comparison uses 40 clusters, and the 64 interleaved DeepSeek pairs span 29 clusters; the intervals use 10,000 bootstrap draws.

\begin{table}[!htbp]
\centering\tabstyle
\caption{Cost contrasts relative to Short DSL. Reductions compare Short DSL with the natural-language protocol; ratios divide the named protocol by Short DSL.}
\label{tab:cost-bootstrap}
\begin{tabular}{@{}llrr@{}}
\toprule
Backend & Contrast & Estimate & 95\% CI\\\midrule
Local 0.8B & Peer-traffic reduction (\%) & 92.8 & $[92.7,93.0]$\\
 & Median TTC reduction (\%) & 68.2 & $[67.9,68.7]$\\
 & JSON output-token ratio & 3.40 & $[3.37,3.44]$\\
 & Free-form output-token ratio & 3.38 & $[3.34,3.42]$\\
 & JSON median-TTC ratio & 3.12 & $[3.08,3.16]$\\
 & Free-form median-TTC ratio & 3.15 & $[3.11,3.20]$\\\addlinespace
DeepSeek & Mean TTC reduction (\%) & 30.4 & $[22.9,38.6]$\\
 & Median TTC reduction (\%) & 20.8 & $[14.4,32.9]$\\
\bottomrule
\end{tabular}
\end{table}

\subsection{Commitment Horizon}
\label{sec:supp-horizon}

Table~\ref{tab:horizon} reports the aggregate paired effects and their family decomposition, which explain the success and decision differences plotted in main Figure~4(c,d).

\begin{table}[!htbp]
  \centering
  \caption{Multistep versus single-step control on 80 held-out-template clusters. Differences are multistep minus single-step, except the final rule comparison. Success differences use percentage points and decision/token differences use counts per episode, with intervals for aggregate paired effects.}
  \label{tab:horizon}
  \widetabstyle
  \begin{tabular}{@{}lrrrr@{}}
    \toprule
    \thead{Family / metric} & \thead{Multi-step} & \thead{Single-step} & \thead{$\Delta$} & \thead{95\% CI} \\
    \midrule
    Destination & 100.0\% & 100.0\% & 0.0 & \textemdash{} \\
    Recipe & 100.0\% & 100.0\% & 0.0 & \textemdash{} \\
    Allocation & 92.5\% & 83.3\% & 9.2 & \textemdash{} \\
    Active branch & 100.0\% & 81.7\% & 18.3 & \textemdash{} \\
    \midrule
    Overall success & 98.1\% & 91.3\% & 6.9 & $[2.9,10.8]$ \\
    Decisions / episode & 1.98 & 2.91 & $-0.93$ & $[-0.97,-0.89]$ \\
    Input tokens / episode & 1,146.5 & 1,689.1 & $-542.6$ & $[-567.8,-517.6]$ \\
    \midrule
    Multi-step versus rule & 98.1\% & 100.0\% & $-1.9$ & $[-4.2,-0.2]$ \\
    \bottomrule
  \end{tabular}
\end{table}

Multistep control uses 1,920 examples and 240 updates per seed, whereas single-step control receives 2,880 examples, 360 updates, and explicit supervision for an intermediate decision after acquisition. We compare those complete training and control protocols because the additional decision changes both the available supervision and online interaction. The experiment extends communicated multistep intent~\cite{aamas2025_p481} to executable obligations whose duration determines when another decision is needed.

The deterministic correct-assignment rule reaches the held-out task ceiling in Table~\ref{tab:horizon}, while Direct preserves the requested assignment in all 471 executable outputs. Its nine remaining trials stop before executable-plan formation, localizing the difference to producing the object rather than carrying out an admitted continuation.

\section{Prompted Communication and Revision}
\label{sec:supp-prompted}
This section supports main Sections 6.1--6.2 by comparing the two communication protocols with the same DeepSeek model and by evaluating revision after a partner supplies structured feedback. We describe the dialogue contract before reporting quality, model usage, and the independent timing sample, so the comparison can be interpreted as two ways to form the same kind of commitment.

\subsection{Messages, Examples, and Acceptance}
Both protocols use DeepSeek-V4.1-Flash for the two participants, the same task instances, and the same skill vocabulary, commitment checks, and world executor. Each prompt contains four worked examples, two for each role orientation, clearly separated from the current instance. The dialogue examples are the teacher conversations used to train the local controlled free-form system. All are initiated by the material supplier, and the current prompt identifies the agent's and partner's public roles. Short DSL likewise uses role-balanced examples with instructions appropriate to the agent's current role. The demonstrations and serialized task context enter the input-token counts alongside the interaction history. These prompt components are specific to each protocol, so input totals include both the examples that teach the protocol and the messages exchanged while solving the task.

The dialogue protocol allows natural-language discussion accompanied by structured control markers, with at most eight logical model calls. A \artifact{COMMIT} marker encloses the two agents and their proposed task lists, while a valid \artifact{ACCEPT\_COMMIT} accepts a referenced proposal. The partner can also close the commitment by submitting the matching valid object. The controller validates the structured object and acceptance relationship. Explanatory text can discuss conditions or alternatives without changing a valid control marker, whereas an explicit \artifact{<REJECT/>} rejects the proposal. Thus the two agents may explain their work in natural language while the markers identify precisely what they submit and accept.

For example, a supplier can state which branch the order requires, promise to obtain and deliver its material, and include the corresponding structured proposal in the same turn. The receiver then accepts the proposal or submits the matching object before the controller closes the commitment. The complete example dialogues in the supplementary data show both forms of acceptance. Short DSL instead expresses the sender's work and requested continuation directly as skill paths, which the path resolver checks under the protocol in Section~\ref{sec:supp-lifecycle}.

\subsection{Task and Formation Outcomes}
We evaluate 128 task clusters with two role assignments and three API repetitions per protocol, giving 768 assigned trials in each arm. Table~\ref{tab:prompted-quality-complete} distinguishes commitment formation from handoff, fulfillment, and success on the original task, while Table~\ref{tab:prompted-role} breaks these results down by role orientation.

\begin{table}[!htbp]
\centering\tabstyle
\caption{Task outcomes with the same prompted DeepSeek model. C counts formed commitments, H verified handoffs, F fulfilled commitments, and G successes on the original task. Counts use all assigned trials.}
\label{tab:prompted-quality-complete}
\begin{tabular}{@{}lrrrrr@{}}\toprule
Protocol & Assigned & C & H & F & G\\\midrule
Short DSL & 768 & 768 & 768 & 768 & 768\\
Dialogue + structured submission & 768 & 767 & 767 & 767 & 766\\\bottomrule
\end{tabular}
\end{table}

\begin{table}[!htbp]
\centering\tabstyle
\caption{Task outcomes by role orientation. Original and swapped roles each contain 384 assigned trials per protocol.}
\label{tab:prompted-role}
\begin{tabular}{@{}llrrrr@{}}\toprule
Protocol & Roles & Assigned & Formed & Fulfilled & Task success\\\midrule
Short DSL & Original & 384 & 384 & 384 & 384\\
Short DSL & Swapped & 384 & 384 & 384 & 384\\
Dialogue & Original & 384 & 384 & 384 & 383\\
Dialogue & Swapped & 384 & 383 & 383 & 383\\\bottomrule
\end{tabular}
\end{table}

The dialogue protocol has one unsuccessful trial in each role orientation. One dialogue trial stops before commitment formation, and the other fulfills a commitment that selects work inconsistent with the task. The first proposal in each of these trials assigns the roles correctly. The communication-cost comparison therefore concerns two protocols that both complete almost all tasks, with the stages in Table~\ref{tab:prompted-quality-complete} identifying the two dialogue outcomes.

\subsection{Peer Traffic, Model Usage, and Charges}
Table~\ref{tab:prompted-usage-complete} reports per-trial model usage and peer traffic in the quality pool. Peer traffic counts UTF-8 bytes actually delivered between the agents, including commitment and closing control messages. It is therefore distinct from API request bodies and from the model's output-token count. The 768 Short DSL trials use 86,592 peer bytes and 1,536 logical calls, whereas the dialogue trials use 701,770 bytes and 1,631 calls. Although both protocols average close to two calls, dialogue generates more output and transmits more bytes to the partner. Input-token counts also include the protocol-specific examples and context encoding described above.

\begin{table}[!htbp]
\centering\tabstyle
\caption{Mean communication and model costs per assigned quality trial. Dialogue token and charge means use 767 known usage entries; other means use 768 entries. Input includes worked examples, context encoding, and interaction history, and charges price both input and output under the stated tariff.}
\label{tab:prompted-usage-complete}
\begin{tabular}{@{}lrrr@{}}\toprule
Metric & Short DSL & Dialogue & Dialogue / DSL\\\midrule
Logical calls & 2.00 & 2.12 & 1.06\\
Peer traffic (B) & 112.75 & 913.76 & 8.10\\
Input tokens & 5,573 & 21,086 & 3.78\\
Output tokens & 100.4 & 276.3 & 2.75\\
Estimated charge (USD) & 0.000284 & 0.000735 & 2.59\\\bottomrule
\end{tabular}
\end{table}

API charges price the complete prompting configuration, including protocol-specific examples, context encoding, generated responses, and cache use. We estimate them from usage at fixed prices of USD 0.30 per million uncached input tokens, USD 0.006 per million cached input tokens, and USD 1.20 per million output tokens. The charge summaries include usage associated with unsuccessful trials as well as successful ones. One successful dialogue trial has no complete usage estimate, so the table reports means over available usage rather than assigning it zero cost. These estimates represent API charges under the stated tariff, distinct from a supplier invoice or the cost of local computation.

\subsection{Interleaved Timing}
The timing comparison uses 64 fresh trials per protocol, interleaving the protocols in a fixed alternating order with HTTP concurrency one and one world at a time. TTC begins with forming the proposal and ends when a commitment is available or the attempt stops, including API service and network time. All 128 timed trials form commitments, although a subsequent Short DSL handoff failure makes that arm's task-success count 63/64. We include all 64 commitment times in its summary because the handoff occurs after formation.

\begin{table}[!htbp]
\centering\tabstyle
\caption{Formation time and estimated charges in the fresh timing sample. Each protocol has 64 assigned and formed trials. Times are seconds, and charges are USD per trial. Median averages the central observations; p95 uses the lower empirical 95th-percentile order statistic.}
\label{tab:prompted-timing-complete}
\begin{tabular}{@{}lrrrrrr@{}}\toprule
Protocol & Task success & TTC mean & Median & p95 & Calls & API charge\\\midrule
Short DSL & 63/64 & 4.29 & 4.38 & 5.87 & 2.00 & 0.000257\\
Dialogue & 64/64 & 6.17 & 5.52 & 10.03 & 2.11 & 0.000599\\\bottomrule
\end{tabular}
\end{table}

As Table~\ref{tab:prompted-timing-complete} shows, Short DSL reduces mean TTC by 30.4\% and median TTC by 20.8\%. The timing sample has complete usage estimates for both protocols, with its own prompt and response lengths, so its charge means are reported independently of the larger quality pool. These measurements quantify the delay to forming the commitment and distinguish it from the physical work that follows.

\subsection{Revising from Structured Feedback}
We evaluate feedback when the requester has a private goal and its peer has a private processing mode. In the controlled condition, we supply an initial proposal and either the correct response or a legal alternative. In the full interaction, the requester generates the initial proposal and then revises it after feedback. The response policy and commitment checks are the same in both cases.

The four worked feedback examples balance role orientations, with the peer-receives-finished mode paired with the original roles and the peer-processes-material mode paired with swapped roles. We cross 40 independent mode pairs with two modes, two roles, and three API repetitions, yielding 480 trials per condition. Table~\ref{tab:prompted-feedback-complete} shows how formation and physical outcomes relate, while Table~\ref{tab:prompted-feedback-role} gives the role breakdown. Correct-feedback revision succeeds in 87.3\% of trials and the full interaction succeeds in 86.3\%, demonstrating that the model can revise the actor and handoff from structured feedback.

\begin{table}[!htbp]
\centering\tabstyle
\caption{Requester revision with the prompted DeepSeek model. C, H, and F measure formation, handoff, and fulfillment; G measures success under the unchanged original task.}
\label{tab:prompted-feedback-complete}
\begin{tabular}{@{}lrrrrr@{}}\toprule
Condition & Assigned & C & H & F & G\\\midrule
Correct response, initial proposal supplied & 480 & 419 & 419 & 419 & 419\\
Alternative response, initial proposal supplied & 480 & 403 & 403 & 403 & 0\\
Initial proposal and revision generated & 480 & 414 & 414 & 414 & 414\\\bottomrule
\end{tabular}
\end{table}

\begin{table}[!htbp]
\centering\tabstyle
\caption{Feedback outcomes by role orientation. Each entry uses 240 assigned trials. The paired effect compares correct with alternative feedback over 40 mode pairs.}
\label{tab:prompted-feedback-role}
\begin{tabular}{@{}lrrrr@{}}\toprule
 & \multicolumn{2}{c}{Original roles} & \multicolumn{2}{c}{Swapped roles}\\
\cmidrule(lr){2-3}\cmidrule(lr){4-5}
Condition & Formed & Task success & Formed & Task success\\\midrule
Correct response & 208 & 208 & 211 & 211\\
Alternative response & 208 & 0 & 195 & 0\\
Full interaction & 202 & 202 & 212 & 212\\\bottomrule
\end{tabular}
\end{table}

The correct-minus-alternative task-success effect is 87.3 percentage points, with a 95\% paired bootstrap interval of $[84.4,90.0]$ from 10,000 resamples of the 40 mode pairs. Alternative feedback produces 403 admitted and fulfilled commitments but no success on the original task, because these commitments follow the substituted division of work. The condition therefore checks how feedback changes the actor and transfer while keeping the task's requirements fixed.

\section{Fulfillment under Private Task Structure}
\label{sec:supp-private-structure}

This section supports main Section 6.3 by testing whether a receiver can realize transmitted work while the sender keeps the instance graph private. We report complete execution outcomes alongside message and computation costs for paths, graph fragments, and goals.

\subsection{Cases and Execution Objects}
We use deterministic projection to supply a path, graph slice, or goal with constraints, while giving the receiver shared primitives and local grounding. PATH supplies the ordered continuation and SLICE supplies its dependency graph, whereas GOAL searches for work satisfying the transmitted target. Because the original graph stays with the sender, these tests isolate the information the receiving interface needs to act.

We specify 114 parameter settings within 24 template--motif families and cross two versions with two roles, yielding 456 base executions. These are specified cases within the family structure rather than independent task families. We evaluate sender generation on 96 source settings and 384 version--role combinations with the three unchanged Direct checkpoints, complementing the receiver test with a measurement of object formation.

Table~\ref{tab:ps-physical-supp} reports admission, transfer, fulfillment, and original-task success for each transmitted representation and replacement condition. The separate columns show whether an executable alternative was carried out and whether it matched the original assignment.

\begin{table}[!htbp]
\centering\tabstyle
\caption{Physical outcomes when the sender keeps its task graph private. Compatible replacements can fulfill their delivered commitment while leaving the unchanged original task unsatisfied.}
\label{tab:ps-physical-supp}
\begin{tabular}{@{}lrrrrr@{}}
\toprule
Condition & Assigned & Executed & $H$ & $F$ & $G$\\
\midrule
Base PATH & 456 & 456 & 456 & 456 & 456\\
GOAL with a different legal path & 100 & 100 & 100 & 100 & 100\\
Compatible alternative request & 16 & 16 & 16 & 16 & 0\\
Incompatible alternative request & 16 & 0 & -- & -- & --\\
\bottomrule
\end{tabular}
\end{table}

The 588 assignments include 572 physically fulfilled objects and 16 requests rejected before execution. We select the 100 GOAL continuations because their legal paths differ from the supplied PATH, demonstrating fulfillment through another solution to the declared objective. The original task stays unchanged in both replacement conditions, so the compatible alternatives can be fulfilled even when task success is zero.

\subsection{Cold, Warm, and Changed-Content Accounting}
We count a session's application request and acknowledgment, including the method, roles, sequence, and either the task content or a cache reference. A first delivery installs the content, a cache hit reuses it, and changed content is transmitted again. Primitive semantics are shared beforehand, while the ledger measures the task communication required within that session.

Table~\ref{tab:ps-full-ledger} gives request, acknowledgment, and total transmitted bytes for each cache condition, with stored receiver content shown in its own column. We weight the 114 settings equally after pairing versions and roles, and provide family-balanced summaries in the accompanying data.

\begin{table}[!htbp]
\centering\tabstyle
\caption{Transmitted bytes for paths, slices, and goals under first, cached, and changed delivery. Requests include content or a cache reference, every acknowledgment is 135 B, and receiver cache is stored content rather than transmission. With caching off, Any denotes identical costs across the three delivery conditions.}
\label{tab:ps-full-ledger}
\begin{tabular}{@{}lllrrrr@{}}
\toprule
Cache & Delivery & Message & Request (B) & ACK (B) & Total (B) & Cache (B)\\\midrule
Off & Any & PATH & 605.63 & 135.00 & 740.63 & 0.00\\
Off & Any & SLICE & 670.30 & 135.00 & 805.30 & 0.00\\
Off & Any & GOAL & 526.89 & 135.00 & 661.89 & 0.00\\
\addlinespace
On & Cold & PATH & 605.63 & 135.00 & 740.63 & 348.88\\
On & Cold & SLICE & 670.30 & 135.00 & 805.30 & 401.21\\
On & Cold & GOAL & 526.89 & 135.00 & 661.89 & 273.31\\
\addlinespace
On & Warm & PATH & 201.00 & 135.00 & 336.00 & 348.88\\
On & Warm & SLICE & 202.00 & 135.00 & 337.00 & 401.21\\
On & Warm & GOAL & 212.00 & 135.00 & 347.00 & 273.31\\
\addlinespace
On & Changed & PATH & 605.63 & 135.00 & 740.63 & 697.75\\
On & Changed & SLICE & 670.30 & 135.00 & 805.30 & 802.42\\
On & Changed & GOAL & 526.89 & 135.00 & 661.89 & 546.61\\
\addlinespace
\bottomrule\end{tabular}\end{table}

\begin{table}[!htbp]
\centering\tabstyle
\caption{Delivery-processing CPU time in microseconds ($\mu$s), including serialization, decoding, realization, checks, and acknowledgment. Entries are medians of setting-level medians after pairing roles and versions, using five technical repeats and excluding model calls and world execution.}
\label{tab:ps-cpu-all}
\begin{tabular}{@{}llrrr@{}}
\toprule
Cache & Delivery & PATH & SLICE & GOAL\\\midrule
Off & Cold & 65 & 66 & 2010\\
Off & Warm & 60 & 61 & 2026\\
Off & Changed & 60 & 62 & 1983\\
On & Cold & 58 & 60 & 1984\\
On & Warm & 50 & 52 & 1981\\
On & Changed & 58 & 60 & 1989\\
\bottomrule\end{tabular}\end{table}

GOAL saves bytes on first delivery by leaving search to the receiver, whereas PATH and SLICE communicate the task structure that it must realize. Caching preserves this difference because a cached target still requires search, while a cached path already specifies the continuation. Tables~\ref{tab:ps-full-ledger} and~\ref{tab:ps-cpu-all} therefore describe both communication and computation within the same application-level exchange.

\subsection{Direct Sender Formation under Private Structure}

For the generative probe, we supply the selected route to the unchanged Direct model and ask it to produce the complete sender object. Across 96 source settings, two versions, two roles, and three checkpoints, all 1,152 outputs parse and 496 satisfy the complete-object criterion. Matched executable objects identify the corresponding physical receipts, while this probe measures generation under the new input contract.

Table~\ref{tab:ps-direct} breaks sender formation down by checkpoint, showing how the 496 complete objects are distributed across the three trained models.

\begin{table}[!htbp]
\centering\tabstyle
\caption{Complete sender objects generated by each unchanged Direct checkpoint under private task structure.}
\label{tab:ps-direct}
\begin{tabular}{@{}lrrr@{}}
\toprule
Checkpoint seed & Calls & Strictly formed & Formation (\%)\\
\midrule
Seed 1 & 384 & 223 & 58.1\\
Seed 2 & 384 & 145 & 37.8\\
Seed 3 & 384 & 128 & 33.3\\
All & 1,152 & 496 & 43.1\\
\bottomrule
\end{tabular}
\end{table}

The physical interface test establishes what a receiver can execute from a supplied object, whereas the Direct probe measures how reliably a particular sender produces that object. Reporting both makes the role of deterministic projection explicit.

\section{Formation and Receiver Implementations}
\label{sec:supp-implementations}

This section supports the implementation choices in main Section 6.3 by comparing the cost of expressing an assignment and realizing it at the receiver. We first hold the peer request fixed across formation backends, then hold sender outputs fixed across receiver implementations.

\subsection{Backend Formation and Composition}
\label{sec:supp-backend}

Table~\ref{tab:backend} compares formation with the same peer-facing object, using a timing sample independent of the surface comparison in Section~\ref{sec:supp-quality}. Table~\ref{tab:factor-minimality} examines how the composer combines fields on unseen assignment cross-products, distinguishing that construction from lookup of a complete known case.

\begin{table}[!htbp]
  \centering
  \caption{Direct and Factor formation at matched quality. Success uses 128 clusters, with two roles and three checkpoints, while generation cost and TTC use a separate 40-cluster timing sample.}
  \label{tab:backend}
  \tabstyle
  \begin{tabular}{@{}lrr@{}}
    \toprule
    \thead{Metric} & \thead{Factor + rule} & \thead{Direct DSL} \\
    \midrule
    Held-out success & 768/768 & 768/768 \\
    Input tokens & 980.2 & 948.9 \\
    Output tokens & 19.4 & 79.1 \\
    Protocol-specific peer payload & 76.6 B & 76.6 B \\
    TTC p50 & 0.82 s & 2.45 s \\
    TTC p95 & 1.00 s & 2.72 s \\
    \bottomrule
  \end{tabular}
\end{table}

Factor predicts the sender's assignment code before delivery and the receiver's accepted assignment code afterward, with the same input information available to Direct. A deterministic composer then expands the code into actors, tasks, objects, destinations, dependencies, and the peer-facing Short DSL request.

\begin{table}[!htbp]
  \centering
  \caption{Complete-case lookup versus field-wise composition on 60 unseen assignment cross-products.}
  \label{tab:factor-minimality}
  \tabstyle
  \begin{tabular}{@{}lr@{}}
    \toprule
    \thead{Method} & \thead{Task success} \\
    \midrule
    Complete-case table & 0/60 \\
    Explicit-factor composer & 60/60 \\
    Direct Short DSL & 60/60 \\
    \bottomrule
  \end{tabular}
\end{table}

Together, Tables~\ref{tab:backend} and~\ref{tab:factor-minimality} connect a shorter generation target to the complete object produced by the composer. Field-wise composition supports new combinations while preserving the executable request that the partner receives.

\subsection{Fixed-Sender Receiver Realization}
\label{sec:supp-receiver}

We hold 1,920 saved sender outputs fixed, with 960 from each backend, and compare learned reception, deterministic request realization, and shared-directory lookup. Each implementation uses the sender's selected assignment, giving 5,760 formation trials without substituting a correct reference choice. This isolates reception from sender accuracy, while Section~\ref{sec:supp-online-descriptions} evaluates the complete online chain.

\begin{table}[!htbp]
\centering\tabstyle
\caption{Learned, deterministic, and directory-based receivers applied to the same saved sender outputs.}
\label{tab:receiver-formation}
\begin{tabular}{@{}lrrr@{}}
\toprule
Sender source & Learned & \shortstack{Deterministic\\path} & \shortstack{Directory\\lookup}\\
\midrule
API (960) & 952 & 959 & 959\\
Local (960) & 958 & 958 & 958\\
All (1,920) & 1,910 & 1,917 & 1,917\\
\bottomrule
\end{tabular}
\end{table}

The seven additional commitments produced by deterministic reception all arise from API sender outputs, and each succeeds in a newly executed world. Directory lookup produces the same seven objects, while the remaining entries in Table~\ref{tab:receiver-formation} compare realization of the saved proposals.

The receiver-stage benchmark crosses 16 world blocks, 64 role instances, three checkpoints, and five technical repeats, producing 960 operations per implementation on one RTX 4090. Timing begins when the receiver processes the request and ends after realization checks, leaving sender generation, loading, and world execution outside the measured stage.

Table~\ref{tab:receiver-timing} reports each implementation's latency distribution, allowing its typical receiver cost to be compared with the paired mean differences below.

\begin{table}[!htbp]
\centering\tabstyle
\caption{Receiver processing latency on the common local sample. Each method has 960 operations over 16 world blocks, including all technical repetitions and the long learned-receiver observation.}
\label{tab:receiver-timing}
\begin{tabular}{@{}lrrrr@{}}
\toprule
Implementation & Median (ms) & Mean (ms) & p95 (ms) & Maximum (ms)\\
\midrule
Learned & 856.11 & 880.84 & 972.99 & 17482.75\\
Deterministic path & 0.63 & 0.63 & 0.65 & 1.21\\
Directory lookup & 0.36 & 0.36 & 0.38 & 0.73\\
\bottomrule\end{tabular}\end{table}

Deterministic realization saves a paired mean of 880 ms $[862,898]$ relative to learned reception, and directory lookup saves 880 ms $[862,899]$. We average within world block before 10,000 bootstrap resamples within template, so the intervals measure changes in receiver processing. The full online measurements in Section~\ref{sec:supp-online-descriptions} additionally account for sender formation and physical execution.

Table~\ref{tab:receiver-seeds} reports receiver-stage means for each local checkpoint, showing how the aggregate difference relates to the three trained implementations.

\begin{table}[!htbp]
\centering\tabstyle
\caption{Mean receiver processing time in milliseconds by local checkpoint, with 320 operations per method and checkpoint.}
\label{tab:receiver-seeds}
\begin{tabular}{@{}lrrr@{}}
\toprule
Seed & Learned & \shortstack{Deterministic\\path} & \shortstack{Directory\\lookup}\\\midrule
Seed 1 & 916.77 & 0.63 & 0.36\\
Seed 2 & 872.80 & 0.63 & 0.36\\
Seed 3 & 852.94 & 0.62 & 0.36\\
\bottomrule\end{tabular}\end{table}

\section{Naturalized Descriptions in Closed-Loop Execution}
\label{sec:supp-online-descriptions}

This section supports the receiver and language comparisons in main Section 6.3 by following natural descriptions through formation, transmission, and actual Minecraft execution. The stage breakdown and paired timing identify which components change when the same assignment interface uses another selector or receiver.

\subsection{Descriptions, Selectors, and Execution}

We use 32 original phrasings distributed over 16 template--assignment choices, with paraphrase and distractor conditions derived from the same sources. Each condition has 320 role instances, and the local generators use three checkpoints. We therefore evaluate 10,240 quality trials across the eight configurations in Table~\ref{tab:online-methods}, keeping the shared language sources explicit in the comparison.

The configurations compare Direct and Factor with learned or deterministic receivers, alongside TF--IDF retrieval, supervised lexical classification, embedding retrieval, and a privileged quality reference. The embedding model has general pretraining, whereas the classifier is fitted on the training descriptions, allowing the interface to support practically different approaches to selecting the same assignment.

\begin{table}[!htbp]
\centering\tabstyle
\caption{Assignment and receiver implementations for natural descriptions. Generation produces a path or code, while an embedding call encodes the query description. The privileged reference supplies the correct assignment for quality evaluation.}
\label{tab:online-methods}
\begin{tabularx}{\textwidth}{@{}L{3.5cm} Z L{3.0cm}@{}}
\toprule
Configuration & Sender formation & Receiver\\\midrule
Direct / learned & Path generation & Learned path generation\\
Direct / deterministic & Same Direct sender output & Deterministic interpretation\\
Factor / learned & Short-code prediction and composition & Learned code and composition\\
Factor / deterministic & Same Factor sender output & Deterministic interpretation\\
TF--IDF retrieval & Nearest TRAIN description, then composition & Deterministic interpretation\\
Supervised lexical classifier & TRAIN-fitted logistic regression, then composition & Deterministic interpretation\\
Embedding retrieval & Pretrained encoder retrieval, then composition & Deterministic interpretation\\
Privileged true-binding reference & True binding and composition & Deterministic interpretation\\
\bottomrule\end{tabularx}\end{table}

Embedding retrieval uses Qwen3-Embedding-0.6B with a fixed index of 64 training descriptions~\cite{qwen3embedding2025}, selecting within the public template. The supervised lexical model chooses its regularization on development data, and the local generators use their existing checkpoints with deterministic decoding. These choices preserve each implementation's trained behavior while changing the formation component used by the common receiver.

The receiver sees \texttt{PROJECTED\_HANDOFF\_CONTEXT}, which specifies the inventory and history expected after the proposed transfer. Once the full commitment has formed, the agents actually execute the handoff and read back the world state, so the online chain validates those projected preconditions physically.

We denote request matching by $M$, which checks the raw receiver continuation against the delivered request before admission. Physical fulfillment $F$ is evaluated later from transfer and completion, allowing a matched continuation to be distinguished from successful environmental execution. The stage columns in Table~\ref{tab:online-quality} preserve this distinction.

\subsection{From Formation to Physical Outcomes}

Figure~\ref{fig:online-descriptions}(a) decomposes all 10,240 trials into mutually exclusive outcomes: 508 formation failures, 12 handoff failures, five failures to complete the declared work, 2,181 fulfilled but inappropriate obligations, and 7,534 successful tasks. Thus 9,715 of the 9,732 formed commitments are fulfilled, with the remaining losses occurring during execution.

\begin{figure}[!htbp]
  \centering
  \includegraphics[width=\textwidth]{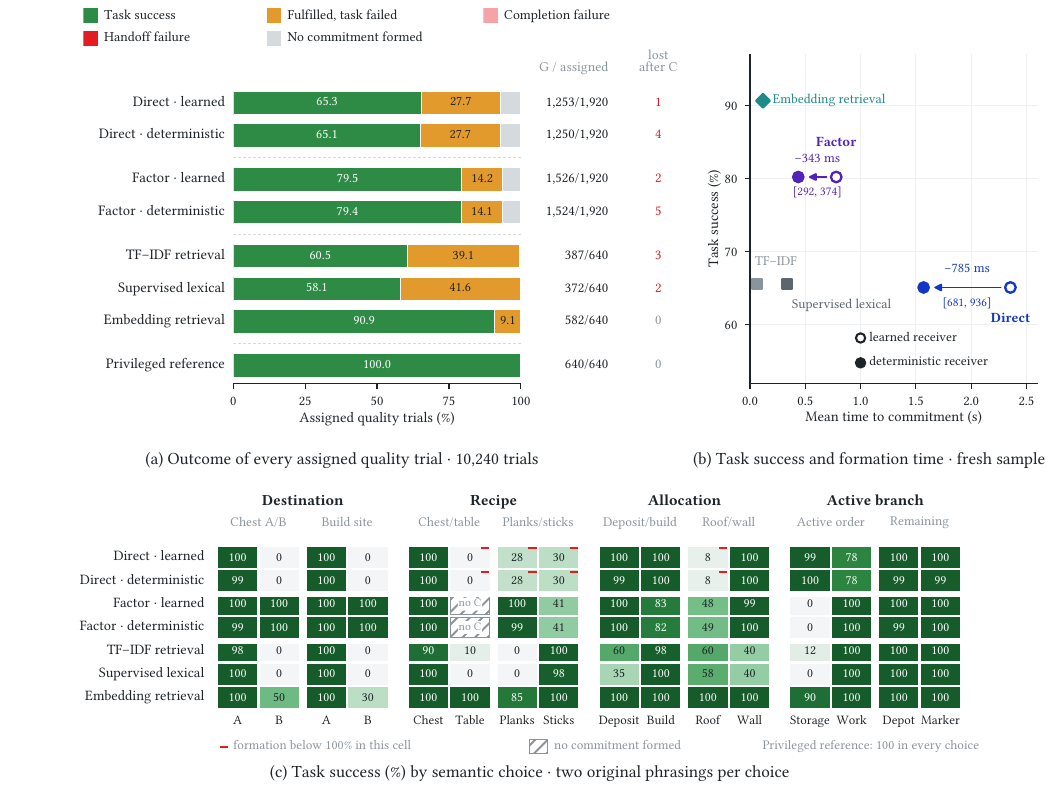}
  \Description{Three panels. Panel a shows one horizontal bar per configuration divided into task success, fulfilled obligations with failed tasks, completion failures, handoff failures, and unformed commitments; embedding retrieval reaches 90.9 percent task success, Factor about 79 percent, Direct about 65 percent, and the lexical selectors about 58 to 61 percent, with at most five trials lost after formation per configuration. Panel b plots task success against mean time to commitment on the fresh timing sample; arrows from learned to deterministic receivers shorten time to commitment by 785 milliseconds for Direct and 343 milliseconds for Factor at the same observed task success. Panel c is a heat map of task success for seven configurations across sixteen semantic choices, marking cells where no commitment formed.}
  \caption{From natural descriptions to world outcomes. Panel (a) shows mutually exclusive results over 10,240 assigned trials, normalizing each configuration by its own denominator and reporting task successes and losses after formation at right. Panel (b) compares success and mean TTC on the same fresh timing sample, with arrows and paired 95\% intervals for receiver replacement over 16 sources. The lexical mean includes the 17.3 s observation, while its median is 0.06 s. Panel (c) reports success for 16 choices, each derived from two phrasings: hatching marks no formed commitment and red ticks mark formation below 100\%. The privileged reference completes every choice and has no timing arm.}
  \label{fig:online-descriptions}
\end{figure}

\begin{table}[!htbp]
\centering\tabstyle
\caption{Physical outcomes under each description condition. All assigned trials contribute to C (formation), H (verified handoff), F (declared fulfillment), and G (original-task success), with task success also shown as a percentage.}
\label{tab:online-quality}
\begin{tabular}{@{}lrrrrrr@{}}
\toprule
Configuration & Assigned & $C$ & $H$ & $F$ & $G$ & $G$ (\%)\\\midrule
\multicolumn{7}{l}{\textit{Paraphrases}}\\
Direct / learned & 960 & 898 & 898 & 898 & 632 & 65.8\\
Direct / deterministic & 960 & 898 & 898 & 898 & 632 & 65.8\\
Factor / learned & 960 & 900 & 899 & 899 & 758 & 79.0\\
Factor / deterministic & 960 & 900 & 900 & 899 & 759 & 79.1\\
TF--IDF retrieval & 320 & 320 & 319 & 319 & 189 & 59.1\\
Supervised lexical classifier & 320 & 320 & 320 & 320 & 190 & 59.4\\
Embedding retrieval & 320 & 320 & 320 & 320 & 290 & 90.6\\
Privileged true-binding reference & 320 & 320 & 320 & 320 & 320 & 100.0\\
\addlinespace
\multicolumn{7}{l}{\textit{With distractors}}\\
Direct / learned & 960 & 888 & 887 & 887 & 621 & 64.7\\
Direct / deterministic & 960 & 888 & 886 & 884 & 618 & 64.4\\
Factor / learned & 960 & 900 & 900 & 899 & 768 & 80.0\\
Factor / deterministic & 960 & 900 & 897 & 896 & 765 & 79.7\\
TF--IDF retrieval & 320 & 320 & 318 & 318 & 198 & 61.9\\
Supervised lexical classifier & 320 & 320 & 318 & 318 & 182 & 56.9\\
Embedding retrieval & 320 & 320 & 320 & 320 & 292 & 91.3\\
Privileged true-binding reference & 320 & 320 & 320 & 320 & 320 & 100.0\\
\addlinespace
\bottomrule\end{tabular}\end{table}

Among the 2,198 task failures after formation, 2,181 follow successful fulfillment of an inappropriate obligation, or 99.2\%. Embedding retrieval provides a clear instance: all 640 commitments are fulfilled, while the 58 unsuccessful tasks select the wrong assignment. The pattern localizes the main error on these descriptions to selection, while the stage breakdown also exposes the 17 execution losses.

\subsection{Paired Receiver and Selector Comparisons}

We pair 3,840 sender outputs across the two receiver implementations, yielding 7,680 quality trials, and add 384 pairs in the fresh timing sample. Because each pair uses the same sender output, both receivers form the same number of commitments and their small task-success differences arise after formation. Table~\ref{tab:online-g-effects} gives the paired outcomes by language condition and selector.

For quality inference, we first average roles, worlds, and checkpoints within each original phrasing, then resample the two phrasings within each semantic choice. An unstratified sensitivity instead resamples all 32 phrasings, allowing the choice mixture to vary. Both use 10,000 paired percentile resamples, and the broader unstratified intervals include zero for Factor versus learned Direct and Factor versus embedding retrieval. Table~\ref{tab:online-cf-effects} reports the corresponding changes in formation and fulfillment.

\begin{table}[!htbp]
\centering\tabstyle
\caption{Paired differences in task success, in percentage points, for each selector and receiver configuration. Stratified and unstratified intervals use the same source pairing.}
\label{tab:online-g-effects}
\begin{tabular}{@{}lrrr@{}}
\toprule
Contrast & Difference & Stratified 95\% CI & Unstratified 95\% CI\\\midrule
\multicolumn{4}{l}{\textit{Paraphrases}}\\
Factor: deterministic $-$ learned & $+0.1$ & $[0.0,0.2]$ & $[0.0,0.3]$\\
Direct: deterministic $-$ learned & $0.0$ & $[0.0,0.0]$ & $[0.0,0.0]$\\
Factor deterministic $-$ Direct learned & $+13.2$ & $[10.4,16.3]$ & $[-3.0,29.7]$\\
Factor deterministic $-$ TF--IDF & $+20.0$ & $[12.0,28.0]$ & $[2.5,38.1]$\\
Factor deterministic $-$ supervised lexical & $+19.7$ & $[12.0,27.4]$ & $[2.2,38.0]$\\
Factor deterministic $-$ embedding & $-11.6$ & $[-15.7,-7.4]$ & $[-28.9,5.8]$\\
\addlinespace
\multicolumn{4}{l}{\textit{With distractors}}\\
Factor: deterministic $-$ learned & $-0.3$ & $[-0.6,0.0]$ & $[-0.7,0.1]$\\
Direct: deterministic $-$ learned & $-0.3$ & $[-0.6,0.0]$ & $[-0.7,0.1]$\\
Factor deterministic $-$ Direct learned & $+15.0$ & $[12.6,17.4]$ & $[-1.1,31.0]$\\
Factor deterministic $-$ TF--IDF & $+17.8$ & $[11.4,24.3]$ & $[0.7,35.4]$\\
Factor deterministic $-$ supervised lexical & $+22.8$ & $[17.3,28.2]$ & $[5.8,40.2]$\\
Factor deterministic $-$ embedding & $-11.6$ & $[-15.1,-8.0]$ & $[-26.2,2.3]$\\
\addlinespace
\bottomrule\end{tabular}\end{table}

\begin{table}[!htbp]
\centering\tabstyle
\caption{Paired changes in commitment formation and declared fulfillment, in percentage points. Both use the same source resampling as the task-success comparison.}
\label{tab:online-cf-effects}
\begin{tabular}{@{}llrrr@{}}
\toprule
Contrast & Metric & Difference & Stratified 95\% CI & Unstratified 95\% CI\\\midrule
\multicolumn{5}{l}{\textit{Paraphrases}}\\
Factor: deterministic $-$ learned & $C$ & $0.0$ & $[0.0,0.0]$ & $[0.0,0.0]$\\
Factor: deterministic $-$ learned & $F$ & $0.0$ & $[-0.2,0.2]$ & $[-0.3,0.3]$\\
Direct: deterministic $-$ learned & $C$ & $0.0$ & $[0.0,0.0]$ & $[0.0,0.0]$\\
Direct: deterministic $-$ learned & $F$ & $0.0$ & $[0.0,0.0]$ & $[0.0,0.0]$\\
Factor deterministic $-$ Direct learned & $C$ & $+0.2$ & $[-1.9,2.3]$ & $[-8.4,6.9]$\\
Factor deterministic $-$ Direct learned & $F$ & $+0.1$ & $[-2.0,2.2]$ & $[-8.5,6.9]$\\
Factor deterministic $-$ TF--IDF & $C$ & $-6.2$ & $[-6.2,-6.2]$ & $[-15.6,0.0]$\\
Factor deterministic $-$ TF--IDF & $F$ & $-6.0$ & $[-6.2,-5.8]$ & $[-15.6,0.4]$\\
Factor deterministic $-$ supervised lexical & $C$ & $-6.2$ & $[-6.2,-6.2]$ & $[-15.6,0.0]$\\
Factor deterministic $-$ supervised lexical & $F$ & $-6.4$ & $[-6.5,-6.2]$ & $[-15.7,0.0]$\\
Factor deterministic $-$ embedding & $C$ & $-6.2$ & $[-6.2,-6.2]$ & $[-15.6,0.0]$\\
Factor deterministic $-$ embedding & $F$ & $-6.4$ & $[-6.5,-6.2]$ & $[-15.7,0.0]$\\
\addlinespace
\multicolumn{5}{l}{\textit{With distractors}}\\
Factor: deterministic $-$ learned & $C$ & $0.0$ & $[0.0,0.0]$ & $[0.0,0.0]$\\
Factor: deterministic $-$ learned & $F$ & $-0.3$ & $[-0.6,0.0]$ & $[-0.7,0.1]$\\
Direct: deterministic $-$ learned & $C$ & $0.0$ & $[0.0,0.0]$ & $[0.0,0.0]$\\
Direct: deterministic $-$ learned & $F$ & $-0.3$ & $[-0.6,0.0]$ & $[-0.7,0.1]$\\
Factor deterministic $-$ Direct learned & $C$ & $+1.3$ & $[-1.0,3.5]$ & $[-7.9,8.3]$\\
Factor deterministic $-$ Direct learned & $F$ & $+0.9$ & $[-1.4,3.2]$ & $[-8.1,8.1]$\\
Factor deterministic $-$ TF--IDF & $C$ & $-6.2$ & $[-6.2,-6.2]$ & $[-15.6,0.0]$\\
Factor deterministic $-$ TF--IDF & $F$ & $-6.0$ & $[-6.8,-5.3]$ & $[-15.7,0.6]$\\
Factor deterministic $-$ supervised lexical & $C$ & $-6.2$ & $[-6.2,-6.2]$ & $[-15.6,0.0]$\\
Factor deterministic $-$ supervised lexical & $F$ & $-6.0$ & $[-6.7,-5.4]$ & $[-15.7,0.6]$\\
Factor deterministic $-$ embedding & $C$ & $-6.2$ & $[-6.2,-6.2]$ & $[-15.6,0.0]$\\
Factor deterministic $-$ embedding & $F$ & $-6.7$ & $[-7.0,-6.4]$ & $[-16.0,-0.3]$\\
\addlinespace
\bottomrule\end{tabular}\end{table}

\subsection{Fresh Timing and Online Work}

We time one phrasing per semantic choice, giving 16 sources and 64 role instances across the two language conditions. Direct and Factor contribute 192 trials per receiver and each remaining configuration contributes 64, for 960 sequential executions on an exclusively used RTX 4090D. Models are resident and batch size is one. We start timing after world readiness and include formation, actions, and readback in episode time, whereas TTC ends at commitment formation or rejection. Table~\ref{tab:online-latency} and Figure~\ref{fig:online-descriptions}(b) use outcomes from this same timed sample.

\begin{table}[!htbp]
\centering\tabstyle
\caption{Latency and verified deadline success on the fresh sample. Times are seconds over all assigned trials, including rejection time. TTC ends at commitment, episode time includes world execution, and $V(30)$/$V(60)$ give percentages completing the task by 30/60 s.}
\label{tab:online-latency}
\begin{tabular}{@{}lrrrrrrrrr@{}}
\toprule
Configuration & $n$ & $G$ & \shortstack{TTC\\mean} & \shortstack{TTC\\median} & \shortstack{TTC\\p95} & \shortstack{Episode\\mean} & \shortstack{Episode\\median} & $V(30)$ & $V(60)$\\\midrule
Direct / learned & 192 & 125 & 2.35 & 2.29 & 2.56 & 24.09 & 24.98 & 64.6 & 65.1\\
Direct / deterministic & 192 & 125 & 1.57 & 1.55 & 1.72 & 23.29 & 24.18 & 65.1 & 65.1\\
Factor / learned & 192 & 154 & 0.78 & 0.78 & 0.90 & 22.31 & 23.55 & 80.2 & 80.2\\
Factor / deterministic & 192 & 154 & 0.44 & 0.43 & 0.51 & 21.97 & 23.14 & 80.2 & 80.2\\
TF--IDF retrieval & 64 & 42 & 0.07 & 0.06 & 0.11 & 22.94 & 22.66 & 65.6 & 65.6\\
Supervised lexical & 64 & 42 & 0.34 & 0.06 & 0.13 & 23.21 & 22.68 & 64.1 & 65.6\\
Embedding retrieval & 64 & 58 & 0.12 & 0.10 & 0.18 & 22.99 & 22.69 & 90.6 & 90.6\\
\bottomrule\end{tabular}\end{table}

Table~\ref{tab:online-timing-effects} reports the change caused by replacing learned reception with deterministic realization, which removes one generation step. We average roles and checkpoints within source before 10,000 paired bootstrap resamples, reporting TTC and complete episode time as distinct outcomes.

The lexical classifier's latency distribution includes a 17.33 s TTC observation despite a 5.35 ms selection service measurement, and that value remains in all summaries. Factor with deterministic reception has a TTC advantage over learned Direct, while their episode-time interval spans zero because episode duration also includes the environmental work.

\begin{table}[!htbp]
\centering\tabstyle
\caption{Paired changes in TTC and episode time, in milliseconds. A negative value denotes shorter time for the first configuration, with every assigned repeat included.}
\label{tab:online-timing-effects}
\begin{tabular}{@{}lrrrr@{}}
\toprule
Contrast & $\Delta$ TTC (ms) & 95\% CI & $\Delta$ episode (ms) & 95\% CI\\\midrule
Factor: deterministic $-$ learned & $-343$ & $[-374,-292]$ & $-335$ & $[-378,-279]$\\
Direct: deterministic $-$ learned & $-785$ & $[-936,-681]$ & $-800$ & $[-956,-680]$\\
Factor deterministic $-$ Direct learned & $-1,918$ & $[-2070,-1812]$ & $-2,114$ & $[-5679,540]$\\
Factor deterministic $-$ TF--IDF & $+371$ & $[359,383]$ & $-963$ & $[-3891,634]$\\
Factor deterministic $-$ supervised lexical & $+99$ & $[-448,379]$ & $-1,239$ & $[-4291,543]$\\
Factor deterministic $-$ embedding & $+319$ & $[301,335]$ & $-1,016$ & $[-3926,529]$\\
\bottomrule\end{tabular}\end{table}

Table~\ref{tab:online-work} counts work per timed trial, distinguishing model generation from query embedding. The retrieval index encodes 64 documents during initialization, while each online request performs one query encoding. The reported peer traffic comprises the Short DSL request and its acceptance acknowledgment, making it directly interpretable as communication between agents.

\begin{table}[!htbp]
\centering\tabstyle
\caption{Mean work per assigned timed trial. Generation tokens and embedding-query tokens are distinguished, and peer traffic comprises the transmitted request and acknowledgment.}
\label{tab:online-work}
\begin{tabular}{@{}lrrrrrrr@{}}
\toprule
Configuration & \shortstack{Gen.\\calls} & \shortstack{Encoder\\calls} & \shortstack{Input\\tokens} & \shortstack{Output\\tokens} & \shortstack{Query\\tokens} & \shortstack{REQ\\bytes} & \shortstack{ACK\\bytes}\\\midrule
Direct / learned & 1.95 & 0.00 & 943.1 & 76.9 & 0.0 & 71.5 & 64.5\\
Direct / deterministic & 1.00 & 0.00 & 455.0 & 53.3 & 0.0 & 71.5 & 64.5\\
Factor / learned & 1.94 & 0.00 & 968.1 & 18.5 & 0.0 & 72.4 & 63.8\\
Factor / deterministic & 1.00 & 0.00 & 457.0 & 9.6 & 0.0 & 72.4 & 63.8\\
TF--IDF retrieval & 0.00 & 0.00 & 0.0 & 0.0 & 0.0 & 75.2 & 68.0\\
Supervised lexical classifier & 0.00 & 0.00 & 0.0 & 0.0 & 0.0 & 75.2 & 68.0\\
Embedding retrieval & 0.00 & 1.00 & 0.0 & 0.0 & 37.4 & 76.4 & 68.0\\
\bottomrule\end{tabular}\end{table}

\subsection{Verified Success by Deadline and Semantic Choice}

We define deadline success as $V(t)=N^{-1}\sum_i\mathbf 1\{G_i=1,T_{\mathrm{episode},i}\leq t\}$, the fraction of assigned trials that complete the original task by time $t$. Unsubmitted or unsuccessful trials contribute zero, so Table~\ref{tab:online-latency} measures completion by a deadline rather than latency conditional on success.

Figure~\ref{fig:online-descriptions}(c) shows that errors concentrate differently across semantic choices. Factor fulfills the wrong storage-branch assignment in all 120 trials for that choice, whereas the crafting-table choice fails to form a complete sender object in all 120. Embedding retrieval instead reaches 30.0\% success for Site B and 50.0\% for Chest B. These distributions distinguish selection errors from incomplete generation and suggest that selector behavior depends on how a task choice is expressed. Each choice has two source phrasings, with complete source and checkpoint outcomes in the accompanying data.

We then use the new descriptions in Section~\ref{sec:supp-prospective-selection} to test the same checkpoints and retriever under another language distribution, invoking additional generation according to a fixed score threshold.

\section{Descriptions at the Planning Layer}
\label{sec:supp-planning}

This section supports assignment from language in main Section 6.3 by evaluating the prediction step before receiver realization. We provide complete method results and source-level uncertainty, then connect those predictions to the online experiment in Section~\ref{sec:supp-online-descriptions}.

\subsection{Binding from Descriptions}\label{sec:supp-planning-descriptions}

We replace the explicit assignment field with a description of the private fact, keeping two possible assignments per template and all 16 semantic choices represented in training. Syntactic recombination evaluates new expressions through Minecraft execution, while paraphrases change lexical wording and distractors add irrelevant material to the same source phrasings. Table~\ref{tab:naturalized-private} distinguishes those outcomes, and Section~\ref{sec:supp-online-descriptions} follows the paraphrase conditions through the complete online protocol.

\begin{table}[!htbp]
  \centering
  \caption{Assignment from descriptions under three language conditions. Syntactic recombination reports Minecraft task success, while paraphrase and distractor columns report planning accuracy. Each uses 320 role instances and three checkpoints for learned methods, with the latter two conditions derived from 32 phrasings across 16 choices.}
  \label{tab:naturalized-private}
  \tabstyle
  \begin{tabular}{@{}l R{2.6cm} R{2.6cm} R{2.6cm}@{}}
    \toprule
    \thead{Formation method} & \shortstack[r]{\thead{Recombination:}\\\thead{task success}} & \shortstack[r]{\thead{Paraphrases:}\\\thead{binding accuracy}} & \shortstack[r]{\thead{With distractors:}\\\thead{binding accuracy}} \\
    \midrule
    Direct Short DSL (0.8B) & 100.0\% & 65.8\% & 64.8\% \\
    Learned factor + rule (0.8B) & 100.0\% & \textbf{79.1\%} & \textbf{80.1\%} \\
    Exact-code parser rule & 100.0\% & 0.0\% & 0.0\% \\
    TF--IDF nearest-neighbour rule & 100.0\% & 59.4\% & 62.5\% \\
    Supervised TF--IDF + logistic regression & --- & 59.4\% & 57.5\% \\
    Complete surface table & 0.0\% & 0.0\% & 0.0\% \\
    Oracle factors + rule & 100.0\% & 100.0\% & 100.0\% \\
    \bottomrule
  \end{tabular}
\end{table}

The exact-code parser reads an explicit code, whereas the surface table looks up a complete known string. Table~\ref{tab:naturalized-private} compares these mechanisms with learned and retrieved assignment selection, with paired source uncertainty in Section~\ref{sec:supp-language-source}. This separation identifies what changes when the input communicates the fact through wording rather than a supplied code.

\subsection{Source Inference and Supervised Classification}
\label{sec:supp-language-source}

\subsubsection{Estimand and Source Hierarchy}

Each language condition contains 160 world clusters and 320 role instances derived from 32 original phrasings, with two phrasings for each of 16 template--assignment choices. The distractor condition modifies the same sources. We average worlds, roles, and checkpoints within phrasing before computing paired differences, so repeated executions do not become additional language sources.

Our primary estimand gives each semantic choice equal weight, averaging its two source differences before combining choices. For paired difference $\bar d_{cs}$, we therefore use
\[
\widehat\Delta=\frac{1}{16}\sum_{c=1}^{16}\frac{1}{2}\sum_{s=1}^{2}\bar d_{cs}.
\]
We obtain primary intervals by resampling phrasings within choice, while unstratified intervals resample all 32 phrasings and allow the semantic mixture to change. Table~\ref{tab:source-ci-new} reports both schemes using 10,000 percentile resamples conditional on the evaluated checkpoints, because the schemes represent different sources of uncertainty.

\begin{table}[!htbp]
\centering\tabstyle
\caption{Paired planning-accuracy differences in percentage points. Source-stratified intervals hold the mixture of 16 choices fixed, while unstratified intervals allow that mixture to vary. Retrieval denotes TF--IDF nearest-neighbor lookup.}
\label{tab:source-ci-new}
\begin{tabular}{@{}lrrr@{}}
\toprule
Contrast & Difference & Stratified 95\% CI & Unstratified 95\% CI\\\midrule
Paraphrases: Factor $-$ retrieval & 19.7 & $[12.0,27.3]$ & $[2.3,37.8]$\\
With distractors: Factor $-$ retrieval & 17.6 & $[11.0,24.0]$ & $[0.4,35.2]$\\
Paraphrases: Direct $-$ retrieval & 6.5 & $[-0.4,13.3]$ & $[-7.5,21.8]$\\
With distractors: Direct $-$ retrieval & 2.3 & $[-3.5,8.1]$ & $[-12.1,17.0]$\\
\bottomrule\end{tabular}\end{table}

A source-level standard-error sensitivity gives intervals of $[0.8,38.6]$ points for Factor minus retrieval on paraphrases and $[-1.2,36.4]$ with distractors, using approximately 2.04 standard errors. The distractor interval and Direct's stratified intervals include zero, illustrating the uncertainty that remains with two source phrasings per semantic choice.

\subsubsection{Supervised Lexical Baseline}

We build the supervised lexical classifier from word 1--2-grams and character-within-word 3--5-grams using TF--IDF and template-conditioned L2 logistic regression. We fit the vocabulary and inverse-document frequencies on 64 training descriptions and select regularization from $\{0.1,1,10\}$ on 16 development descriptions. Development accuracy selects 0.1, with the smaller value used to resolve ties.

The classifier correctly predicts 190/320 role instances on paraphrases and 184/320 with distractors, as shown in Table~\ref{tab:naturalized-private}. The first condition contains 32 surface strings, of which 19 are correct, and the second contains 64 strings derived from the same 32 sources. The online experiment applies the same fitted classifier before physical execution.

\section{Selection before Commitment}
\label{sec:supp-prospective-selection}

This section supports the assignment choices in main Section 6.3 by testing fixed selectors on new descriptions and measuring selective use of additional generation. We report the complete prospective comparison with Factor, followed by assignment accuracy and decision costs when DeepSeek supplies the hosted alternative.

\subsection{Descriptions and Prospective Design}
We generate descriptions with a language model from task specifications and screen them for a unique intended assignment. The corpus contains 928 descriptions: 32 for execution checks, 256 for calibration, and 640 for testing. Each of 16 choices has 16 calibration descriptions and 40 test descriptions, with the split fixed before scoring.

We keep the same three Factor checkpoints, Qwen3-Embedding-0.6B retriever, and 64-description training index, while expressing the tasks through newly written descriptions. For example, a description may require a crafting table while explicitly mentioning that a chest would be unsuitable. This changes the wording of a familiar choice without retraining either selector.

We calibrate on the difference between retrieval scores for the two legal assignments, selecting the 52nd smallest margin among 256 calibration descriptions. The resulting threshold, approximately $0.00942109$, uses strict inequality, so the tied threshold value stays with retrieval. It routes 97/640 test descriptions to additional generation, or 15.2\%, with the threshold fixed before their scores are evaluated.

\subsection{Strategies and Physical Outcomes}
Retrieval selects an assignment and composes its request, whereas Factor predicts the assignment code before the same deterministic composition. Selective upgrading first runs retrieval and invokes Factor when the score margin falls below the threshold. Factor receives the original local description, and its output is checked before transmission. The implementation uses the retrieval choice if that output is inadmissible, although every Factor output in this test passes the check.

The random and maximum-similarity controls upgrade the same number of descriptions within each public template, allowing us to compare how each ranking spends a fixed quota. All roles and checkpoints share the description's routing decision. The abstention strategy instead submits above the same margin threshold, with unsubmitted trials counted as unsuccessful assignments.

Table~\ref{tab:selection-quality-full} reports all six strategies across 17,920 assigned trials, including two roles per description and three checkpoints for strategies that use Factor. The deterministic receiver checks the transmitted request using projected handoff context, after which Minecraft supplies actual transfer and terminal readback. All 17,726 formed commitments are physically fulfilled, while 2,103 select work inconsistent with the original task and the remaining 194 assignments abstain before formation.
\begin{table}[!htbp]
\centering\tabstyle
\caption{Prospective task outcomes for all six strategies. C, H, and F coincide within each row, while $F-G$ counts fulfilled commitments that select the wrong work. Factor use is the percentage of assigned trials that invoke its generator.}
\label{tab:selection-quality-full}
\begin{tabular}{@{}lrrrrrr@{}}
\toprule
\textbf{Strategy} & \textbf{Assigned} & $C=H=F$ & $G$ & $G$ (\%) & $F-G$ & \textbf{Factor} (\%)\\
\midrule
Embedding retrieval & 1,280 & 1,280 & 1,052 & 82.2 & 228 & 0.0\\
Factor & 3,840 & 3,840 & 3,684 & 95.9 & 156 & 100.0\\
Margin-guided upgrade & 3,840 & 3,840 & 3,395 & 88.4 & 445 & 15.2\\
Random upgrade & 3,840 & 3,840 & 3,268 & 85.1 & 572 & 15.2\\
Score-guided upgrade & 3,840 & 3,840 & 3,280 & 85.4 & 560 & 15.2\\
Margin-based abstention & 1,280 & 1,086 & 944 & 73.8 & 142 & 0.0\\
\bottomrule
\end{tabular}
\end{table}

\subsection{Ranking and Paired Effects}
We average roles and checkpoints within each test description, then resample the 640 descriptions within their 16 semantic choices for 10,000 paired draws. The two primary effects use 97.5\% intervals, while secondary contrasts use 95\%. Unstratified resampling provides a sensitivity analysis over the same independent description unit.

We measure ranking by pairing a correctly and an incorrectly selected description from the same public template, giving tied confidence scores half weight. The resulting AUROC is 0.771 for the margin and 0.595 for maximum similarity. Pooling across templates instead gives 0.757 and 0.570, with the within-template comparison used for primary inference.

Table~\ref{tab:selection-effects-full} compares selective upgrading with retrieval, both matched-quota controls, and full Factor. Margin ranking improves how the quota is spent, although full Factor provides the highest task success. The paired quality gap and the decision-time differences locate selective upgrading between retrieval and full Factor, improving retrieval while using fewer generation calls.
\begin{table}[!htbp]
\centering\tabstyle
\caption{Paired prospective effects over 640 descriptions. Confidence levels are shown per row, and both interval schemes average repeated roles and checkpoints within description.}
\label{tab:selection-effects-full}
\begin{tabular}{@{}llrrrr@{}}
\toprule
\textbf{Contrast} & \textbf{Metric} & \textbf{Effect} & \textbf{Level} (\%) & \textbf{Stratified CI} & \textbf{Unstratified CI}\\
\midrule
Margin $-$ maximum similarity & AUROC & +0.176 & 97.5 & $[0.109,0.245]$ & $[0.107,0.247]$\\
Upgrade $-$ retrieval & $G$ (pp) & +6.2 & 97.5 & $[4.4,8.2]$ & $[4.2,8.4]$\\
Upgrade $-$ random & $G$ (pp) & +3.3 & 95 & $[1.5,5.2]$ & $[1.5,5.2]$\\
Upgrade $-$ score & $G$ (pp) & +3.0 & 95 & $[1.1,4.9]$ & $[1.1,4.9]$\\
Upgrade $-$ Factor & $G$ (pp) & -7.5 & 95 & $[-9.9,-5.2]$ & $[-10.0,-5.1]$\\
\bottomrule
\end{tabular}
\end{table}

Margin-based abstention submits 1,086 of 1,280 assigned requests and succeeds in 944, giving 84.8\% coverage and 73.8\% all-assigned task success. The submitted error rate is 13.1\%, corresponding to 71 errors among 543 distinct submitted description choices. These values characterize the accepted workload under the fixed quantile threshold.

\subsection{Same-Card Timing and Calls}
We use 64 descriptions for timing, four per choice, and cross their two roles with the existing checkpoints where applicable. Resident retrieval and Factor models run on the same RTX 4090D with batch size one and a Latin-square method order. Timing begins after warm-up and world setup, measuring formation to commitment and then the additional environmental work through terminal readback.

Table~\ref{tab:selection-timing-full} reports 1,024 timed trials, with 1,008 formed commitments and 1,007 fulfilled commitments because one Factor trial fails at declared completion. Selective upgrading invokes Factor on 12.5\% of this timed subset, compared with 15.2\% in the quality sample. Table~\ref{tab:selection-timing-effects} pairs the time, quality, and call differences after averaging roles and checkpoints within description.
\begin{table}[!htbp]
\centering\tabstyle
\caption{Formation time, physical outcomes, and generator use on the fresh timing sample. TTC is mean milliseconds, episode time is mean seconds, and Factor calls are per assigned trial.}
\label{tab:selection-timing-full}
\begin{tabular}{@{}lrrrrrrrr@{}}
\toprule
\textbf{Strategy} & $n$ & $C$ & $H$ & $F$ & $G$ & \textbf{TTC} & \textbf{Episode} & \textbf{Factor calls}\\
\midrule
Embedding retrieval & 128 & 128 & 128 & 128 & 116 & 91 & 22.98 & 0.000\\
Margin-guided upgrade & 384 & 384 & 384 & 384 & 366 & 144 & 23.01 & 0.125\\
Factor & 384 & 384 & 384 & 383 & 373 & 448 & 23.30 & 1.000\\
Margin-based abstention & 128 & 112 & 112 & 112 & 106 & 93 & 20.20 & 0.000\\
\bottomrule
\end{tabular}
\end{table}
\begin{table}[!htbp]
\centering\tabstyle
\caption{Paired effects on the fresh sample with 95\% intervals over 64 descriptions. Differences are first strategy minus second, so negative timing effects denote savings.}
\label{tab:selection-timing-effects}
\begin{tabular}{@{}llrr@{}}
\toprule
\textbf{Contrast} & \textbf{Metric} & \textbf{Difference} & \textbf{95\% CI}\\
\midrule
Margin upgrade minus retrieval & Success (pp) & +4.7 & $[0.0,9.4]$\\
Margin upgrade minus retrieval & TTC (ms) & +53 & $[34,71]$\\
Margin upgrade minus retrieval & Episode (ms) & +29 & $[-9,64]$\\
Margin upgrade minus retrieval & Factor calls & +0.125 & $[0.078,0.172]$\\
Margin upgrade minus Factor & Success (pp) & -1.8 & $[-6.5,2.3]$\\
Margin upgrade minus Factor & TTC (ms) & -304 & $[-355,-261]$\\
Margin upgrade minus Factor & Episode (ms) & -286 & $[-350,-227]$\\
Margin upgrade minus Factor & Factor calls & -0.875 & $[-0.922,-0.828]$\\
Abstention minus retrieval & Success (pp) & -7.8 & $[-12.5,-3.1]$\\
Abstention minus retrieval & TTC (ms) & +1 & $[-3,5]$\\
Abstention minus retrieval & Episode (ms) & -2789 & $[-3923,-1675]$\\
Abstention minus retrieval & Factor calls & +0.000 & $[0.000,0.000]$\\
\bottomrule
\end{tabular}
\end{table}

\subsection{Semantic Choices and Checkpoints}
Table~\ref{tab:selection-by-choice} reports each choice's outcomes over 40 descriptions, corresponding to 80 role trials for retrieval and 240 role--checkpoint trials for Factor strategies. Table~\ref{tab:selection-by-seed} gives the checkpoint decomposition, allowing readers to see which assignments benefit from additional generation and how that pattern varies across trained models.
\begin{table}[!htbp]
\centering\tabstyle
\caption{Task success by semantic choice, in percent. The three upgrade strategies have matched quotas, and abstention counts unsubmitted trials as unsuccessful.}
\label{tab:selection-by-choice}
\begin{tabular}{@{}lrrrrrr@{}}
\toprule
\textbf{Choice} & \textbf{Retrieval} & \textbf{Factor} & \textbf{Margin} & \textbf{Random} & \textbf{Score} & \textbf{Abstain}\\
\midrule
Active order: storage & 100.0 & 98.3 & 100.0 & 100.0 & 100.0 & 100.0\\
Active order: work & 65.0 & 99.2 & 75.0 & 65.0 & 70.0 & 60.0\\
Build site A & 92.5 & 97.5 & 95.0 & 95.0 & 95.0 & 82.5\\
Build site B & 82.5 & 100.0 & 95.0 & 82.5 & 90.0 & 40.0\\
Chest A & 100.0 & 100.0 & 100.0 & 100.0 & 100.0 & 85.0\\
Chest B & 40.0 & 100.0 & 87.5 & 75.0 & 72.5 & 10.0\\
Recipe: chest & 87.5 & 94.6 & 90.8 & 87.5 & 90.0 & 87.5\\
Recipe: crafting table & 97.5 & 100.0 & 97.5 & 97.5 & 97.5 & 97.5\\
Allocation: build & 87.5 & 94.2 & 87.5 & 87.5 & 87.5 & 87.5\\
Allocation: deposit & 92.5 & 100.0 & 92.5 & 92.5 & 92.5 & 92.5\\
Build component: roof & 97.5 & 94.6 & 97.5 & 97.5 & 97.5 & 92.5\\
Build component: wall & 77.5 & 96.7 & 80.0 & 77.5 & 75.0 & 77.5\\
Recipe: planks & 97.5 & 94.2 & 100.0 & 95.8 & 95.0 & 97.5\\
Recipe: sticks & 50.0 & 100.0 & 65.0 & 60.0 & 55.0 & 37.5\\
Terminal: depot & 60.0 & 66.7 & 58.8 & 60.8 & 59.2 & 52.5\\
Terminal: marker & 87.5 & 99.2 & 92.5 & 87.5 & 90.0 & 80.0\\
\bottomrule
\end{tabular}
\end{table}
\begin{table}[!htbp]
\centering\tabstyle
\caption{Task success for the existing Factor checkpoints, with counts out of 1,280 followed by percentages. The retrieval reference has 1,052 successes for every checkpoint comparison.}
\label{tab:selection-by-seed}
\begin{tabular}{@{}lrrrr@{}}
\toprule
\textbf{Checkpoint} & \textbf{Factor} & \textbf{Margin} & \textbf{Random} & \textbf{Score}\\
\midrule
Checkpoint 1 & 1,216 (95.0) & 1,133 (88.5) & 1,088 (85.0) & 1,094 (85.5)\\
Checkpoint 2 & 1,207 (94.3) & 1,126 (88.0) & 1,086 (84.8) & 1,090 (85.2)\\
Checkpoint 3 & 1,261 (98.5) & 1,136 (88.8) & 1,094 (85.5) & 1,096 (85.6)\\
\bottomrule
\end{tabular}
\end{table}

\subsection{Routing Outcomes}
Table~\ref{tab:selection-repair} distinguishes retrieval errors corrected by upgrading from correct retrieval choices changed into errors. Both the upgraded and retained groups contain errors that Factor can correct, explaining why selective use improves retrieval while full Factor reaches higher quality. The accompanying data also provide the writing-task sensitivity analysis.

\begin{table}[H]
\centering\tabstyle
\caption{Assignment changes caused by replacing retrieval with Factor. Corrected and introduced errors use matched descriptions and roles, with full Factor providing the comparison on descriptions retained by retrieval.}
\label{tab:selection-repair}
\begin{tabular}{@{}llrrrrr@{}}
\toprule
\textbf{Checkpoint} & \textbf{Routing group} & \textbf{Trials} & \textbf{Corrected} & \textbf{Changed to wrong} & \textbf{Both wrong} & \textbf{Both correct}\\
\midrule
1 & Retrieval kept & 1086 & 115 & 32 & 27 & 912\\
1 & Upgraded & 194 & 82 & 1 & 4 & 107\\
2 & Retrieval kept & 1086 & 117 & 36 & 25 & 908\\
2 & Upgraded & 194 & 80 & 6 & 6 & 102\\
3 & Retrieval kept & 1086 & 138 & 13 & 4 & 931\\
3 & Upgraded & 194 & 84 & 0 & 2 & 108\\
\bottomrule
\end{tabular}
\end{table}

\FloatBarrier

\subsection{Local and Hosted Assignment Models}
\label{sec:supp-selector-comparison}
We also apply a prompted DeepSeek selector to the same 640 test descriptions, comparing it with embedding retrieval, the unchanged Factor checkpoints, and retrieval with low-margin escalation to DeepSeek. The common composer expresses each selector's chosen assignment through the same executable request interface. Table~\ref{tab:assignment-comparison} evaluates the chosen binding against the intended binding before world execution, which directly measures the decision made by the formation backend.

\begin{table}[!htbp]
\centering\tabstyle
\caption{Assignment accuracy and measured decision cost. Correct assignments are determined before execution on 640 descriptions, with two roles and three repetitions for generative selectors. Timing uses 128 decisions per strategy. API charges are usage estimates.}
\label{tab:assignment-comparison}
\begin{tabular}{@{}lrrrr@{}}\toprule
Selector & Correct/assigned & Accuracy (\%) & Time (ms) & API USD\\\midrule
Embedding retrieval & 1,052/1,280 & 82.19 & 92 & 0.00000\\
Factor & 3,684/3,840 & 95.94 & 412 & 0.00000\\
DeepSeek & 3,839/3,840 & 99.97 & 1402 & 0.00009\\
Retrieval + DeepSeek & 3,414/3,840 & 88.91 & 350 & 0.00001\\
Random upgrade (expectation) & -- & 84.88 & -- & --\\\bottomrule
\end{tabular}
\end{table}

DeepSeek selects correctly in 3,839 of 3,840 trials, compared with 3,684 for Factor and 1,052 of 1,280 for retrieval. Routing the 97 lowest-margin descriptions to DeepSeek selects correctly in 3,414 of 3,840 trials. The random-upgrade value is an expectation under the same overall upgrade probability, computed from each description's retrieval and DeepSeek outcomes. Consequently it denotes a policy expectation over the observed decisions, rather than an additional set of executed trials.

For the decision-time comparison, all four strategies have 128 measured observations. Table~\ref{tab:assignment-comparison} reports the mean elapsed decision time and estimated API charge, including query encoding, any hosted generation, and the common formation checks within the recorded timing interval. Local methods incur zero API charges, while their local computation is represented by the measured latency. The timing subset invokes DeepSeek on 17 of 128 routed decisions, whereas the quality set routes 97 of 640 descriptions, so each sample's call rate follows its own inputs.

Although full DeepSeek attains the highest assignment accuracy, Factor uses about 29\% of its measured mean decision time and performs the decision locally. Selective escalation improves retrieval at an intermediate operating point, while Factor remains more accurate than that routed policy. Together with the closed-loop task results on template-derived descriptions in Section~\ref{sec:supp-online-descriptions}, these comparisons show why the common interface is useful: the selector can be chosen for the language distribution and deployment objective without changing what constitutes an executable commitment.

\section{Protocol Diagnostics}
\label{sec:supp-content}

This section supports the protocol checks associated with main Sections 3 and 5.3 by testing which assignment a delivered request or response causes the system to execute. The full-information reference and information-removal conditions provide complementary checks of the local contexts used in those interactions.

Figure~\ref{fig:supp-content-original} summarizes the local request and feedback diagnostics. We preserve the original task requirements when substituting a legal message, so the plots show how communication selects the tested assignment while execution and fulfillment are checked through the protocol.

\begin{figure}[!htbp]
\centering
\includegraphics[width=0.63\textwidth]{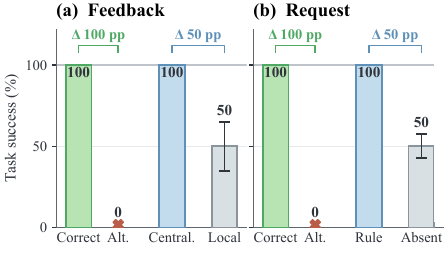}
\caption{Local content-substitution checks. (a) Correct and alternative feedback select different workcell assignments, while local planning and centralized information provide information references. (b) Correct, absent, and alternative requests test whether the receiver selects the branch communicated by the sender. Each axis measures success against the original task.}
\Description{The two original local content diagnostic panels. Feedback gives 100 percent success with the correct response or centralized information, zero with counterfactual feedback, and 50 with local planning. Requests give 100 percent with the correct message or a binding rule, zero with the legal alternative, and 50 without a request.}
\label{fig:supp-content-original}
\end{figure}

\subsection{Paired Local Information}
\label{sec:supp-ambiguity}
The local-information bound supports the diagnostic construction described in main Section 3. Consider equiprobable tasks $z_1,z_2$ with the same local view and disjoint acceptable route sets $\mathcal P_1,\mathcal P_2$. A policy choosing a route from that view has a common route distribution $\mu$, and therefore
\[
\frac{\mu(\mathcal P_1)+\mu(\mathcal P_2)}{2}\leq\frac12.
\]
The inequality follows because probabilities assigned to disjoint sets sum to at most one. We use the bound to construct a balanced check of whether a delivered request or response supplies the information needed for the prescribed route.

To quantify message substitution, let $Y_z(m)$ denote task success after message $m$ is delivered in task $z$, and let $\overline Y_z$ average its repeated role and model evaluations. We replace the true message $m_z$ with a legal alternative $\widetilde m_z$ under the same world and call procedure, defining
\begin{equation}\label{eq:supp-content-effect}
\widehat\Delta_{\mathrm{content}}=\frac{100}{|\mathcal Z|}\sum_{z\in\mathcal Z}\big[\overline Y_z(m_z)-\overline Y_z(\widetilde m_z)\big].
\end{equation}
Both outcomes use the original task specification, so the difference measures the consequence of the delivered content while fulfillment is evaluated against the commitment actually formed.

\subsection{Checking Failure Stages}
\label{sec:supp-fault-checks}
We validate the stage predicates by applying controlled changes to a proposal, a handoff, or the work after receipt, with 256 assigned trials for each of nine conditions. Table~\ref{tab:fault-stages} reports the common outcome within each condition, distinguishing a rejected proposal from an unsuccessful physical transfer and a completed but inappropriate assignment. The clean and irrelevant-change controls each have zero alarms in 256 trials.

\begin{table}[!htbp]
\centering\tabstyle
\caption{Stage checks under controlled faults. Every row describes 256 trials with the shown outcome, where a dash denotes a stage that is not reached. T is declared terminal completion.}
\label{tab:fault-stages}
\begin{tabular}{@{}lrrrrrr@{}}\toprule
Condition & Assigned & C & H & T & F & G\\\midrule
Clean & 256 & 1 & 1 & 1 & 1 & 1\\
Irrelevant change & 256 & 1 & 1 & 1 & 1 & 1\\
Malformed proposal & 256 & 0 & -- & -- & -- & 0\\
Incompatible handoff & 256 & 0 & -- & -- & -- & 0\\
Insufficient handoff & 256 & 1 & 0 & -- & 0 & 0\\
Wrong transfer item & 256 & 1 & 0 & -- & 0 & 0\\
Resource loss after receipt & 256 & 1 & 1 & 0 & 0 & 0\\
Terminal action failure & 256 & 1 & 1 & 0 & 0 & 0\\
Compatible wrong assignment & 256 & 1 & 1 & 1 & 1 & 0\\\bottomrule
\end{tabular}
\end{table}

Every injected condition is detected at the corresponding check, and independently evaluated predicates agree with those used by the execution pipeline. This verifies the measurement tools used to locate failure stages in the communication comparisons, including the distinction between fulfilling an accepted request and satisfying the original task.

\subsection{Request Content}
\label{sec:supp-request}

The true and alternative requests in Table~\ref{tab:request-intervention} preserve message presence, schedule, schema, approximate length, checkpoint, and calls, so the contrast changes delivered content. The absent-request reference instead uses an independently trained model that plans its own work, while the zero-call rule tests execution from an explicitly supplied correct assignment.

\begin{table}[!htbp]
  \centering
  \caption{Request substitutions and a deterministic assignment reference over 160 clusters. Learned methods use three checkpoints and 960 episodes, while the rule uses 320 role views. Differences are true request minus each condition, in percentage points. Bytes count canonical request objects rather than Short DSL surface strings.}
  \label{tab:request-intervention}
  \tabstyle
  \begin{tabular}{@{}lrrrrrr@{}}
    \toprule
    \thead{Condition} & \thead{Episodes} & \thead{Success} & \thead{Calls} & \thead{Request (B)} & \thead{$\Delta$ vs.\ true} & \thead{95\% CI}\\
    \midrule
    True learned request & 960 & 100.0\% & 2.0 & 304.4 & -- & -- \\
    Request removed (self-plan-only) & 960 & 50.0\% & 2.0 & 0.0 & $+50.0$ & $[42.5, 57.5]$ \\
    Same-template executable alternative & 960 & 0.0\% & 2.0 & 304.4 & $+100.0$ & $[100, 100]$ \\
    Deterministic correct binding & 320 & 100.0\% & 0.0 & 304.4 & $0.0$ & -- \\
    \bottomrule
  \end{tabular}
\end{table}

Replacing the request with the other legal option makes the receiver follow that option in 960/960 diagnostic episodes. The held-out pool similarly preserves the delivered assignment in all 471 executable outputs, while nine trials stop before forming an executable plan. These observations check consistency between the message, admission, and the branch subsequently realized.

\subsection{Feedback and Full Information}
\label{sec:supp-private}

\subsubsection{Paired Construction and Matched Training}

We cross two goals, two prescribed processing modes, and ten worlds to obtain 40 clusters, with 20 goal--world pairs sharing the requester's input and initial proposal. Two roles and three checkpoints give 240 trials per condition, or 960 across the four conditions. Primary intervals average repetitions within each cluster, and sensitivity intervals also average the paired modes. Tables~\ref{tab:e1e2-training}, \ref{tab:private-info}, and~\ref{tab:feedback-field-fidelity} connect the training configuration, task outcomes, and affected fields, while Figure~\ref{fig:unit-consistency}(c) displays individual clusters.

\begin{table}[!htbp]
  \centering
  \caption{Workcell feedback and centralized training budgets. Token occurrences include prompt, target, and end markers across the training epochs.}
  \label{tab:e1e2-training}
  \tabstyle
  \begin{tabularx}{\textwidth}{@{}Z Z r r r r r@{}}
  \toprule
  \thead{Model} & \thead{Visible information} & \thead{Rows} & \thead{Epochs} & \thead{Updates} & \thead{Train-token occurrences} & \thead{Seeds} \\
  \midrule
  Distributed requester & 320 initial local + 320 after-response local & 640 & 2 & 80 & 868,720 & 3 \\
  Centralized full information & merged current local views & 320 & 4 & 80 & 970,080 & 3 \\
  \bottomrule
\end{tabularx}

\end{table}

Distributed supervision contains 320 initial proposals and 320 revisions, whereas the centralized reference receives 320 joint-plan targets. Both use Qwen3.5-0.8B, effective batch size 16, learning rate $10^{-5}$, and maximum length 2,048 with final-step checkpoints. At matched updates the centralized model processes 1.117 times as many token occurrences, including prompts, targets, and end markers, and every example fits the length limit.

\subsubsection{Feedback-Content Intervention}

\begin{table}[!htbp]
  \centering
  \caption{Workcell task success and information components over 40 clusters forming 20 mode pairs. Intervals use the cluster bootstrap, and strict success requires all six role--checkpoint repetitions to succeed. The final column measures the stated message component.}
  \label{tab:private-info}
  \tabstyle
  \begin{tabularx}{\textwidth}{@{}Z r r r c r r Z@{}}
  \toprule
  \thead{Condition} & \thead{Seeds} & \thead{Epis.} & \thead{Success} & \thead{95\% CI} & \thead{Strict} & \thead{Calls} & \thead{Measured boundary component} \\
  \midrule
  Requester-local (A-only) & 3 & 240 & 50.0\% & $[35,65]$ & 20/40 & 1.0 & 0 B cross-boundary payload \\
  Correct bounded feedback & 3 & 240 & 100.0\% & $[100,100]$ & 40/40 & 2.0 & response: mean 103 B \\
  Counterfactual feedback & 3 & 240 & 0.0\% & $[0,0]$ & 0/40 & 2.0 & injected response: mean 103 B \\
  Centralized full information & 3 & 240 & 100.0\% & $[100,100]$ & 40/40 & 1.0 & state aggregation: mean 956 B \\
  \bottomrule
\end{tabularx}

\end{table}

Correct feedback resolves the assignment left ambiguous by local planning, while an admissible counterfactual redirects the work under the same two-call procedure (Table~\ref{tab:private-info}). Pair-block sensitivity gives a 50.0-point gain $[50,50]$ over local planning and a 100.0-point gain $[100,100]$ over the counterfactual. Every checkpoint reproduces the 50/100/0 profile. The 720 feedback trials parse and complete planning, and all 240 counterfactual trials complete transfer and deposit while violating the prescribed actor and handoff form.

\begin{table}[!htbp]
  \centering
  \caption{Assignment fields after counterfactual feedback, separating the altered division of work from the unchanged requested product.}
  \label{tab:feedback-field-fidelity}
  \tabstyle
  \begin{tabularx}{\textwidth}{@{}Z r r Z@{}}
  \toprule
  \thead{Field} & \thead{Matches injected} & \thead{Matches true world} & \thead{Observed behavior} \\
  \midrule
  Craft actor & 240/240 & 0/240 & rewritten \\
  Handoff item & 240/240 & 0/240 & rewritten \\
  Handoff count & 240/240 & 0/240 & rewritten \\
  Peer suffix & 240/240 & 0/240 & rewritten \\
  Goal item & 240/240 & 240/240 & invariant \\
  Goal count & 240/240 & 240/240 & invariant \\
  \bottomrule
\end{tabularx}

\end{table}

Because the intervention fixes the requester's context, initial proposal, world, decoding, and call schedule, Table~\ref{tab:feedback-field-fidelity} identifies the allocation fields controlled by the response. The craft actor, handoff object and count, and continuation follow the injected alternative in all 240 counterfactual trials, while the goal object and count stay fixed.

The data labels \texttt{RAW\_PROCESSOR} and \texttt{FINISHED\_RECEIVER} correspond to Mode B-crafts and Mode A-crafts, with the first mode assigning the peer to process transferred material. In the second, \artifact{CRAFT\_AT\_REQUESTER\_}\allowbreak\artifact{HANDOFF\_FINISHED} names the alternative that moves crafting before transfer, allowing both routes to produce the required item under different assigned responsibilities.

\subsubsection{Centralized Full-Information Reference}

We supply both local views to one model call in the centralized reference, which succeeds in 240/240 trials and matches the correct-feedback outcome. Its mean state upload is 955.5 B, with a 963 B 95th percentile, whereas the response component averages 103 B with a 118 B 95th percentile. Since these components carry different information, the complete accounting below sums uploads, downloads, and peer exchanges for each protocol.

\Needspace{12\baselineskip}
\begin{samepage}
For architecture accounting, let $\mathcal A_t\subseteq\mathcal A$ be the participating agents at round $t$, let $E_t\subseteq\mathcal A\times\mathcal A$ be the directed active-edge set, and let $\mathcal M_{ij}^t$ be the complete sequence of application messages sent on directed edge $(i,j)$. The function $b(m)$ counts canonical UTF-8 payload bytes, including any application wrapper present in $m$ but excluding transport framing. Let $m_{i,\uparrow}^{t}$ and $m_{i,\downarrow}^{t}$ denote centralized upload and download payloads. With $R_c$ centralized rounds and $R_l$ local rounds,
\begin{align}
B_{\mathrm{central}}
  &= \sum_{t=1}^{R_c}\sum_{i\in\mathcal A_t}
     \left[b(m_{i,\uparrow}^{t})+b(m_{i,\downarrow}^{t})\right], \\
B_{\mathrm{local}}
  &= \sum_{t=1}^{R_l}\sum_{(i,j)\in E_t}
     \sum_{m\in\mathcal M_{ij}^{t}} b(m).
\end{align}
Feedback sent in the reverse direction belongs to the corresponding edge $(j,i)$. Totals depend on coordination rounds; peer traffic also depends on active edges and messages per edge.
\end{samepage}

\subsection{Request Transfer to a Prompted External Model}
\label{sec:supp-api}

This section supports the use of an external formation backend by giving the complete request-transfer outcomes and the prompts used in that diagnostic. We preserve the executable interface while changing how the receiver produces its continuation, so the complete outcomes can be interpreted against the same task requirements.

\subsubsection{Instances, Conditions, and Statistical Units}

We construct the request-transfer set from eight templates and ten paired world blocks per template, each with two assignments and two role views. Three evaluations per role view produce 960 trials for each condition and 7,680 across both backends and four conditions. API repetitions call the same service, whereas local repetitions use the three unchanged Direct checkpoints. Test worlds are distinct from the request pool in Section~\ref{sec:supp-request} and from development worlds.

Correct, absent, and alternative request trials share an upstream plan and the receiver's context before delivery, varying the message supplied at that point. Here removal uses the same backend, unlike the separately trained absent-request reference in Section~\ref{sec:supp-request}. In the complete protocol, sender and receiver generate sequentially from their own inputs. A legal path matching the injected alternative can proceed to execution, whereas a path contradicting the delivered request is rejected before commitment formation.

For inference we average the two assignments, two roles, and three repetitions within each world block, then resample ten blocks within each template for 10,000 draws. The true-minus-alternative contrast measures replacement under the same procedure, whereas true-minus-absent also changes message presence. Both contrasts keep all assigned trials in the task-success calculation.

\begin{table}[!htbp]
\centering
\caption{Prompted and local request-transfer outcomes over 960 trials per condition. Formation and success are counts with percentages, while failure columns partition sender formation, receiver formation, request conflict, and formed commitments failing the original task.}
\label{tab:api-transfer}
\tabstyle
\begin{tabular}{@{}l l c c rrrr@{}}
\toprule
\thead{Backend} & \thead{Condition} & \thead{Formed (\%)} & \thead{Success (\%)} & \thead{S} & \thead{R} & \thead{C} & \thead{W}\\
\midrule
DeepSeek & Correct request & 949 (98.9) & 949 (98.9) & 0 & 9 & 2 & 0\\
DeepSeek & Absent request & 795 (82.8) & 385 (40.1) & 0 & 165 & 0 & 410\\
DeepSeek & Alternative request & 957 (99.7) & 0 (0.0) & 0 & 3 & 0 & 957\\
DeepSeek & Full protocol & 952 (99.2) & 952 (99.2) & 1 & 7 & 0 & 0\\
\midrule
Local & Correct request & 960 (100.0) & 960 (100.0) & 0 & 0 & 0 & 0\\
Local & Absent request & 960 (100.0) & 480 (50.0) & 0 & 0 & 0 & 480\\
Local & Alternative request & 960 (100.0) & 0 (0.0) & 0 & 0 & 0 & 960\\
Local & Full protocol & 958 (99.8) & 958 (99.8) & 2 & 0 & 0 & 0\\
\bottomrule
\end{tabular}
\end{table}

\begin{table}[!htbp]
\centering
\caption{Paired task-success differences in percentage points, with 95\% intervals from resampling 80 paired world blocks within template.}
\label{tab:api-effects}
\tabstyle
\begin{tabular}{@{}l l rr@{}}
\toprule
\thead{Backend} & \thead{Contrast} & \thead{Effect (pp)} & \thead{95\% CI}\\
\midrule
DeepSeek & Correct minus legal alternative & $+98.9$ & $[98.1,99.5]$\\
DeepSeek & Correct minus absent & $+58.8$ & $[56.8,60.8]$\\
Local & Correct minus legal alternative & $+100.0$ & $[100.0,100.0]$\\
Local & Correct minus absent & $+50.0$ & $[50.0,50.0]$\\
\bottomrule
\end{tabular}
\end{table}

\subsubsection{Content Effects and Failure Location}

Tables~\ref{tab:api-transfer} and~\ref{tab:api-effects} show how request use carries from a fine-tuned local model to prompted external generation. The prompted correct-request failures comprise nine invalid receiver formations, including seven in Active Branch, and two legal branches that contradict the request. The examples below distinguish those contradictions from syntax errors and include them as unsuccessful trials.

Figure~\ref{fig:unit-consistency} expands these aggregate outcomes to individual world blocks, relating the absent-request behavior to the selected assignment and comparing it with the feedback clusters.

\begin{figure}[!htbp]
  \centering
  \includegraphics[width=\textwidth]{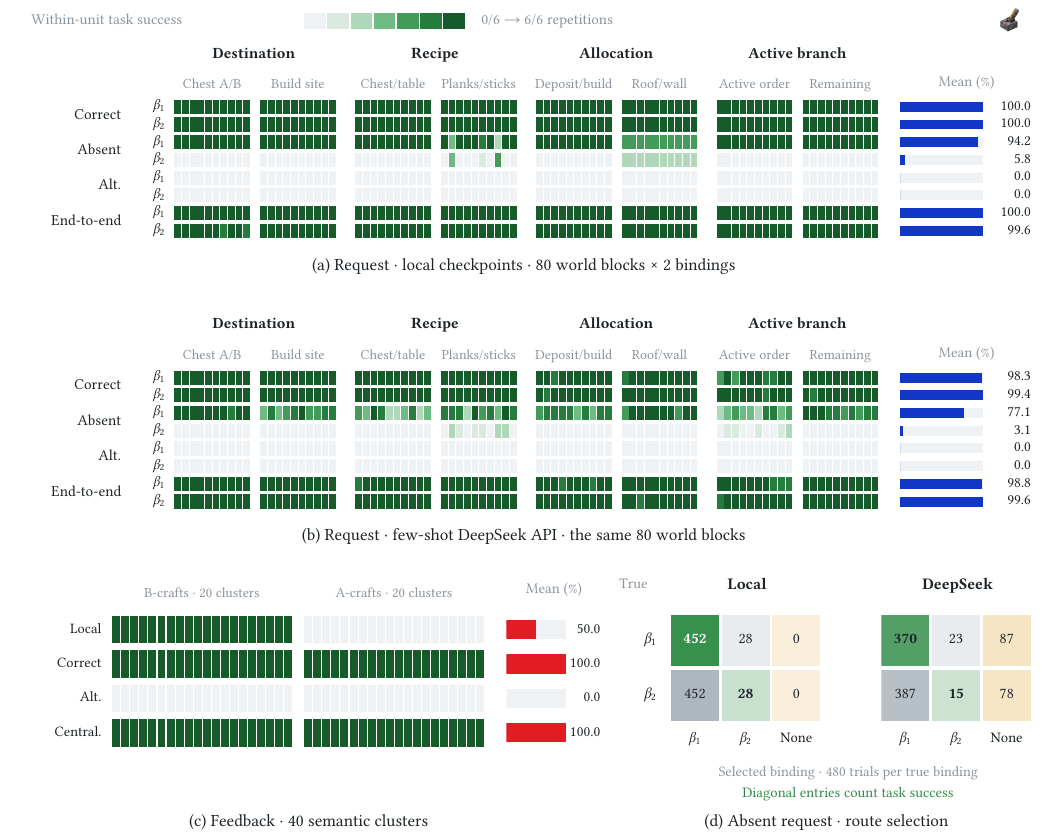}
  \Description{Barcode heat maps with one cell per world block or cluster and marginal bars of pooled task success. For both local checkpoints and the DeepSeek API, correct requests and end-to-end runs succeed in nearly every block and alternative requests fail in every block; absent requests succeed mainly on the first binding. For feedback, correct and centralized conditions succeed in all clusters, alternative feedback fails in all clusters, and local planning succeeds only for the B-crafts mode. Two confusion tables show that without the request both backends mostly select the first binding.}
  \caption{Request and feedback outcomes by evaluation unit. Panels (a,b) show the local and prompted backends on 80 paired world blocks, and panel (c) shows 40 workcell clusters by processing mode. Each cell averages six role--repetition or role--checkpoint trials, with pooled rates beside the rows. Panel (d) relates the choice made without a request to the true assignment, while None denotes the absence of a formed plan. The outcomes provide a unit-level check of the paired-information construction in Section~\ref{sec:supp-ambiguity}.}
  \label{fig:unit-consistency}
\end{figure}

\paragraph{Two legal branches that violate the request.}
In the Deposit-or-Build example, the receiver is asked to construct an object but generates a deposit instead. The displayed request and output are the transmitted and generated strings from the first API repetition, with Agent B sending to Agent A:
\begin{lstlisting}[style=gcpdsl]
Delivered:
REQ agent_a build.component(bind=BUILD,q=1,item=cobblestone,site=marker_site)
Generated:
SELF resource.deliver(bind=DEPOSIT,q=4,dst=reserve_chest,item=cobblestone)
REQ -
\end{lstlisting}
\Needspace{9\baselineskip}
The Remaining Terminal example has the same sender and receiver roles, but substitutes a depot continuation for the requested marker construction:
\begin{lstlisting}[style=gcpdsl]
Delivered:
REQ agent_a build.component(bind=MARKER_TERMINAL,q=1,item=oak_planks,site=terminal_marker)
Generated:
SELF resource.deliver(bind=DEPOT_TERMINAL,q=2,dst=terminal_chest,item=oak_planks)
REQ -
\end{lstlisting}
Although both outputs have valid syntax and name a public branch, request checking rejects them because their work contradicts the delivered request. The examples therefore illustrate a consistency requirement applied before world execution.

\paragraph{Information absence and grounded outcomes.}
Without a request, DeepSeek forms 795 commitments and selects the correct assignment in 385 of them, or 48.4\% of formed cases. Its all-assigned success is 40.1\% because another 165 trials stop during formation, consistent with an information upper bound that permits lower success. Across both backends and conditions, all 2,807 formed plans selecting the wrong goal fail the task, whereas all 4,684 matching plans succeed. This verifies the link from the admitted branch to the mutually exclusive terminal target.

\subsubsection{Model, Prompt, and Generation Settings}

The request-transfer diagnostic called DeepSeek-V4.1-Flash through the service alias \texttt{deepseek-flash} on 2026-09-28, using the prompt and settings below. The API documentation citation records its consultation date of September 29~\cite{deepseek2026api}, distinct from the experiment's access date.

We used an Anthropic-compatible endpoint with \texttt{thinking.type} set to \texttt{disabled} and an 8,192-token output limit. The request also set \texttt{output\_config.effort} to \texttt{high}, while temperature, top-$p$, seed, tools, and stop sequences were omitted. Each call supplied one user message and the diagnostic allowed four concurrent calls. Parsing removed an optional outer Markdown fence before applying the task contract, and unsuccessful generations remained in the assigned outcomes.

Each prompt provides four training examples, one from Active Order Branch, Build Site A/B, Chest/Crafting Table, and Deposit/Build. Sender and receiver calls use the examples for their own stage, beginning with the following instruction:
\begin{lstlisting}[style=gcpdsl]
Four TRAIN syntax examples; solve only the final MODEL_INPUT.
\end{lstlisting}
After the examples, the prompt appends \texttt{CURRENT INSTANCE} and the agent's local input, followed by the task instruction:
\begin{lstlisting}[style=gcpdsl]
Use only MODEL_INPUT. Generate self tasks and, when locally justified, one peer request. Return exactly two Short DSL lines: SELF ... and REQ .... No prose or semantic repair.
\end{lstlisting}
The local input identifies the agent and includes its task structure, observations, history, and delivered messages, while the sender's private choice appears only in the sender view. The accompanying data contain the complete examples and prompts. Local calls use their checkpoint prompts with deterministic decoding and a 128-token output cap, so the common interface is evaluated with each backend's stated generation settings.

\subsubsection{API Timing and Generation Settings}

\paragraph{API deployment timing.}
We also time the two-call API protocol on 64 role instances repeated three times, with sequential calls and no model warm-up. Among 192 trials, 189 form commitments with median TTC 1.52 s, mean 2.45 s, 95th percentile 2.70 s, and maximum 154.24 s. These measurements include the service, network, and retry time of the evaluated API interaction.

The diagnostic uses non-thinking generation with an 8,192-token cap, allowing the complete structured response to be returned under the prompt specified above.

\section{Structural Cues in Formation}
\label{sec:supp-structure}

This section supports main Section 4 by examining how task-graph visibility and node order affect generated paths.

\begin{figure}[!t]
\centering
\includegraphics[width=\textwidth]{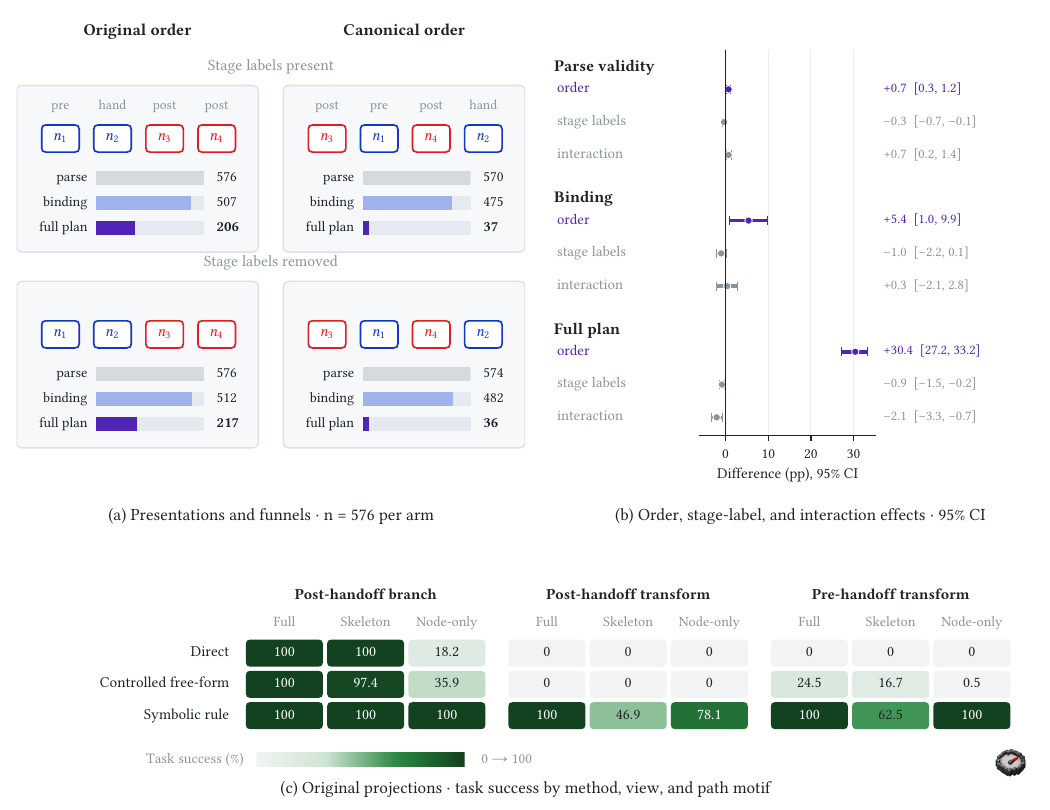}
\Description{Panel a is a two-by-two grid of schematic node sequences crossing stage labels present or removed with original or canonical order, each with bars for parse validity, binding correctness, and full-plan correctness out of 576. Panel b is a forest plot of order, stage-label, and interaction effects for the three outcomes; the order effect on full-plan correctness is about thirty points. Panel c is a heat map of task success for three methods across Full, Skeleton, and Node-only views for three path motifs.}
\caption{Task structure and presentation order. Panel (a) crosses stage labels with original or canonical node order and reports parsing, assignment, and complete-plan counts out of 576. Panel (b) gives paired factorial effects with 95\% cluster intervals. Panel (c) compares the Full, Skeleton, and Node-only projections over 32 clusters per motif, using two roles and three checkpoints for learned methods and a seedless symbolic rule. Blue and red nodes indicate sender and receiver work.}
\label{fig:stage-order-control}
\end{figure}

\subsection{Visibility Projections}
\label{sec:supp-visibility}

Figure~\ref{fig:stage-order-control}(c) compares Direct, controlled free-form, and a symbolic rule across 96 clusters. Full supplies typed nodes, stages, dependencies, and edges. Skeleton removes dependencies, whereas Node-only also removes stages and orders nodes independently of the answer. The symbolic rule uses a fixed path priority. Direct completes every familiar post-handoff case under Full and Skeleton, while both new transformation-order motifs have zero success.

\subsection{Stage Labels and Node Order}
\label{sec:supp-stage-order}
We cross stage-label presence with original versus canonical order on the same nodes, fields, targets, and checkpoints. Each view contains 576 evaluations over 96 clusters, two roles, and three checkpoints (Figure~\ref{fig:stage-order-control}(a,b)). Main effects average over the other factor, and interaction compares the order effects. We average roles and checkpoints within cluster and use 10,000 paired resamples within template--motif strata. The positive complete-plan effect belongs to node order, whereas stage labels have a small negative average effect.

\FloatBarrier
\end{document}